\documentclass[a4paper, 11pt]{elsarticle}

\usepackage{graphics}
\usepackage{graphicx}
\usepackage{psfig}
\usepackage{rotating}
\usepackage[absolute]{textpos}
\usepackage{color}
\usepackage{lineno}
\usepackage[algo2e]{algorithm2e} 
\usepackage[utf8]{inputenc}
\usepackage{amsmath}
\usepackage{enumitem}

\usepackage{amsfonts}
\usepackage{amssymb}
\usepackage{amsthm}
\usepackage{mathrsfs}
\usepackage{lscape}
\usepackage{latexsym}
\usepackage{multirow}
\usepackage{rotating}
\usepackage{caption}
\usepackage{graphicx}
\usepackage{float}
\usepackage{array}
\usepackage{subfigure}
\usepackage{tabularx}
\usepackage{graphicx,psfig}
\usepackage{footnote}
\usepackage[titletoc,toc,title]{appendix}
\usepackage{colortbl}
\usepackage[dvipsnames]{xcolor}
\usepackage{hyperref}
\usepackage{natbib}
\usepackage{amssymb}
\usepackage{soul}
\usepackage{algorithmic}
\usepackage{algorithm}
\usepackage{epsfig}
\usepackage{relsize}
\usepackage{color}
\usepackage{mathtools}
\usepackage{graphicx,booktabs}
\usepackage{multirow}
\usepackage[utf8]{inputenc}
\usepackage[english]{babel}
\newtheorem{definition}{Definition}
\newtheorem{theorem}{Theorem}

\newtheorem{prop}{Proposition}
\newtheorem{remark}{Remark}

\begin{document}

\title{Epicasting: An Ensemble Wavelet Neural Network for Forecasting Epidemics}

\author{Madhurima Panja\textsuperscript{1,}
\footnote[0]{\textit{Equal Contributions}}, 
Tanujit Chakraborty\textsuperscript{1, 0, 3}
\footnote[2]{\textit{Corresponding author}: \textit{Mail}: tanujit.chakraborty@sorbonne.ae},
Uttam Kumar\textsuperscript{1},
Nan Liu\textsuperscript{3}\\
{\scriptsize \textsuperscript{1} Spatial Computing Laboratory, Center for Data Sciences, International Institute of Information Technology Bangalore, India.}\\
{\scriptsize \textsuperscript{2} Department of Science and Engineering, Sorbonne University Abu Dhabi, UAE.} \\
{\scriptsize \textsuperscript{3} Sorbonne Center for Artificial Intelligence, Paris and Abu Dhabi.} \\
{\scriptsize \textsuperscript{4} Duke-NUS Medical School, National University of Singapore, Singapore.}}

\begin{abstract}
Infectious diseases remain among the top contributors to human illness and death worldwide, among which many diseases produce epidemic waves of infection. The lack of specific drugs and ready-to-use vaccines to prevent most of these epidemics worsens the situation. These force public health officials and policymakers to rely on early warning systems generated by accurate and reliable epidemic forecasters. Accurate forecasts of epidemics can assist stakeholders in tailoring countermeasures, such as vaccination campaigns, staff scheduling, and resource allocation, to the situation at hand, which could translate to reductions in the impact of a disease. Unfortunately, most of these past epidemics exhibit nonlinear and non-stationary characteristics due to their spreading fluctuations based on seasonal-dependent variability and the nature of these epidemics. We analyze various epidemic time series datasets using a maximal overlap discrete wavelet transform (MODWT) based autoregressive neural network and call it Ensemble Wavelet Neural Network (EWNet) model. MODWT techniques effectively characterize non-stationary behavior and seasonal dependencies in the epidemic time series and improve the nonlinear forecasting scheme of the autoregressive neural network in the proposed ensemble wavelet network framework. From a nonlinear time series viewpoint, we explore the asymptotic stationarity of the proposed EWNet model to show the asymptotic behavior of the associated Markov Chain. We also theoretically investigate the effect of learning stability and the choice of hidden neurons in the proposal. From a practical perspective, we compare our proposed EWNet framework with twenty-two statistical, machine learning, and deep learning models for fifteen real-world epidemic datasets with three test horizons using four key performance indicators. Experimental results show that the proposed EWNet is highly competitive compared to the state-of-the-art epidemic forecasting methods. 
\end{abstract}
\begin{keyword}
Wavelet methods, MODWT, epidemiology, neural networks, time series forecasting.
\end{keyword}
\maketitle
\section{Introduction} \label{sec:intro}
Epidemiological modeling is a centuries-old field of research; however, still handy in guiding decision-making and devising appropriate interventions that mitigate the impacts of epidemics \cite{snow1855mode, hamer1906epidemic, mckendrick1914studies}. Most recently, epidemiological modeling and forecasting have become an immediate choice for designing policies for public health officials during outbreaks \cite{ferguson2001foot, keeling2001dynamics, funk2018real}. Epidemiological forecasting models (we will henceforth refer to as epicasters) can be used to forecast the total number of confirmed cases to define intervention strategies (e.g., \cite{thompson2019detection}). Recent examples of real-time modeling during epidemic outbreaks can be drawn from vector-borne diseases such as Malaria \cite{rouamba2020addressing}, Dengue \cite{johnson2018phenomenological}, the flu (Influenza) \cite{rangarajan2019forecasting}, viral infection (Hepatitis) \cite{wang2018comparison}, and most recent Covid-19 pandemic \cite{chakraborty2020real,chakraborty2020nowcasting}. Despite tremendous progress in public health practice in the 21st century, infectious diseases caused by microorganisms are still the leading cause of morbidity and mortality on the global level. Out of many causes of mortality, deaths due to infectious diseases (more precisely, epidemics and pandemics) are one of the leading causes of death in the last centennial \cite{jemal2005trends}. 
Since many of these epidemics were not foreseen or predicted thus, their untimely outbreak results in the mass destruction of limited resources and the collapse of the economy \cite{bhatt2013global}. This problem is pivotal in developing countries, particularly with the concurrent rising trends in the occurrence of epidemics. Therefore, early knowledge of epidemic timing, intensity, and mortality rates are crucial in designing countermeasures to reduce the impact of such cumbersome outbreaks. However, these early warning systems are usually designed following two strategies: ``nowcasting" and ``forecasting". While the former helps develop situational awareness by predicting the disease incidence at a time near the available data \cite{wu2020nowcasting, chakraborty2020nowcasting}, the latter is designed for formulating control response strategies well ahead of time to handle large-scale emergencies \cite{johansson2019open,chakraborty2019forecasting}. In our research, we combine the tasks of nowcasting and forecasting for predicting the disease incidence (specifically epidemics) at a time near and after the available data and collectively designate it as ``epicasting". The primary goal of the epicasting models is to accurately forecast the disease dynamics for formulating real-time outbreak management decisions and developing informed future response policy \cite{roosa2020real, mcroberts2019using}. 

Within the scope of epidemic modeling and forecasting, several mechanistic (or deterministic) and phenomenological models have been proposed. Amongst the available deterministic methodologies, compartmental models are widely used to study the changes in the characteristics (e.g., age, gender) and state (e.g., susceptible to, infectious with, or recovering from a particular disease) of the population by segregating them into several ``compartments'' \cite{brauer2008compartmental}. The simple SIR (susceptible-infected-recovered) model, consisting of a system of three coupled non-linear ordinary differential equations, yields several fundamental insights into outbreaks of infectious diseases and their control \cite{weiss2013sir}. Despite these mechanistic models' vast applicability, they are more suitable for ``understanding'' the disease dynamics rather than real-time forecasting the outbreak, which is one of the primary motivations for epicasting \cite{keeling2011modeling}. To overcome the problem of limited predictability of the mechanistic approaches, several attempts to anticipate the infectious disease dynamics with statistical and machine learning approaches have been adopted \cite{clayton2013statistical,chakraborty2019forecasting,chakraborty2020nowcasting}. Some examples of epicasting models are as follows: Modified version of autoregressive (AR) model for forecasting dengue epidemic datasets \cite{deb2022ensemble}; Bayesian methodology for analyzing malaria outbreak \cite{rouamba2020addressing}; Autoregressive likelihood ratio for forecasting influenza incidence \cite{rangarajan2019forecasting} amongst many others. While statistical models focus on parametric methods for predicting disease outbreaks, modern machine learning, and deep learning methodologies have been used to learn temporal disease dynamics in a purely data-driven approach \cite{santosh2020lstm,wu2020deep}. Several other statistical forecasters have been developed in the recent literature; among them, the most popular models are Random Walk (RW) \cite{pearson1905problem}, Random Walk with Drift (RWD) \cite{entorf1997random}, Autoregressive Integrated Moving Average (ARIMA) \cite{box2015time}, Exponential Smoothing State Space (ETS) \cite{hyndman2008forecasting}, Theta Model \cite{assimakopoulos2000theta}, Wavelet-based ARIMA (WARIMA) \cite{aminghafari2007forecasting}, Self-exciting Threshold Autoregressive (SETAR) \cite{tong2009threshold}, Trigonometric Box-Cox ARIMA Trend seasonality (TBATS) \cite{de2011forecasting}, Bayesian Structural Time Series (BSTS)  \cite{scott2014predicting}, and Hybrid ARIMA-WARIMA (we call it Hybrid-1) \cite{chakraborty2020real}. With the increasing data availability and computation power, machine learning and deep learning architectures have become a vital part of epidemic forecasting and are widely used as individual forecasters or in a hybridized environment \cite{wang2019defsi,chakraborty2019forecasting,johansson2019open}. A non-exhaustive list of such machine learning and deep learning models are Artificial Neural Networks (ANN) \cite{rumelhart1986learning}, Autoregressive Neural Networks (ARNN) \cite{faraway1998time}, Support Vector Regression (SVR) \cite{smola2004tutorial}, Long Short-term Memory (LSTM) network \cite{hochreiter1997long}, NBeats \cite{oreshkin2019n}, Deep AR \cite{salinas2020deepar}, Temporal Convolutional Networks (TCN) \cite{chen2020probabilistic}, Transformers \cite{wu2020deep}, Hybrid ARIMA-ANN (we call it Hybrid-2) \cite{zhang2003time}, Hybrid ARIMA-ARNN (we call it Hybrid-3) \cite{chakraborty2019forecasting}. Applying leading-edge research concerning epicasting of dengue, malaria, influenza, and other infectious disease confirmed cases, recovered cases, and mortality using the above-mentioned compartmental, statistical, machine learning, and deep learning methods are given in Table \ref{Table1}. 

\begin{table*}
\caption{Related works on epidemiological forecasting}
  \centering
\tiny
\begin{tabular}{|p{2.5cm}|p{1cm}|p{1.8cm}|p{3.3cm}|p{3.3cm}|p{3.5cm}|}
\hline
\textbf{Research Topic}      
& \textbf{Disease}   
& \textbf{Countries} 
&\textbf{Model}      
& \textbf{Results}   
& \textbf{Conclusion} 
\\\hline
Forecasting epidemics based on geographical hierarchy \cite{gibson2021improving}
& Influenza
& United States
& Weighted combination of forecasts for different regions where the weights are selected relative to the population size -  a probabilistic coherence approach.
& The proposed approach is 79\% more efficient in predicting influenza incidence for multiple seasons.
& National incidence is a weighted average of region-wise incidence and selecting the weights based on the demography of regions is an essential consideration in improving forecasts.
\\ \hline
Parameter identification in  epidemic forecasting \cite{mummert2019parameter}
& Influenza
& United States
& Local lagged adapted generalized method of moments (LLGMM) for  parameter identification in compartmental SEIRS model.
& The model shows a good qualitative fit for long-term forecasts.
& The LLGMM parameter estimation technique shows promising results in forecasting the incidence rate and can be further improved by considering more complex models than SEIRS.
\\ \hline
Forecasting epidemics with sparse representation \cite{rangarajan2019forecasting}
& Dengue and influenza
& Brazil, Mexico, Singapore,
Taiwan, Thailand, and the United States
& Autoregressive Likelihood Ratio (ARLR) Methodology. 
& The forecasts generated by the ARLR model reduce the RMSE and MAE scores by 18\% compared to traditional forecasting techniques.
& Electronic health records, historical incidence data, and frequency of internet search terms on Google trends provide valuable information for epicasting.
\\\hline
Epidemic analysis and forecasting \cite{ho2015time}
& Dengue
& Malaysia
& Seasonal and Trend decomposition with Loess method (STL), Holt Method, ARIMA, and STL-ETS.
& MAE, RMSE, and MASE scores are the least for the STL method.
& The dengue data exhibits trend and seasonality and can be best forecasted with the STL model.
\\\hline
Overcoming the challenges in epidemic forecasting due to data scarcity \cite{rouamba2020addressing}
& Malaria
& Burkina Faso
& Bayesian methodology for spatio-temporal prediction.
& 6-months ahead forecasts have actual cases within 95\% credible interval.
& Spatial fractional variance value suggested a strong spatial dependence of malaria incidence. \\\hline
Early detection of epidemic outbreak \cite{deb2022ensemble}
& Dengue
& San Juan and Iquitos
& A weighted ensemble of negative binomial regression, seasonal ARIMA, and generalized linear ARMA models, with weights, selected relative to the  performance on training data. 
& Ensemble method is most suitable for forecasting outbreaks compared to its components as evident from the MAE score.
& Climate and terrain factors provide useful information for forecasting the dengue outbreak in these regions.
\\\hline
Predicting epidemic incidence with Baidu search-engine data \cite{liu2019dengue}
& Dengue
& South China
& Generalized Additive Mixed (GAMX) Model. 
& GAMX showed 72\% and 10\% improvement in RMSE and $R^2$ compared to the Generalized Additive Model (GAM) for generating 6-months ahead forecasts.
& Historical incidence data along with climatic conditions played an essential role in accurately forecasting dengue incidence in South China. 
\\\hline
Forecasting Dengue \cite{buczak2018ensemble}
& Dengue
& San Juan and Iquitos
& Ensemble framework including two-dimensional method of analogs, additive Holt Winter's method with and without wavelet smoothing.
& Ensemble model forecasts a maximum number of weekly cases and total case count with minimum RMSE score compared to traditional forecasters.
& Their method scored the maximum rank in predicting weekly maximum count and total count in the 2015 NOAA Dengue Challenge.
\\\hline
Modelling epidemic transmission \cite{jing2018imported}
& Dengue
&  Guangzhou, China
& ARIMA with exogenous variables (ARIMAX).
& The forecasts generated  by the model report an RMSE value of 0.6445 and a consistency rate of 0.7917.
&  Imported cases and climatic conditions are key determinants of modeling local epidemic transmission.
\\\hline
Hybrid methodology for epicasting \cite{chakraborty2019forecasting}
& Dengue
& Peru, Philippines, Puerto Rico
& Remodeling the ARIMA residuals with an ARNN model and hybridizing the ARIMA and ARNN outputs for forecasting dengue cases.
& Hybrid model produces the best forecast with a one-year lead time based on MAE, RMSE, and sMAPE scores.
& Hybrid ARIMA-ARNN model is best suited for long-term forecasting.
\\ \hline
Modelling trajectories of Dengue \cite{johnson2018phenomenological} 
& Dengue
& Iquitos and San Juan
& Gaussian process (GP) regression model. 
& The GP approach predicts the future by memorizing historical data and performs superior to the generalized linear model (GLM) techniques that model the lagged observations along with climatic conditions.
& This method is advantageous in situations with a lack of ancillary covariates. \\ \hline
Modelling and forecasting epidemics \cite{wang2018comparison}
& Hepatitis B 
& China
& Seasonal ARIMA and grey model (GM). 
& RMSE \& MAE scores of the SARIMA model were lower than the GM model in forecasting the future trajectory.
& Utilizing SARIMA model forecasts is a supporting tool for health officials to control hepatitis outbreaks in China.
\\ \hline
Malaria forecasting data from 1994 to 1999 \cite{PMID:14521780}
& Malaria
& Honghe State, China
& Artificial Neural network (ANN).
& ANN model has been used and decreased the error of statistical models. 
& Neural network model was effective for forecasting malaria. It has the ability for more accurate forecasting and easy applicability. 
\\\hline
Prediction of the spread of influenza epidemics \cite{viboud2003prediction}
& Influenza-like illness (ILI)
& France
& Naive method.
& Ten weeks ahead forecast for the temporal and spatial spread of influenza was generated.
& Their method proved appropriate for forecasting both national and regional ILI incidences during the epidemic and pre-epidemic periods.\\\hline
Deep learning approach for modeling epidemic \cite{santosh2020lstm}
& Malaria
& Telangana, India
& Long short-term memory (LSTM) model.
& 12-months ahead prediction was evaluated based on several accuracy measures.
& LSTM successfully forecasts the endemic periods in the upcoming year for four different regions in Telangana.
\\\hline
Machine learning-based algorithm to determine epidemic transmission \cite{ch2014support}
& Malaria
& Rajasthan, India
& Hybridized Support Vector Machine with Fire Fly Algorithm (SVM-FFA).
& One step ahead forecast was evaluated based on different performance indicators.
& The coupled SVM-FFA approach exhibited better accuracy in predicting malaria incidence than several benchmark forecasters. \\ \hline

\end{tabular}
\label{Table1}
\end{table*}

Albeit the applicability of statistical models for epicasting, these models impose some restrictions on the data characteristics before their application. For example, real-world epidemic datasets show complex, noisy, non-stationary, and nonlinear behavior owing to the changing population size and climatic conditions \cite{weiss2013sir, duncan1996whooping}. In such a scenario, pre-processing the complex time series with suitable mathematical transformations has often yielded satisfactory results \cite{cazelles2007time}. One such widespread mathematical transformation is log transformation which effectively analyzes skewed data and reduces variability. Log transformation generally makes the transformed dataset conform more closely to the normal distribution. In recent literature, log-transformed time series data is modeled using a linear AR model, followed by the inverse transformation of the forecasts \cite{lutkepohl2012role}. However, this transformation changes the symmetric measurement errors on the original scale to asymmetric errors on the log scale because the linear fit is performed on the log-scaled data. Log transformation is also highly impacted by outliers or peaks in the time series datasets visible in most epidemic data. Another popularly used transformation in time series literature is the Fourier transformation. Although Fourier transforms are ideal for periodic signals, their performance for non-periodic signals and signals with changing characteristics over time (i.e., non-stationary time series) is unsatisfactory as this transformation will generally give the averaged data. Hence, the direct use of Fourier transformation to pre-process the non-stationary real-world epidemic signals is avoided \cite{brunton2022data}. 
To overcome this problem, wavelet transform has been considered as an efficient mathematical tool for the past three decades \cite{percival1997analysis,percival2000wavelet,walden2001wavelet}. Wavelet transformations are in many ways a generalization of the Fourier transform that allows the independent choice of time and frequency resolution at different times and frequencies \cite{brunton2022data}. The ability of the wavelet transformation to decompose the original series into many high and low-frequency coefficients allows for the appropriate extraction of signal from noise \cite{percival2000wavelet}. In the literature, most wavelet decomposition included a discrete wavelet transform (DWT) followed by a statistical or machine learning approach to generate forecast \cite{mabrouk2008wavelet, saadaoui2019wavelet, zhu2014modwt, chakraborty2019forecasting}. However, the restriction on signal length imposed by the DWT approach led to the application of a maximal overlap discrete wavelet transform (MODWT), which has similar properties to DWT but is free from the limitations \cite{percival2000wavelet}.
Moreover, the MODWT approach provides increased resolution for noisy data, and unlike DWT, the number of coefficients at each level is the same as that of the original series. Applications of the MODWT-based autoregressive moving average (ARMA) model and hybrid ARIMA-WARIMA (based on error correction approach) have been proposed for meteorological forecasting and epidemic forecasting \cite{zhu2014modwt,chakraborty2019forecasting}. Recent studies have also focused on the application of MODWT-based deep learners, Wavelet Transformers (W-Transformers) and Wavelet NBeats (W-NBeats) for modeling real-world time series and stock-price datasets, respectively \cite{sasal2022w, singhal2022fusion}. Several studies have also attempted to model MODWT decomposed coefficients with an artificial neural network for predicting electricity price \cite{saadaoui2019wavelet}, generating weather forecasts \cite{nury2017comparative}, analyzing the wholesale price of agricultural commodities \cite{anjoy2019comparative}, forecasting the occurrence of flood \cite{nanda2016wavelet}, and foretelling the daily river discharge \cite{quilty2021maximal}. These studies suggest that the wavelet-based neural network model generates more accurate forecasts than the multilayered perceptrons. However, these wavelet neural networks \cite{alexandridis2014wavelet} have less application in the epidemic incidence prediction owing to the unavailability of a vast amount of historical data and discrepancy regarding the choice of the hidden neurons in wavelet neural network resulting in an unstable learning algorithm. Another major disadvantage of the previously built wavelet neural networks is that they lack the desired theoretical properties like asymptotic stationarity, which makes long-term forecasts unstable and inaccurate. To mitigate these concerns, this paper attempts to design a novel ensemble of wavelet neural networks, and we call it EWNet, that can handle epicasting problems and generate short, medium, and long-term forecasts that are more reliable and accurate as compared to state-of-the-art methods from statistics and machine learning literature. EWNet is first built theoretically with the help of the MODWT algorithm combined with ARNN models in an ensemble setup and further used to solve the epicasting problem. More precisely, our proposed EWNet model initially decomposes the epidemic datasets into several ``details'' (describing high-frequency variations at a particular time scale) and ``smooth'' (describing low-frequency variations) using a MODWT-based additive decomposition. In the subsequent step, EWNet models the ``details'' and ``smooth'' segments of the data with a series of autoregressive feedforward neural networks having pre-defined architecture specified in the theoretical sections. Finally, an ensemble approach is applied to ensure the reduction of bias in the overall forecast.

The main contributions of the paper can be summarized in the following manner:
\begin{enumerate}
\item We present a novel formulation of the proposed EWNet model designed to handle nonlinear, non-stationarity, and long-range dependency of real-world epidemic datasets. We analyze several theoretical properties of the proposed framework, including asymptotic stationarity, ergodicity, irreducibility, and learning stability. 
\item The proposed EWNet model has a solid mathematical basis and is more explainable and reliable than modern deep learning techniques. In addition, the model does not have growing variance over time and exhibits better long-range forecastability for epidemic datasets. 
\item From a practitioner's viewpoint, we extensively study the global characteristics of fifteen real-world infectious disease datasets covering influenza, malaria, dengue, and hepatitis B from different regions. We demonstrate the epicasting ability of the proposed EWNet model on all the fifteen epidemic datasets by a rolling window approach having three test horizons - short, medium, and long-term and measure their performance using four accuracy metrics, namely Root Mean Squared Error (RMSE), Mean Absolute Error (MAE), Mean Absolute Scaled Error (MASE), and symmetric Mean Absolute Percent Error (sMAPE).
\item We check the efficacy of the proposed model by comparing its performance indicators with a total of 22 state-of-the-art forecasters ranging from traditional time series models to the most recent deep learning algorithms. We show that our proposal can generate a better long-term forecast and outperform most forecasters on average. Moreover, we report the robustness of the forecast generated by our proposed EWNet method using a non-parametric test. Finally, the statistical significance of the experimental results and the potential threats to validate the results provide a strong justification for the multi-disciplinary usability of the proposed EWNet model in future studies.
\end{enumerate}

The remaining sections of this paper are structured as follows. Section \ref{ensemble} provides a detailed description of the formulation of the proposed EWNet model. Then, in Section \ref{prop_EWNet}, we provide the statistical properties of the proposed EWNet model describing its stable learning, geometric ergodicity, and asymptotic stationarity, along with the practical implications of these theoretical results. A detailed summary of the real-world epidemic data characteristics, performance measures used in this study, and forecast evaluation of the proposed methodology with other state-of-the-art forecasters are provided in Section \ref{sec:result}. Finally, Section \ref{stat_signif_test} evaluates the statistical significance of the improvements in forecasts due to the application of the proposed EWNet model and investigates the unexpected threats to the validity of these results. We conclude this paper in Section \ref{sec:Discussion} with some discussion and the future scope of research.

\section{Method}\label{ensemble}
This section gives an overview of the maximal overlap discrete wavelet transform (MODWT) approach. We then present the detailed formulation of the EWNet model. The key of the ensemble wavelet neural network (EWNet) model is the wavelet decomposition of time series and the construction of an ensemble of autoregressive neural networks.

\subsection{\label{formulation}Wavelet Transformations and DWT Approach}

In our study, we utilize a discrete wavelet transformation (DWT) approach to denoise the epidemiological data (time-indexed) followed by an autoregressive neural network architecture \cite{faraway1998time}. In particular, we concentrate on `maximal overlapping' versions of DWT that are applicable for arbitrary time series. DWT represents a signal using an orthonormal basis representation that has been widely used in smoothing signals \cite{percival2000wavelet, walden2001wavelet}, compressing digital images \cite{hilton1994compressing}, geophysics \cite{zhu2014modwt}, atmosphere \cite{percival1997analysis}, economics \cite{anjoy2017hybrid}, energy \cite{yang2021forecasting}, and material science \cite{li2020comparative} among many others. We start with the description of wavelets and the DWT approach that can create a basis for the MODWT algorithm to be used in the proposal. 

The Daubechies wavelets \cite{daubechies1992ten} are a family of orthogonal wavelets defining a discrete wavelet transform. We consider discrete compactly supported filters of Daubechies class of wavelets here. We denote by $\{g_m : m = 0,1,\ldots,M-1\}$ the scaling filters and $\{h_m : m = 0,1,\ldots,M-1\}$ the wavelet filters. We restrict the scaling filter and wavelet filter to satisfy unit energy assumptions (refer to Eqn. \ref{eq_prop_1}) and even-length scaling assumptions (refer to Eqn. (\ref{eq_prop_2})) defined as follows:
\begin{equation}
    \sum_{m=0}^{M-1}g_m^2 = \sum_{m=0}^{M-1}h_m^2 = 1
    \label{eq_prop_1}
\end{equation}
\begin{equation}
    \sum_{m=0}^{M-1}g_m g_{m + 2n} =\sum_{m=0}^{M-1}h_m h_{m + 2n} = 0
    \label{eq_prop_2}
\end{equation}
for all non-zero and integer $n$. These two properties together are called the \textit{``orthonormality property"} in wavelet literature \cite{percival2000wavelet}.
Scaling and wavelet filters are also related by the following restriction:
\begin{equation*}
    g_m \equiv (-1)^{m+1} h_{M-1-m} \; \; 
    \textrm{ or } \; \; h_m \equiv (-1)^m g_{M-1-m}; \; 
    \text{for} \; \; m = 0, 1, \ldots, M-1.
\end{equation*}
Thus, we call $\{g_m\}$ as ``quadrature mirror" filter corresponding to $\{h_m\}$. The construction scheme of DWT coefficients is well known as the `pyramid algorithm' \cite{percival1997analysis}.

Suppose we denote the epidemic time series to be transformed by $Y = \{Y_t: t=0, 1, \ldots, N-1\}$. With $V_{0,t} \equiv Y_t$, the $j^{th}$ stage input to the pyramid algorithm is $\{V_{j-1,t}:t=0,\ldots,N_{j-1}-1\}$, where $N_j = {N}/{2^j}$. In the DWT pyramid algorithm, $j^{th}$ stage outputs are the $j^{th}$ level wavelet and scaling coefficients and these $j^{th}$ level coefficients can directly be linked to the series $\{Y_t\}$, following \cite{walden2001wavelet}. 
\begin{equation*}
    U_{j,t} = \sum_{m=0}^{M_j - 1} h_{j,m} Y_{(2^j(t+1)-1-m)\text{ mod }N} \; \; \textrm{and} \; \; V_{j,t} = \sum_{m=0}^{M_j - 1} g_{j,m} Y_{(2^j(t+1)-1-m)\text{ mod }N}; 
\end{equation*}
where the $j^{th}$ level filters have the same unit energy and related properties as discussed in Eqn. (\ref{eq_prop_1}) and Eqn. (\ref{eq_prop_2}) along with
\begin{equation*}
    \sum_{m=0}^{M_j-1} g_{j,m} = 2^{{j}/{2}}  \textrm{ and } \sum_{m=0}^{M_j-1} h_{j,m} = 0.
\end{equation*}
At level $j$ the nominal frequency band to which the corresponding wavelet coefficients $\{U_{j,t}\}$ is given by $|l| \in \left(\frac{1}{2^{j+1}},\frac{1}{2^j} \right)$. However, DWT restricts the sample size to be exactly a power of $2$, whereas wavelet details and scaling coefficients of a DWT decomposed signal do not scale and are shift-invariant. We may overcome these deficiencies of DWT by using a modified version of DWT, namely the maximal overlap discrete wavelet transformation (MODWT) based on haar filter \cite{percival1997analysis, percival1995estimation}.

\subsection{\label{theoretical}MODWT algorithm}
The MODWT is an improved and modified version of the DWT algorithm. Both DWT and MODWT allow to perform a multi-resolution analysis which is a scale-based additive decomposition \cite{nason1999wavelets}. However, the MODWT algorithm overcomes the deficiencies of the DWT algorithm and can handle the circular shift in the signal. Thus it is best suited for decomposing epidemiological time series that exhibit non-stationary seasonal patterns. Several applications of MODWT in time series analysis can be found in \cite{zhu2014modwt, anjoy2019comparative, zhang1992wavelet}. Therefore, in our study, we consider MODWT, which is well-defined for all sample sizes and shift-invariant. This is also called \textit{nondecimated wavelet transform},
as there is a redundancy of wavelet and scaling coefficients at each decomposition level of the original series following a
particular pattern. A mathematical formulation of MODWT can be extended directly from the DWT formulation in Section \ref{formulation}.

Here, we define MODWT filters $\{\tilde{h}_{j,m}\}$ and $\{\tilde{g}_{j,m}\}$ by re-normalizing the DWT filters:
\begin{equation}\label{eq:3}
    \tilde{h}_{j,m} = \frac{h_{j,m}}{2^{j/2}} \; \textrm{ and } \; \tilde{g}_{j,m} = \frac{g_{j,m}}{2^{j/2}};
\end{equation}
and width $M_j$ of MODWT and DWT are the same. Another modification made w.r.t. the DWT filter is that MODWT filters do not have unit energy, i.e.,
\begin{equation*}
    \sum_{m=0}^{M_j-1}\tilde{h}_{j,m}^2 = \sum_{m=0}^{M_j-1}\tilde{g}_{j,m}^2 = \frac{1}{2^j},
\end{equation*}
and, therefore, there is no need for downsampling by $2^j$ in the MODWT. With $\tilde{V}_{0,t}\equiv Y_t$, then the MODWT pyramid algorithm generates the MODWT wavelet coefficients $\{\tilde{U}_{j,t}\}$ and the MODWT scaling coefficients $\{\tilde{V}_{j,t}\}$ \cite{percival2000wavelet}. These coefficients can also be formulated in terms of filtering of $\{Y_t\}$, using the filters as in Eqn. (\ref{eq:3}):
\begin{equation*}
  \tilde{U}_{j,t} = \sum_{m=0}^{M_j-1}\tilde{h}_{j,m} Y_{(t-m) \textrm{ mod }N} \; \textrm{ and } \;     \tilde{V}_{j,t} = \sum_{m=0}^{M_j-1}\tilde{g}_{j,m} Y_{(t-m) \textrm{ mod }N};  
\end{equation*}
where $M_j = (2^j -1)(M-1)+1$. Similar to DWT, the MODWT coefficients at level $j$ are associated to the same nominal frequency band $|f_q| \in  \left(\frac{1}{2^{j+1}}\frac{1}{2^j}\right]$ and are defined as the convolutions of the time series $Y_t$. Thus, the wavelet coefficients at each level will have the same length as that of the original series. The coefficients can also be expressed using matrix notation as follows \cite{percival2000wavelet}:
\begin{equation*}
  \tilde{U}_j = \tilde{u}_j Y \; \textrm{ and } \; \tilde{V}_j = \tilde{v}_j Y,
\end{equation*}
where the square matrices $\tilde{u}_j$ and $\tilde{v}_j$ of order $N \times N$ comprises values dictated by wavelet filters and  scaling filters, respectively.
\begin{equation}\label{bmatrix}
    \tilde{u}_j = \begin{bmatrix}
    \tilde{h}_{j,0} & \tilde{h}_{j,N-1} & \tilde{h}_{j,N-2} & \ldots & \tilde{h}_{j,3} & \tilde{h}_{j,2} & \tilde{h}_{j,1}\\
    \tilde{h}_{j,1} & \tilde{h}_{j,0} & \tilde{h}_{j,N-1} & \ldots & \tilde{h}_{j,4} & \tilde{h}_{j,3} & \tilde{h}_{j,2}\\
    \vdots & \vdots & \vdots & \ldots & \vdots & \vdots & \vdots \\
    \tilde{h}_{j,N-2} & \tilde{h}_{j,N-3} & \tilde{h}_{j,N-4} & \ldots & \tilde{h}_{j,1} & \tilde{h}_{j,0} & \tilde{h}_{j,N-1} \\
    \tilde{h}_{j,N-1} & \tilde{h}_{j,N-2} & \tilde{h}_{j,N-3} & \ldots & \tilde{h}_{j,2} & \tilde{h}_{j,1} & \tilde{h}_{j,0} \\
    \end{bmatrix}
\end{equation}
and $\tilde{v}_j$ can similarly be expressed as in Eqn. (\ref{bmatrix}) with each $\{\tilde{h}_{j,m}\}$ replaced by $\{\tilde{g}_{j,m}\}$. Thus, the original series $(Y)$ can be written from its MODWT based via,
\begin{equation*}
    Y = \sum_{j=1}^J \tilde{u}_j^T \tilde{U}_j + \tilde{v}_J^T \tilde{V}_J = \sum_{j=1}^J D_j + S_J,
\end{equation*}
where $D_j = \tilde{u}_j^T \tilde{U}_j$ is the $j^{th}$ level $(j = 1,2,\ldots,J)$ details and $S_J = \tilde{v}_J^T \tilde{V}_J$ is the $J^{th}$ level smooth of the MODWT decomposition. 
A more detailed description and pseudo-code of the MODWT algorithm is available in \cite{percival1997analysis}. MODWT is valid for any integer $N$, whereas DWT needs $N$ to be an integer multiple of $2$. Also, MODWT is a more handy tool for handling non-stationary and seasonal discrete time series, which is the case in most epidemic datasets. These properties of MODWT are a key element for pre-processing highly non-stationary and long-term dependent epidemic datasets. The remaining nonlinearity of the epidemic time series is further modeled with the ARNN model in the proposed EWNet framework. For graphical illustration, the MODWT decomposition on the \href{https://github.com/JohannHM/Disease-Outbreaks-Data/blob/master/Colombia_Dengue.dat}{Colombia Dengue} dataset is presented in Fig. \ref{MODWT_ts}. We aim to create a new set of random variables (equal-sized time series named as details and smooth coefficients of MODWT algorithm) and use them to build a novel ensemble of autoregressive neural nets in the EWNet framework. In the next subsection, we combine the MODWT algorithm and ARNN model to utilize their complimentary benefits for \textit{epicasting}.  

\begin{figure}[t]
    \centering
    \includegraphics[scale=0.30]{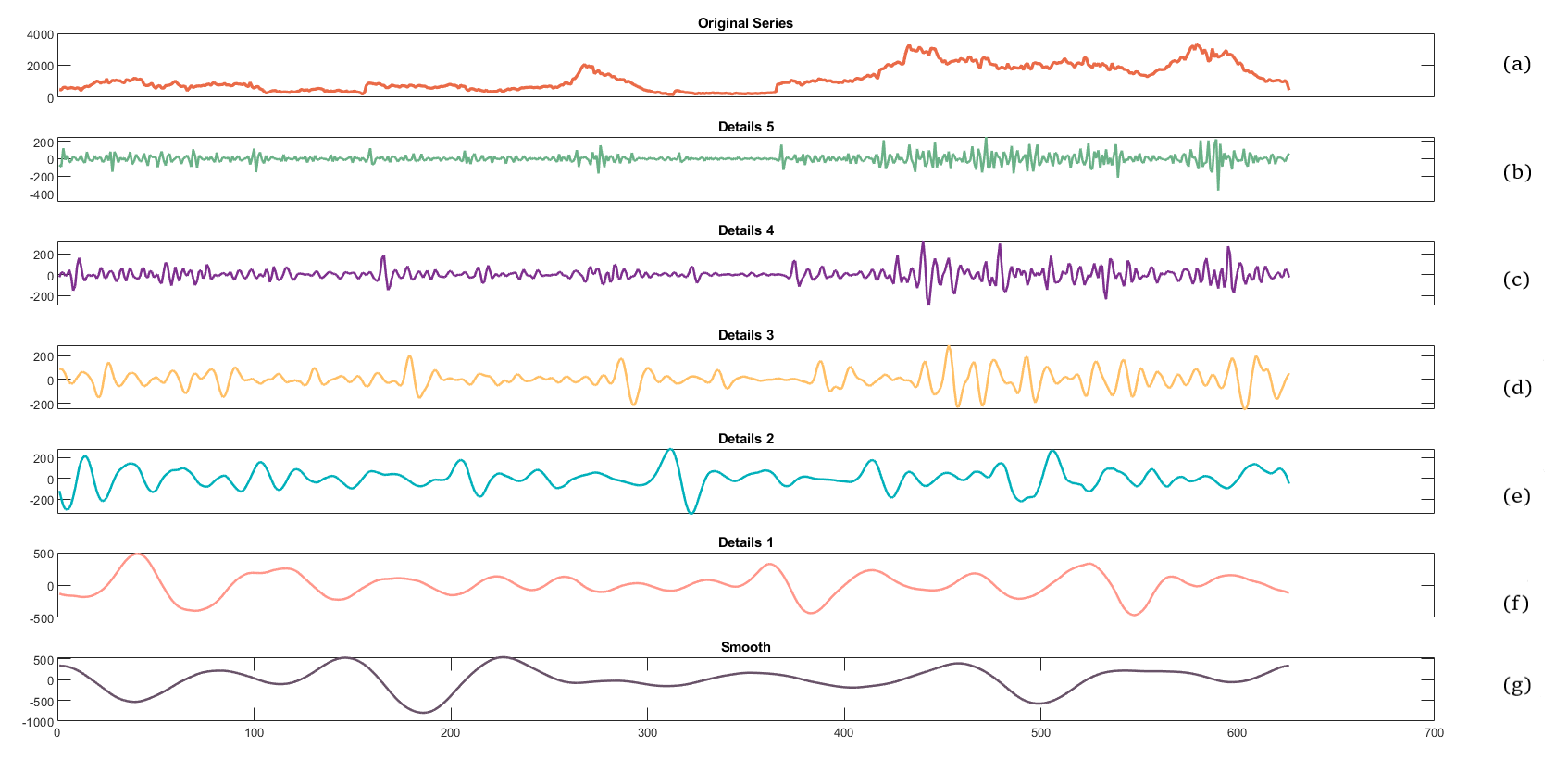}
    \caption{MRA-based MODWT decomposition of the Colombia dengue dataset with the original epidemic time series and its 6 levels. In Figure, (a) denotes the original time series in actual frequency scale; (b)-(f) denote the detail coefficients reproduced by the MODWT algorithm with haar filter, and (g) represents the scaling coefficients of the series generated by MODWT algorithm with haar filter. The figure depicts time-localized information on frequency patterns that are identified by wavelets.}
    \label{MODWT_ts}
\end{figure}

\subsection{\label{Proposed_Model}Proposed EWNet model}
This section provides a detailed formulation of our proposed EWNet methodology that utilizes a wavelet decomposition algorithm as a data pre-processing step. A salient feature of the MODWT algorithm is that it helps to decompose epidemic time series in trend and higher frequency bands which are exploited for forecasting in the proposal. The multiresolution analysis of MODWT decomposes the discrete time series $Y_t \; (t=1,2,\ldots,N)$, where $N$ is the number of historical samples, into high-frequency information and low-frequency information by applying corresponding filters. These high and low-frequency decomposed series are termed wavelet (details), and scaling (smooth) coefficients can track the original series as:
\begin{equation*}
    Y_t = \sum_{j=1}^J D_{j,t} + S_{J,t}.
\end{equation*}
Theoretically, $D_{j,t} \; (j=1,2,\ldots, J)$ components capture the non-smooth bumpy details (local fluctuations) of the series $Y_t$, indicated by the fast dynamics whereas its counterpart $S_{J,t}$ apprehends the smooth tendencies (overall ``trend'' of the original signal) of the series, signalized by slow dynamics. Epidemic time series considered in this study have long-term memory (as reported in Table \ref{Global_cht}), and long-term memory processes have a high degree of correlation. With the help of MODWT (with `haar' filter), we create a new set of random variables (equal-sized time series), namely, the wavelet coefficients, that are approximately uncorrelated (both within and between scales). The decomposition process can be iterated, with successive approximations being decomposed in turn, so that the original signal is broken down into many lower-resolution
components. Simultaneously, the problem of generating forecasts $\hat{Y}_{N+h}$ (h-step ahead forecasts) based on ${Y_1, Y_2, \ldots, Y_N}$ can be solved by generating the forecasts $\hat{D}_{j,N+h} \; (j = 1, 2, \ldots, J)$ and $\hat{S}_{J,N+h}$, based on their previous observations, i.e., 
\begin{align*}
 \hat{D}_{j,N+h} & = f(D_{j,1},D_{j,2},\ldots,D_{j,N}) ; \; j= 1,2, \ldots, J, \\
 \hat{S}_{J,N+h} & = f(S_{J,1}, S_{J,2}, \ldots, S_{J,N}),
\end{align*}
where $f$ is the autoregressive neural network function. We choose the value of $J+1$ as a floor function of $\log \; (\text{base} \; e)$ of the length of training subset as suggested by \cite{percival1997analysis}.

In our proposed framework, we utilize these decomposed time series using an ensemble of neural networks for generating the forecasts from several decomposed components. The neural net comprises of three layers - one input layer with $p$ nodes, one hidden layer with $k$ nodes, and an output layer with no recurrent connections (feedforward structure). We operate $J+1$ of these feedforward neural networks, each of which models $p$ lagged observations from a series to generate a one-step-ahead forecast in a single iteration.
\begin{align*}
 \hat{D}_{j,N+1} &=
 \alpha_{0,j} + \sum_{i=1}^k \beta_{i,j} \phi(\alpha_{i,j} + \beta_{i,j}^{'} \underbar{D}_{j}) ; \;  j =1, 2, \ldots, J, \\
 \hat{S}_{J,N+1} &= \eta_0 + \sum_{i=1}^k \delta_i \phi(\eta_i + \delta_i^{'} \underbar{S}_{J}) ; 
\end{align*}
where $\underbar{D}_{j} $, $\underbar{S}_{J}$ denotes $p$ lagged observations of the corresponding decomposed series $(j = 1,2, \ldots, J)$, $\alpha_{0,j}$, $\eta_0$, $\alpha_{i,j}$, $\beta_{i,j}$, $\eta_i$, $\delta_i$ $(i  = 1,2,\ldots,k; j = 1,2,\ldots,J)$ are the connection weights of the network, $\beta_{i,j}^{'}, \delta_i^{'}$ are $p$ dimensional weight vectors, and $\phi$ is the nonlinear activation function (precisely, logistic sigmoidal activation function). The weights of the network take random values at the beginning and are then trained by gradient descent back-propagation approach \cite{rumelhart1986learning}.
This procedure is continued iteratively until the forecast of the desired horizon is obtained. Eventually, the forecasts originating from all the trained networks are aggregated to produce the final forecast as 
\begin{equation*}
    \hat{Y}_{N+h} = \sum_{j=1}^J \hat{D}_{j,N+h} + \hat{S}_{J,N+h}
\end{equation*}
The choice of the hyperparameter $p$ is based on the minimization of forecast error for the validation set in a cross-validation way
\begin{equation*}
    p = \underset{p}{argmin} \frac{1}{|\textit{V}|}\sum_{t \in \textit{V}} \frac{2|\hat{Y}_t - Y_t|}{|\hat{Y}_t|+|Y_t|}*100\%,
\end{equation*}
where $Y_t$ is the series at time point $t$, $\hat{Y}_t$ is the predicted value at time point $t$, $\textit{V}$ is the validation set and the number of neurons $k=[\frac{p+1}{2}]$ in the hidden layer is chosen (proof is discussed in Section \ref{Stable_Learning}). Detailed descriptions of the EWNet model parameters are described below.
\begin{enumerate}
\item \textit{Wavelet levels $(J+1)$:} An integer value specifying the level of the wavelet decomposition of the original series. \\
In order to account for the maximum level in the decomposition, we set $J+1 = \lfloor \log_e N \rfloor$ based on the recommendation of \cite{percival1997analysis}.
\item \textit{Fast Flag:} Denotes the wavelet decomposition  achieved by using \textit{pyramid algorithm} described in \cite{percival2000wavelet}.
\item \textit{Boundary:} A ``periodic'' boundary is set and it is used to  obtain coefficients from the training time series. 
\item \textit{MaxARParameter $(p)$:} An integer indicating the value of the lagged inputs in each of the $J+1$ ARNN models. 
This is a tuning parameter in EWNet and is chosen using cross-validation.
\item \textit{Hidden neurons $(k)$:} The number of hidden neurons in $(J+1)$ ARNN models are set to $k = \left[\frac{p+1}{2}\right]$ (discussed in details in Section \ref{Stable_Learning}). 
\item \textit{NForecast $(h)$: } The desired forecast horizon.
\end{enumerate}

\begin{figure}
    \centering
    \includegraphics[scale=0.40]{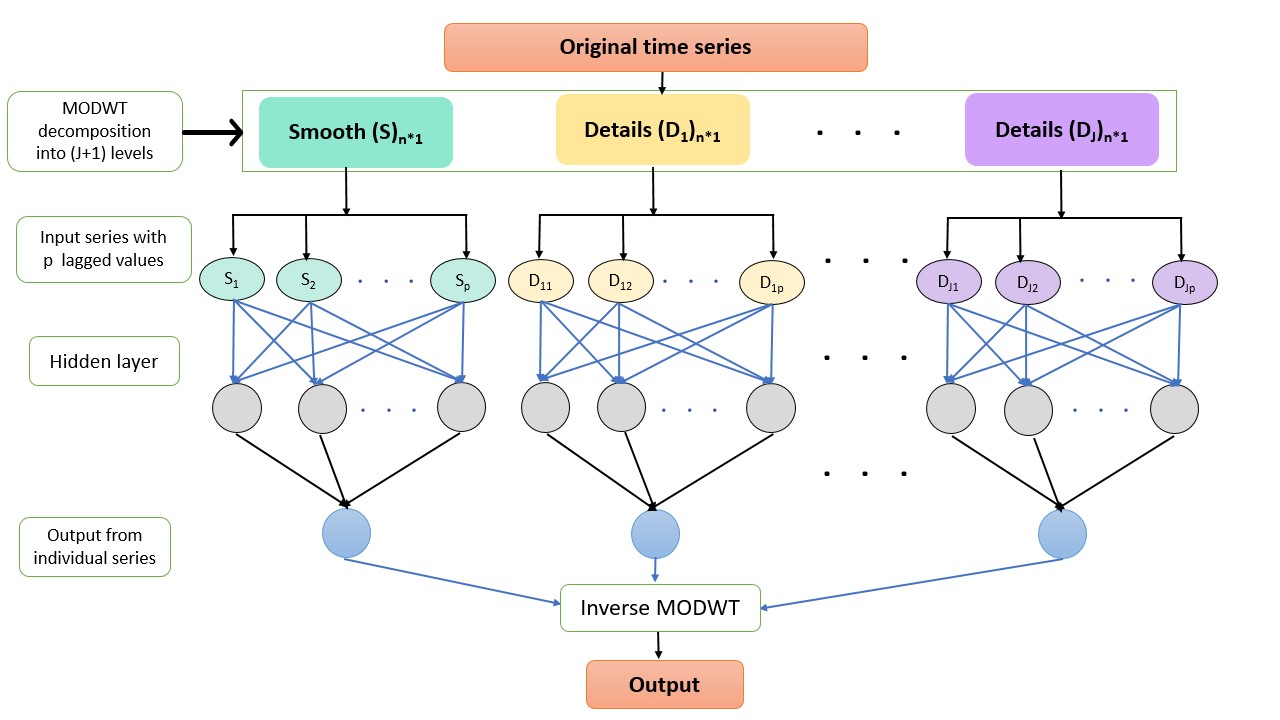}
    \caption{Schematic diagram of the EWNet framework: Given the original input series of size $n$, we employ MODWT transformation to decompose the series into one smooth and $J$ details coefficients each of size $n$. In the subsequent step, each of the transformed series is modeled with an autoregressive neural network and their forecasts are combined via inverse MODWT transformation to generate the one-step ahead ensemble forecast.}
    \label{EWNet_Model}
\end{figure}
A schematic flow diagram of the proposed EWNet model is portrayed in Fig. \ref{EWNet_Model}.
A detailed inspection of Fig. \ref{EWNet_Model} describes the mechanism of generating a one-step-ahead forecast in the proposed EWNet model, where each wavelet decomposed series is modeled with autoregressive neural network architecture. Using one-step ahead forecasts, we iteratively find the multi-step ahead forecasts from the EWNet model. Based on the non-stationary and nonlinear characteristics of the time series, we apply MODWT-based decomposition to break the series into multiple sub-frequencies. Following this, each detail and smooth component is fed into an ARNN model for prediction purposes. The wavelet analysis can efficiently diagnose the main frequency components of the signal, and the ARNN can now model the details and smooth components of the series with higher accuracy; thus, the name of the model, Ensemble Wavelet Neural Network (EWNet), is justified. In the proposed model, the time series is first decomposed into several sub-time series $[D_1, D_2, \ldots, D_J, S_J]$, where the former $J$ series are the wavelet detail components, and $S_J$ is the smooth component as depicted for the Colombia dengue data in Fig. \ref{MODWT_ts}. Finally, the forecasted series is formed through inverse wavelet transform from the forecast generated by the details and smooth components. A detailed description concerning the implementation of the model is available in Algorithm \ref{algo}. So far as the study proceeds, in the following section, we develop the theoretical results of the proposed EWNet model from a nonlinear time series viewpoint and show the stability in the learning of our proposal, asymptotic behavior, and their practical implications. 

\begin{remark}
    Most machine learning and deep learning frameworks utilize a sliding window approach to reconstruct the time series forecasting task as a supervised learning problem. The previous steps are used as inputs, and the next step as the outputs. However, in the proposed EWNet architecture, we employ an ensemble ARNN framework on the MODWT decomposed training series. Unlike most machine learning and deep learning approaches, the proposed model does not reconstruct the epidemic series into an input-output supervised framework; instead, it utilizes p-lagged observations of each of the wavelet decomposed training data and generates a one-step-ahead forecast using a nonlinear function as discussed in Algorithm \ref{algo}. Moreover, we recursively update the training data with the latest forecast (obtained from EWNet) to develop the multi-step ahead forecasts for each transformed series using the same nonlinear activation function. Finally, we consider an ensemble of the forecasts generated from each wavelet decomposed series and obtain our desired results. In the experimental evaluation, we utilize the original test data only to compute the forecasting accuracy of the proposed EWNet framework in comparison with benchmark methods. 
\end{remark}

\begin{algorithm}[H]
    \SetKwInOut{KwIn}{Input}
    \SetKwInOut{KwOut}{Output}
    \KwIn{Univariate time series $\{Y_1, Y_2, \ldots, Y_N\}$ with $N$ historical observations.}
    \KwOut{Record prediction corresponding to the historical data window, fitted values of the original series, and $h$-step ahead forecast ($h$ to be specified by user).}
    \vspace{0.2cm}
{\textbf{Train Procedure:}\\
\nl  Compute the maximal overlap discrete wavelet transform (MODWT) of the original time series via pyramid  algorithm. \\ 
\nl Extract the wavelet and scaling coefficients corresponding to each level and transform them to time series objects.\\
\nl Model these individual time series using an autoregressive neural network with $p$ lagged values. \\
\nl Select the MaxARParam corresponding to the minimum accuracy measure (MASE) on the validation set and $k$ as specified. \\
\vspace{0.1cm}
\textbf{Test Procedure:}\\
Execute the previously mentioned steps for acquiring the forecast on the hold-out test set. \\
\nl The fitted model generates a one-step-ahead prediction. \\
\nl Iterate the process until the forecast of the desired horizon is computed. \\
\nl Combine the final forecasts generated from the wavelet and scaling series using an inverse MODWT approach to achieve the desired output. }\\
\vspace{0.2cm}
\caption{{\bf Proposed EWNet model}
\label{algo}
}
\end{algorithm}

\section{\label{prop_EWNet} Statistical Properties of EWNet Model}
In this section, we explore several theoretical aspects of the proposed EWNet approach and discuss their practical implications from practitioners' points of view. We start with the learning stability problem of EWNet and then investigate the asymptotic behavior of the associated Markov chain. 

\begin{subsection}{\label{Stable_Learning}Stable Learning using EWNet model}
We investigate the effect of learning stability and the choice of hidden neurons in the EWNet model. In unstable neural network models, the number of hidden nodes in hidden layers either becomes too large or too small. This instability in the neural network gets reflected in the output layer of the neural net, and a trade-off is required. Several previous studies established theoretical results on the choice of hidden neurons of feedforward neural network, for example, see \cite{tamura1993determination, zeng2006hidden, chakraborty2019nonparametric}. In the proposed EWNet, we consider the following assumptions to ensure learning stability in the proposed ensemble framework. 
\begin{enumerate}[label = (\alph*)]
    \item EWNet has three layers: one input, one hidden, and one output layer with no recurrent connections (feedforward structure). Also, there is no direct connection from the input to the output layer in EWNet.
    \item Gradient descent backpropagation \cite{hinton2006reducing} learning is used without introducing an inertia term to train the EWNet model.
    \item EWNet starts with random weights, and the network is mainly trained for one-step forecasting, although multi-step ahead forecasts can also be computed recursively.
    \item We further assumed that the learning rate $\left( \eta \right)$ is the same for all the connections and connection weights $\left(w^{(o)}\right)$ and error signal $\left(\delta^{(o)}\right)$ in the output layer are assumed to have a symmetrical distribution with respect to the origin.
    \item The number of lagged inputs $(p)$ in EWNet$(p,k)$ model is selected by a grid search optimization algorithm and the number of hidden neurons is set to $k = \left[\frac{(p+1)}{2}\right]$ unless it is particularly specified. The above choice of $k$ provides stability of learning in the proposed EWNet model. 
\end{enumerate}
Assumptions (a) - (d) are trivially true. But, the assumption (e) is critical, and we discuss below the choice of hidden neuron and stability of learning for the EWNet$(p,k)$ model. Throughout this section, we denote the triplet notion $(i,h,o)$ as the (input, hidden state, and output) of the EWNet model. The change of internal state $\Delta u$ through learning for the same input patterns is considered a rough standard for the stability of learning in the proposed EWNet, as previously described in seminal papers on statistical properties of neural networks \cite{white1989learning, hornik1993some}. The change in weights from the $i^{th}$ input to the $\tilde{j}^{th}$ hidden neuron can be mathematically written as
\begin{equation*}
    \Delta w_{\tilde{j}i}^{(h)} = \eta \delta_{\tilde{j}}^{(h)} x_i^{(i)},
\end{equation*}
where, $x_i^{(i)}$ is the output of the $i^{th}$ input neuron, $\delta_{\tilde{j}}^{(h)}$ is the propagated error signal for the $\tilde{j}^{th}$ hidden neuron and can be mathematically expressed as
\begin{equation*}
    \delta_{\tilde{j}}^{(h)} = \frac{\partial E}{\partial u_{\tilde{j}}^{(h)}},
\end{equation*}
where $E$ is the $L_2$-error loss between the training signal $y_l$ and output value $x_l^{(o)}$. The change in the internal state can be written as 
\begin{equation*}
    \Delta u_{\tilde{j}}^{(h)} = \sum_{i=1}^p \Delta w_{{\tilde{j}}i}^{(h)} x_i^{(i)} = \eta \delta_{\tilde{j}}^{(h)} \sum_{i=1}^p \left(x_i^{(i)}\right)^2.
\end{equation*}
The propagated error signal $\delta_{\tilde{j}}^{(h)}$  in the hidden layer is computed as
\begin{equation}
    \delta_{\tilde{j}}^{(h)} = g^\prime \left(u_{\tilde{j}}^{(h)}\right) w_{l{\tilde{j}}}^{(o)} \delta_l^{(o)},
    \label{Eq_Stable}
\end{equation}
where $g^\prime(\cdot)$ is the derivative of the activation function (logistic sigmoidal activation function of EWNet model is both continuous and differentiable), $\delta_l^{(o)}$ is the output error signal, and $w_{l{\tilde{j}}}^{(o)}$ is the output weights. Accordingly, $\delta_{\tilde{j}}^{(h)}$ is inversely proportional to the number of hidden neurons and can be computed using Eqn. (\ref{Eq_Stable}). Under the standard regularity condition that $g^\prime \left(u_{\tilde{j}}^{(h)}\right)$ and $w_{l{\tilde{j}}}^{(o)} \delta_l^{(o)}$ are independent \cite{fujita1998statistical, hornik1993some}, the variance of $\delta_{\tilde{j}}^{(h)}$, denoted by $\mathbb{V}\left(\delta_{\tilde{j}}^{(h)}\right)$, is mathematically represented as
\begin{align*}
    \mathbb{V}\left(\delta_{\tilde{j}}^{(h)}\right) &= \mathbb{E}\left[g^\prime \left(u_{\tilde{j}}^{(h)}\right)w_{l{\tilde{j}}}^{(o)} \delta_l^{(o)} - \mathbb{E}\left(g^\prime \left(u_{\tilde{j}}^{(h)}\right)w_{l{\tilde{j}}}^{(o)} \delta_l^{(o)}\right)\right]^2\\
    &= \left[\mathbb{E}^2\left(g^\prime \left(u_{\tilde{j}}^{(h)}\right)\right)+ \mathbb{V}\left(g^\prime \left(u_{\tilde{j}}^{(h)}\right)\right)\right]\mathbb{E}\left[w_{l{\tilde{j}}}^{(o)} \delta_l^{(o)}\right]^2.
\end{align*}
The boundary of the stable learning of a hidden neuron is summarized as $\eta.\left(\frac{p}{k}\right)$ by adding the effect of learning rate $\eta$ to the above discussion. The number of hidden neurons becoming too large can make the output neurons unstable, whereas if the number of hidden neurons becomes too small, the hidden neurons become unstable again. Here a trade-off is derived for the learning structure of the EWNet algorithm. We introduce a balancing equation as follows:
\begin{equation}
    \alpha \eta.\left( \frac{p}{k} \right)= \eta. k,
    \label{Eq_Stable_2}
\end{equation}
where the L.H.S. and R.H.S. of Eqn. (\ref{Eq_Stable_2}) are obtained from the viewpoint of the boundary of stable learning in hidden and output neurons, respectively. Here, we also pose $\alpha$ as a constant for consistency. Therefore, we initially choose the number of hidden neurons to be $k = \sqrt{\alpha. p}$.
We take the minimum value of $\alpha$ to be $1$ and the maximum value of $\alpha$ to be $p \; (\geq 1)$. Thus, $k$ lies between $\sqrt{p}$ and $p$ for stable learning in the EWNet model. A natural choice of $k \in (\sqrt{p},p)$ is $[\frac{p+1}{2}]$, can be easily derived using  AM-GM inequality. Thus, we conclude that the network structure proposed in the EWNet model has stable learning property that is desired from the statistical perspective. Next, we prove the asymptotic stationarity of the associated stochastic process from a nonlinear time series point of view, following \cite{meyn2012markov}. 
\end{subsection}

\subsection{Geometric Ergodicity and Asymptotic Stationarity \label{implementation}}
The proposed ensemble wavelet-based autoregressive neural network (EWNet) model is an integrated approach that combines wavelet transformation with the ARNN algorithm. First, the wavelet decomposition coefficients for time series data are transported into the ARNN model to set up a forecast ensemble in the proposed framework. Wavelet transformation decomposes a time series into $J+1$ independent orthogonal components with both time and frequency localization. Then, we fit several specific autoregressive neural network models to the component series and obtain forecasts later aggregated to get the actual predictions and, after that, out-of-sample forecasts. Therefore, we only need to show that, under the sufficient conditions stated below, a single autoregressive neural network process is ergodic and asymptotically stationary to ensure that the whole process is ergodic and asymptotically stationary. 

\par We start with a simple ARNN$(1,k)$ process with $k$ hidden units that can be defined by the following stochastic differential equation: 
$$y_t=f(y_{t-1}, \Theta)+ \varepsilon_t,$$
where $y_{t-1}$ is the previous lagged input, $\Theta$ denotes the weight vector, $\varepsilon_t$ is a sequence of independently and identically distributed (i.i.d.) random errors, and $f$ denotes an autoregressive neural network function. The output of an ARNN$(1,k)$ model with activation function $G$ (e.g., logistic sigmoidal activation function) is given by 
\begin{align}\label{eq1}
    f(y_{t-1}, \Theta) &= \psi_1 y_{t-1}+ \nu + \sum_{i=1}^{k}\beta_i G\left(\phi_{i,1}y_{t-1}+\mu_i\right) \nonumber \\ 
    &=\psi_1 y_{t-1}+g\left(y_{t-1},\beta,\phi\right),
\end{align}
where the weight components are the shortcut connections $\psi_1$, the hidden layer to output unit weights $\beta =(\nu, \beta_1, \beta_2, \ldots, \beta_k)'$ and the input to hidden unit weights $\phi=\left(\phi_{1,1},\ldots,\phi_{k,1},\mu_1,\ldots,\mu_k\right)'$ are collected together in the network weight vector $\Theta$. 
\begin{remark}
Our proposed EWNet model can be thought of as a sum of $J+1$ different ARNN$(p,k)$ processes, where $J+1$ denotes the number of details and smooth coefficients obtained using the MODWT algorithm.
\end{remark}
Now, we show the ergodicity and stationarity of a simple ARNN$(1,k)$ process. In statistical analysis of nonlinear time series, the ergodicity and stationarity of the underlying process are of particular interest since, for such processes, a single realization displays the whole probability law of the data generation process \cite{meyn2012markov}. Before discussing the results for ergodicity and stationarity, we discuss the concept of irreducibility for the ARNN$(1,k)$ process, which acts as a connectionist in establishing the theoretical results.

\subsubsection{Irreducibility}
\textit{`Irreducibility'} is a very primordial concept of a Markov chain in which, irrespective of the starting point, the Markov chain can reach all parts of the state space \cite{meyn2012markov}. Another key property of Markov chains is called \textit{`aperiodicity'} which refers to a Markov chain with no cycles. More formally, the definition of \textit{`irreducibility'} can be given as follows \cite{panja2022interpretable}. 
\begin{definition}
A Markov chain is called irreducible if $\displaystyle\sum_{t=1}^{\infty} P^t(y,\mathcal{A}) > 0 $ for all $y \in \mathcal{X}$, whenever $\lambda(\mathcal{A})>0$, where $P^t(y,\mathcal{A})$ denotes the transition probability from the state $y$ to the set $\mathcal{A}\in \mathcal{B}$ in $t$ steps where the state space $\mathcal{X}\subseteq \mathcal{R}^2, \; \text{and} \; \mathcal{B}$ is the usual Borel $\sigma$-field and $\lambda$ be the Lebesgue measure. 
\end{definition}
\noindent Now, we write the ARNN$(1,k)$ process in the state space form as follows: 
\begin{equation}\label{eq2}
y_t=\psi_1 y_{t-1}+F(y_{t-1})+\varepsilon_t,
\end{equation}
where $F(y_{t-1}) = g\left(y_{t-1}, \beta,\phi\right)$ refers to the nonlinear component of $y_t$. Thus, $y_t$ is considered as a Markov chain with state space $\mathcal{X} \subseteq \mathcal{R}^2$ equipped with Borel $\sigma$-field $\mathcal{B}$ and Lebesgue measure $\lambda$. To establish the results for irreducibility, we begin by writing Eqn. (\ref{eq2}) as a control system driven by the control sequence $\{\varepsilon_t\}:$ 
$$y_t = F_t (y_0, \varepsilon_1,\ldots,\varepsilon_t),$$
where the definition of $F_t(\cdot)$ follows inductively from Eqn. (\ref{eq2}). We define $A_+^t(y)$ as the set of all states that are accessible from $y$ at time $t$: 
$$A_+^0 :=\{y\} \; \text{and} \; A^t_+(y):=\{F_t(y_0,\varepsilon_1,\ldots,\varepsilon_t);\varepsilon_i\in \theta\},$$ where the control set $\theta$ is an open set in $\mathcal{R}$. The control system $F_t$ is said to be forward accessible if the set $\displaystyle\bigcup_{t=0}^\infty A_+^t(y)$ has a nonempty interior for each $y\in \mathcal{X}$. Generally, forward accessibility refers to the set of reachable states that is not concentrated in some lower dimensional subset of $\mathcal{X}$. This property together with an additional assumption on the noise process ensures the irreducibility of the corresponding Markov process \cite{meyn2012markov}. Now, we write the control system defined in Eqn. (\ref{eq2}) as follows: 
\begin{align}
y_t &= \psi_1 y_{t-1} +F(y_{t-1}) + \varepsilon_t \nonumber \\
    &= \psi_{1}^2 y_{t-2} +\psi_1 F (y_{t-2}) +F(y_{t-1})+\varepsilon_t.
\end{align}

Consider a special case: when $F \equiv 0$, the control system $F_t$ is referred to as a controllable linear system, where every point of the state space is accessible regardless of its initialization for any control value $\varepsilon_t$. The underlying assumptions of a forward control system (as in Eqn. (\ref{eq2})) are presented below in Prop. \ref{prop1}.
\begin{prop}\label{prop1}
The sufficient conditions of forward accessibility for the control system (in Eqn. (\ref{eq2})) are the followings:
\begin{enumerate}
    \item {$G \in C^{\infty}$ is a bounded, non-constant, and asymptotically constant function $(C^{\infty})$} (any function is $C^{\infty}$ if derivatives of all orders are continuous).
    \item {The linear part of R.H.S. of Eqn. (\ref{eq2}) is controllable, i.e., $\psi_1 \neq 0$.} 
\end{enumerate}
\end{prop}
\begin{proof}
The proof builds on \cite{panja2022interpretable, trapletti2000stationary}. 
Logistic squasher activation functions (used in the EWNet model) satisfy Assumption 1. Assumption 2 of Prop. \ref{prop1} implies the non-vanishing criterion (controllability) of the linear part of R.H.S. of Eqn. (\ref{eq2}). Since Assumption 1 holds for the ARNN model, then for any $k \in  \mathbb{Z}^+$ and any scalars $\beta_0,\beta_i,\mu_i$ and $\phi_i \neq 0$, the condition  $$\beta_0 +\displaystyle\sum_{i=1}^k\beta_i G'(\phi_i y +\mu_i)=0, \; \forall \; y \in \mathcal{R}$$ implies $\beta_0 =0$ (from Assumption 1). Next, we define a major element of the generalized controllability matrix (GCM) as follows: 

$$c =\psi_1+\displaystyle\sum_{i=1}^k \beta_i \phi_{i,1}G'\left(\phi_{i,1}(\hat{y}_1)+\mu_i\right).$$ 
We can set $\theta \equiv \mathcal{R}$ and choose any $\hat{y}_1$. Then Assumption 2 implies that $c\neq 0$. This indicates that the GCM matrix is a non-singular matrix and, therefore, the control system in Eqn. (\ref{eq2}) is forward accessible, concluding the proof of Prop. \ref{prop1}. Related lemmas for multilayered perceptron are given as Lemma 2.5-2.7 in \cite{hwang1997prediction}. 
\end{proof}

\begin{remark}
The controllability of the linear components of the ARNN process is shown in Prop. \ref{prop1} implies forward accessibility. But, the associated Markov chain is said to be irreducible when the support of the distribution of the noise process is sufficiently large.
\end{remark}
Therefore, under suitable conditions on the distribution of the noise process $\varepsilon_t$, we can show the irreducibility of the corresponding Markov chain.

\begin{theorem} \label{th1}(Theorem of Irreducibility)
Suppose the distribution of $\varepsilon_t$ is absolutely continuous w.r.t. the Lebesgue measure $\lambda$ and the probability distribution function (p.d.f.) $\nu(\cdot)$ of $\varepsilon_t$ is positive everywhere in $\mathcal{R}$ and lower semi-continuous. Then under the condition prescribed in Prop. \ref{prop1}, the Markov chain in Eqn. (\ref{eq2}) is irreducible on the state space $(\mathcal{R}^2, \mathcal{B})$. 
\end{theorem}
\begin{proof}
The proof build on \cite{trapletti2000stationary, chakraborty2020unemployment, panja2022interpretable}. 
It is trivial that the state $y^*=0$ is globally attractive from the control system defined in Eqn. (\ref{eq2}), and the next component of $y_t$, regardless of its origin, can reach the point $0$ in one step. Furthermore, we consider the iterated first component from $t=0$ to $t=2$ and define it as $y_2= \ldots +g(\ldots,\beta,\phi)$, where all the terms that are functions of the starting point or the second component are necessarily omitted. Owing to the bounded and continuous function $g(\cdot)$ and non-zero value of $\psi_1$, it is obvious that the initial component can reach the point $0$, irrespective of its starting point and the second component, in two steps. Following the above-stated argument, we can conclude that the state space in $\mathcal{R}^2$ is connected since every state can be approached in two steps. Hence, the Markov Chain, defined in Eqn. (\ref{eq2}) is `aperiodic' and `irreducible'. An immediate instance is a Gaussian white noise that satisfies the conditions stated in Theorem \ref{th1} without loss of generality.
\end{proof}

\begin{remark}
Theorem \ref{th1} shows the irreducibility property for the ARNN$(1,k)$ process and demonstrates its proximity to the concept of forward accessibility of a control system. However, we also showed that ARNN processes might not exhibit forward accessibility, and in such scenarios, inferring about the data-generating process from the observed data is impossible.
\end{remark}

\subsubsection{Ergodicity and Stationarity}
This section shows the (strict) stationarity of the state-space form defined in Eqn. (\ref{eq2}). For a state-space $\{y_t\}$, the notion of stationarity has a close relationship with the geometric ergodicity of the process. The geometric ergodicity of a stochastic process implies that the underlying distribution of the process converges to the unique stationary solution at a geometric rate for any initials of the model \cite{meyn2012markov}. A formal definition of geometric ergodicity and asymptotic stationarity can be given following \cite{trapletti2000stationary}. 
\begin{definition}
A Markov chain $\{y_t\}$ is called geometrically ergodic if there exists a probability measure $\Pi$ on $(\mathcal{X, B}, \lambda)$ and a constant $\rho >1$ such that $\displaystyle\lim_{t\to\infty} \rho^t ||P^t(y,\cdot)-\Pi(\cdot)||=0$ for each $y \in \mathcal{X}$, where $||\cdot||$ denotes the total variation norm. Then, we say the distribution of $\{y_t\}$ converges to $\Pi$ and $\{y_t\}$ is asymptotically stationary. 
\end{definition}
Hence, $\{y_t\}$ is (strictly) stationary when it starts in the infinite past or with initial distribution $\Pi$. We give the main result on ergodicity and stationarity of the associated Markov chain in the theorem below. 

\begin{theorem}\label{th2} (Main Theorem)
Suppose the Markov chain $\{y_t\}$ of the ARNN$(1,k)$ process satisfies the conditions of Theorem \ref{th1} and $E|\varepsilon_t| < \infty$. Then, a sufficient condition for the geometric ergodicity (vis-a-vis asymptotic stationarity) of the Markov chain $\{y_t\}$ is that $|\psi_1|< 1$. 
\end{theorem}
\begin{proof}
To show the geometric ergodicity, we use Theorem 15.0.1 of \cite{meyn2012markov} and verify the drift criterion 15.3 of Theorem 15.0.1 of \cite{meyn2012markov}. Similar results for the vector threshold autoregressive model are discussed in \cite{tjostheim1990non}. 

We begin the proof by recalling the state-space model in Eqn. (\ref{eq2}): $$y_t = \psi_1 y_{t-1}+F(y_{t-1})+\varepsilon_t,$$
where $F(\cdot)$ is the nonlinear part and the intercept. For the general ARNN$(p,k)$ process, we define the following matrix:
$$\Psi := 
\begin{bmatrix} 
\psi_1 & \psi_2 & \dots  & \psi_{p-1} & \psi_p \\
1      &      0 & \dots  &          0 &      0 \\
0      &      1 & \dots  &          0 &      0 \\
\vdots & \vdots & \ddots &     \vdots & \vdots \\
0      &      0 &  \dots &          1 &      0 \\
\end{bmatrix}
\quad
$$
as the shortcut connections to the autoregressive part. Now, there exists a transformation $\mathcal{Q}$ such that $\Gamma = \mathcal{Q}\Psi\mathcal{Q}^{-1}$ where the diagonal elements $\Gamma$ consists of the eigenvalues of $\Psi$ and the off-diagonal elements are arbitrarily small. Considering, $T(y) = ||\sum y||$ as the test function and $\tau =  \{y \in \mathcal{R}^p, T(y) \leq c'\}$, for some $c'<\infty$, as the test set, we have
\begin{align}
    \mathbb{E}[T(y_t)|y_{t-1}=y] &\leq ||\mathcal{Q}\Psi y|| + ||\mathcal{Q} F(y)|| + \mathbb{E}||\mathcal{Q}\varepsilon_t|| \nonumber \\
    & \leq (||\Lambda|| +||\Delta||) T(y) \nonumber + ||\mathcal{Q}F(y)|| +\mathbb{E}||\mathcal{Q}\varepsilon_t||,
\end{align}
where $\Lambda$ is a diagonal matrix with the eigenvalues of $\Psi$, i.e., $\Lambda = diag(\Gamma)$ and $\Delta = \Gamma - \Lambda$. Since, the absolute value of the largest eigenvalue of $\Psi$ is strictly less than one, following the assumption of Theorem \ref{th2}, then $||\Lambda||< 1$, and  the transformation $\mathcal{Q}$ can be so chosen that $(||\Lambda||+||\Delta||)< 1- \epsilon$ for some $\epsilon > 0$.
Since the second and third terms are bounded, we can choose $\epsilon$ such that $\mathbb{E}[T(y_t)|y_{t-1}=y]\leq (1-\epsilon)T(y)+\delta {1}_\tau(y)$ for some $0<\delta<\infty$ and for all $y$. The result is also valid for the test function $T(y)+1$ and hence, we get the desired result. 
\end{proof}

\begin{remark}
Theorem \ref{th2} states the sufficient condition for the geometric ergodicity of the ARNN$(1,k)$ process. Consider the following example: if $\psi_1=1$, then the long-term behavior of the ARNN$(1,k)$ process can be determined by the nonlinear part and the intercept term of the process. Moreover, the geometric convergence rate in Theorem \ref{th2} implies that the memory of the ARNN process vanishes exponentially fast. This means that the simplest version of the ARNN$(p,k)$ process converges to a Wiener process \cite{li2018nonlinear}. Also, theoretical results suggest that the shortcut weight corresponding to the autoregressive part determines whether the overall process is ergodic and asymptotically stationary.
\end{remark}

\subsubsection{Practical Implications of Theoretical Results}
Some interpretations and practical implications of the theoretical results are discussed below from practitioners' points of view:
\begin{enumerate}[label=(\alph*)]
\item In the ideal situation, when an irreducible ARNN process generates the data, the estimated weights are not too far from the true weights. Then, one can draw an indirect conclusion on the statistical nature of the estimated shortcut weight corresponding to the autoregressive part being less than one in absolute terms, and then the data generation process is said to be ergodic and stationary. But, if the conditions are not met, the model is likely to be unspecified, and the estimation procedure should be diligently done. 
\item The theoretical results of asymptotic stationarity and ergodicity for the EWNet$(p,k)$ model would directly follow from the ARNN$(p,k)$ process since the proposed EWNet is a simple aggregation of several ARNN models fitted after the Wavelet decomposition of the time series data. These theoretical results guarantee that the proposed EWNet model cannot show `explosive' behavior or growing variance over time. 
\item The theoretical result for the number of hidden nodes in the EWNet model is set to a fixed value depending on the number of lagged inputs (as discussed in Section \ref{Stable_Learning}). Due to this, the running time of the EWNet model is minimal as compared to unstable neural networks in which the number of hidden nodes either becomes too large or too small. Thus, our proposed model does not face the problem of under-fitting or over-fitting.  
\end{enumerate}

\section{\label{sec:result}Experimental Analysis}
In this section, we present a detailed description of the: Epidemic datasets and their global characteristics (refer to Section \ref{Epi_data}); Performance measures used in our study (refer to Section \ref{metrics}); Benchmark forecasters utilized in our study (\ref{Models}), and Implementation of the proposed EWNet model for epidemiological datasets along with its performance comparison with the state-of-the-art forecasters (refer to Section \ref{Comparison}).

\subsection{\label{Epi_data}Epidemic Datasets and their Global Characteristics}
Epidemic datasets are accumulated from publicly available data resources (health websites, published manuscripts, etc.). They represent crude data of diseases, namely dengue, malaria, hepatitis B, and influenza, occurring in distinct regions. In this study, we have considered 15 datasets, amongst which 11 of them represent the overall number of subjects infected by a particular disease in a week, whereas the remaining corresponds to the aggregated monthly caseload. For example, the dengue incidence cases in Ahmedabad are recorded weekly per $10^4$ population, whereas, for the Philippines, we consider the total number of people suffering from dengue across several regions per $10^6$ population. These epidemic time series datasets are of different lengths and free from missing observations. Moreover, we analyze several global attributes of these datasets to understand real-world epidemiological datasets' structural patterns and identify the best-suited epicasting framework for the given scenario. Since the primary objective of this study is to provide a meaningful epicasting technique for real-world epidemic datasets, comprehensive knowledge of the data is the foundation step to accomplish this goal. Thus, we study several classical and advanced time series characteristics of the epidemic datasets based on the recommendations of \cite{de200625, lemke2015metalearning}. A detailed description and usage of these global characteristics are summarized below: \\
\emph{Stationarity} is a time series's foremost fundamental statistical property essential for many classical forecasting models. A time series is said to be generated from a stationary process if the series does not change over time. Our study used the Kwiatkowski–Phillips–Schmidt–Shin (KPSS) test to test the null hypothesis that the given time series is stationary \cite{shin1992kpss}. This test is implemented using the \emph{kpss.test} function of “tseries” package in R. \\
\emph{Nonlinearity} is another essential time series feature that determines the model variant to be used. For testing the null hypothesis that the observed time series is linear, we perform a Teraesvirta's neural network test, using the \emph{nonlinearityTest} function of the R package ``nonlinearTseries" \cite{terasvirta1993power}.\\
\emph{Seasonality} is another essential feature of a time series that refers to the repeating patterns of the series within a fixed period. We analyze the given series by performing a combined test comprising of the Kruskall-Wallis test and QS test of seasonality, often termed Ollech and Webel's test, to determine the presence of seasonal patterns. This test was performed using \emph{isSeasonal} function of ``seastests" in R.\\
\emph{Long range dependence} in time series processes has attracted much attention in probabilistic time series. To compute the time series's long-range-dependency or self-similarity parameter, Hurst exponent($H$), is used \cite{hurst1965experimental}. The value of $H$ is computed using the \emph{hurstexp} function of the R package ``pracma”. \\

On performing the above-mentioned statistical tests and computing the global characteristics of epidemic datasets, we summarize the relevant results in Table \ref{Global_cht}. 

\begin{table*}
    \centering
    \caption{Global characteristics of epidemic datasets}
    \scriptsize
    \begin{tabular}{|c|c|c|c|c|}
    \hline
        Datasets & Time span & Frequency & Length & Behavior\\ \hline
        \href{https://github.com/JohannHM/Disease-Outbreaks-Data/blob/master/Australia_Flu.dat}{Australia Influenza} & 1947 - 2015 & Weekly & 974 & Long term dependent, Non-stationary, Non-seasonal, Nonlinear \\ 
        \href{https://github.com/JohannHM/Disease-Outbreaks-Data/blob/master/Japan_Flu.dat}{Japan Influenza} & 1998 - 2015 & Weekly & 964 & Long term dependent, Stationary, Non-seasonal, Nonlinear \\
        \href{https://github.com/JohannHM/Disease-Outbreaks-Data/blob/master/Mexico_Flu.dat}{Mexico Influenza} & 2000 - 2015 & Weekly & 830 & Long term dependent, Non-stationary, Non-seasonal, Nonlinear\\ \hline
        Ahmedabad Dengue \cite{enduri2017estimation} & 2005 - 2012 & Weekly & 424 & Long term dependent, Non-stationary, Non-seasonal, Nonlinear \\ 
        Bangkok Dengue \cite{polwiang2020time} & 2003 - 2017 & Monthly & 180 & Long term dependent, Non-stationary, Seasonal, Nonlinear \\ 
        \href{https://github.com/JohannHM/Disease-Outbreaks-Data/blob/master/Colombia_Dengue.dat}{Colombia Dengue}  & 2005 - 2016 & Weekly & 626 & Long term dependent, Non-stationary, Non-seasonal, Nonlinear  \\ 
        \href{https://data.gov.hk/en-data/dataset/hk-dh-chpsebcdde-dengue-fever-cases}{Hong Kong Dengue} & 2002 - 2017 & Monthly & 192 & Long term dependent, Non-stationary, Seasonal, Linear \\ 
        Iquitos Dengue \cite{deb2022ensemble} & 2002 - 2013 & Weekly & 598 & Long term dependent, Stationary, Non-seasonal, Nonlinear \\ 
        Philippines Dengue \cite{chakraborty2019forecasting} & 2008 - 2016 & Monthly & 108 & Long term dependent, Stationary, Non-seasonal, Nonlinear \\ 
        San Juan Dengue \cite{johansson2019open} & 1990 - 2013 & Weekly & 1196 & Long term dependent, Stationary, Non-seasonal, Nonlinear \\ 
        \href{https://github.com/JohannHM/Disease-Outbreaks-Data/blob/master/Singapore_Dengue.dat}{Singapore Dengue} & 2000 - 2015 & Weekly & 838 & Long term dependent, Non-stationary, Non-seasonal, Linear \\ 
        \href{https://github.com/JohannHM/Disease-Outbreaks-Data/blob/master/Venezuela_Dengue.dat}{Venezuela Dengue} & 2002 - 2014 & Weekly & 660 & Long term dependent, Non-stationary, Non-seasonal, Linear \\ \hline
        China Hepatitis B \cite{wang2018comparison} & 2010 - 2017 & Monthly & 92 & Long term dependent, Non-stationary, Seasonal, Nonlinear\\  \hline
        \href{https://github.com/JohannHM/Disease-Outbreaks-Data/blob/master/Colombia_Malaria.dat}{Colombia Malaria} & 2005 - 2016 & Weekly & 626 & Long term dependent, Non-stationary, Non-Seasonal, Linear \\
        \href{https://github.com/JohannHM/Disease-Outbreaks-Data/blob/master/Venezuela_Malaria.dat}{Venezuela Malaria} & 2002 - 2014 & Weekly & 669 & Long term dependent, Non-stationary, Non-Seasonal, Nonlinear \\ \hline
    \end{tabular}
    \label{Global_cht}
\end{table*}

\subsection{Performance Measures\label{metrics}}

In our analysis, we evaluate the forecasts obtained from the proposed model and other baseline models using four popularly used accuracy measures, namely Root Mean Squared Error (RMSE), Mean Absolute Scaled Error (MASE), Mean Absolute Error (MAE), and symmetric Mean Absolute Percent Error (sMAPE) \cite{hyndman2018forecasting}. The mathematical formula for calculating these measures is given below:

{\scriptsize
\[ \text{ RMSE} = \sqrt{\frac{1}{N}\sum_{t =1}^{N} (y_t - \hat{y_t})^2}; \;
\text{ MASE} = \frac{\sum_{t = F + 1}^{F+N} |\hat{y_t} - y_t|}{\frac{N}{F-S} \sum_{t = S+1}^F |y_t - y_{t-S}|};  
\;\]
\[
\text{ MAE} = \frac{1}{N}\sum_{t =1}^{N} |y_t - \hat{y_t}|; \;\text{and} \;
\text{ sMAPE} = \frac{1}{N} \sum_{t=1}^N \frac{2|\hat{y_t} - y_t|}{|\hat{y_t}|+ |y_t|} \times 100 \%;
\]
}
where $N$ denotes the forecast horizon, $\hat{y_t}$ is the forecast against the actual value $y_t$. By definition, the minimum value of these performance measures suggests the `best' model. 

\subsection{\label{Models}Benchmark Forecasting Models}

Below we provide a brief description of the baseline models included in the experimental analysis and their implementation:\\ 
\noindent (a) Statistical Models:
\begin{itemize}
        \item \emph{Random Walk} (RW), also popularly known as the persistence model, is one of the simplest stochastic models based on the assumption that in each period the time-dependent variable takes a random step away from its previous value, and the steps are independently and identically distributed in size with zero-mean \cite{pearson1905problem}.
        \item \emph{Random Walk with Drift} (RWD) is a variant of the persistence model where the distribution of step sizes has a non-zero mean \cite{entorf1997random}. If the series being fitted by a random walk model has an average upward (or downward) trend that is expected to continue in the future, one includes a non-zero constant term in the model, i.e., assume that the random walk undergoes ``drift''.
        \item \emph{Autoregressive Integrated Moving Average} (ARIMA) is one of the most popular forecasting techniques that track linearity in a stationary time series \cite{box2015time}. The ARIMA model is a linear regression model indulged to track linear tendencies in stationary time series data. The model is expressed as ARIMA(p,d,q), where p, d, and q are integer parameter values that decide the structure of the model. More precisely, p and q are the order of the AR and MA models, respectively, and parameter d is the level of the difference applied to the data.
        \item \emph{Exponential Smoothing State Space} (ETS) models are very effective univariate forecasting techniques. This model comprises of three components - an error component (E), a trend component (T), and a seasonal component(S). Forecasts are computed in this model as a weighted average of historical data, with exponentially decreasing weights for distant observations \cite{hyndman2008forecasting}. 
        \item \emph{Theta Method} is a univariate time series framework that decomposes the series into two or more `theta lines' and extrapolates them using various forecasting techniques; the predictions for each series are aggregated to produce the outcome \cite{assimakopoulos2000theta}.
        \item \emph{Trigonometric Box-Cox ARIMA Trend seasonality} (TBATS) model handles time series data with multiple seasonal patterns using an exponential smoothing method \cite{de2011forecasting}.
        \item [] Various statistical models, namely RW, RWD, ARIMA, ETS, Theta, and TBATS models are implemented using the ``forecast" package of R statistical software. 
        \item \emph{Self-exciting Threshold Autoregressive} (SETAR) is an extended autoregressive model that allows for flexibility in the model parameters through a regime-switching behavior \cite{tong1990non}. We execute this model using the \emph{setar} function of the ``tsDyn" package in R with the default embedding dimension as 4.
        \item \emph{Wavelet-based ARIMA} (WARIMA) model is a variant of the ARIMA method. This model decomposes a non-stationary time series into several wavelet coefficients and generates forecasts from each of these series using an ARIMA model, and the final prediction is an aggregate of these candidate forecasts \cite{aminghafari2007forecasting}. The WARIMA models are trained on the datasets with ``WaveletArima"  package of R with the default parameters $MaxARParam = MaxMAParam = 5$.  
        \item \emph{Bayesian Structural Time Series} (BSTS) framework models structural time series in Bayesian framework for generating short-term forecasts and was implemented using ``bsts" package in R \cite{scott2014predicting}. 
\end{itemize}
\noindent (b) Machine Learning Approaches:
\begin{itemize}
        \item \emph{Artificial Neural Networks} (ANN), also known as neural nets, are popularly used in supervised learning problems. It is an extreme simplification of human neural systems and comprises of computational units analogous to biological neurons. ANNs consist of three layers: input, hidden (one or more), and output. Each neuron in the $m^{th}$ layer is interconnected with the neurons of the $(m+1)^{th}$ layer by some signal. Each connection is assigned a weight. The output may be calculated after multiplying each input with its corresponding weight. The output passes through an activation function to get the final ANN output. This multi-layered feedforward neural network can also model time-dependent signals using fully connected hidden layers \cite{rumelhart1986learning}. In standard ANN, a cross-validation approach is applied to choose the number of hidden layers and the number of hidden nodes. Furthermore, the weights are optimized using a gradient descent back-propagation algorithm. The ANN framework is implemented using the \emph{mlp} function of ``nnfor" package in R.

        \item \emph{Autoregressive Neural Network} (ARNN) is a modification of the ANN algorithm specialized for time series forecasting applications. Many potential problems in fitting ANN models were revealed such as the possibility that the fitting routine may not converge or may converge to a local minimum. Moreover, it was found that an ANN model which fits well with the training data may give poor out-of-sample forecasts \cite{faraway1998time}. To overcome these challenges, a single hidden-layered feedforward architecture, namely ARNN is proposed to generate forecasts in time series datasets \cite{faraway1998time}. It uses an autoregressive (AR) model to choose the optimal number of nodes in the input layer. This tends to reduce the effect of extreme input values, thus making the network somewhat robust to outliers as compared to a standard ANN model. The inputs to each node are combined using a weighted linear combination and modified by a nonlinear (sigmoidal activation) function before computing output. The model weights are directly estimated from the data using backpropagation, and the number of neurons in the hidden layer is set to $k = (p+1)/2$ where $p$ denotes the number of inputs selected using AR model \cite{hyndman2018forecasting}. We use the \emph{nnetar} function of the ``forecast" package of R to implement the  ARNN framework. 
         
        \item \emph{Support Vector Regression} (SVR) is a supervised learner that fits an optimal hyperplane to predict the future values of a time series \cite{smola2004tutorial}. To apply the SVR model, we transform the time series data into a matrix in which each value relates to the time window (lags = 15) that precedes it. Followed by the transformation, the radial basis kernel-based SVR model is fitted to the dataset by setting the regularization parameter to 1.0 and the loss penalty parameter value to 0.2 to generate the multi-step ahead forecasts in a recursive manner. In this study, we utilize the ``sktime" library in python to perform the data transformation and implement the SVR framework on epidemic datasets. 
\end{itemize}

\noindent (c) Deep Learning Models:
\begin{itemize}
        \item \emph{Long Short-term Memory} (LSTM) is a variant of the recurrent neural network (RNN) approach that models the long-term dependencies in a sequence prediction problem using several feedback connections in the training phase \cite{hochreiter1997long}. For implementing the LSTM networks, we utilize the default number of input and output observations as 10 and 3, respectively; the number of feature maps for each hidden RNN layer is set as 25, and the model is trained over 100 epochs \cite{herzen2022darts}. It is a popular benchmark deep learner for time series forecasting tasks. 
        
        \item \emph{Neural Basis Expansion Analysis for Time Series} (NBeats) model is extensively designed for forecasting time series datasets. It comprises of a fully connected neural network architecture with several blocks. Each block contains two layers - the first is responsible for processing the time series data to reproduce the past and forecast the future, and the second layer remodels the residuals obtained from the first to update the forecasts \cite{oreshkin2019n}. For the experimentation, we set the default number of blocks as 4.
        \item \emph{Deep Autoregressive} (Deep AR) is a time series forecasting model that utilizes a recurrent neural architecture for generating point estimates and interval estimates about future time points \cite{salinas2020deepar}.
        \item \emph{Temporal Convolutional Networks} (TCN) model utilizes convolutions to learn the sequential pattern in a time series and generalizes this pattern in the future \cite{chen2020probabilistic}. We train the TCN model with a default kernel size of 2 and 4 filters.
        \item \emph{Transformers} is a state-of-the-art deep learning model that analyzes the sequential patterns in time series using a multi-headed attention mechanism. This model can learn complex dynamic systems of historical data \cite{wu2020deep}. We implemented the transformers model with the input dimensionality as 64 and specified the number of heads in the multi-headed attention mechanism as 8. These default parameters avoid over-fitting in univariate series.
        \item [] All the above-stated deep learning frameworks have been implemented using the python library ``Darts" \cite{herzen2022darts}  specially designed for modeling time series datasets.
        \item \emph{W-Transformer} is a wavelet-based deep learner which has been recently proposed as an extension of the EWNet framework \cite{sasal2022w} for large-frequency time series data. This model utilizes a MODWT decomposition to the time series data and builds multi-head attention-based local transformers on the decomposed datasets to vividly capture the time series's non-stationarity and long-range nonlinear dependencies.
        \item \emph{Wavelet NBeats} (W-NBeats) is a wavelet variant of the data-driven NBeats framework, proposed as an extended version of EWNet \cite{singhal2022fusion}. 
        This model decomposes the time-indexed signal using a DWT approach with a Daubechies 4 filter into high-frequency and low-frequency wavelet coefficients. Followed by the DWT mechanism, the transformed series are individually modeled using an NBeats framework to generate one-step ahead forecasts. Finally, the forecasts generated by the detailed and smooth coefficients are aggregated to recursively generate the desired multi-step ahead predictions. This method is more useful for handling time series with multiple seasonal patterns. We implement the W-Transformers framework using the procedure described in \cite{sasal2022w}. In a similar way, we implement the W-NBeats framework.
\end{itemize}

\noindent (d) Hybrid Models:\\
The idea of generating hybrid forecasts of a time series after splitting it into linear and nonlinear components was first suggested by Zhang \cite{zhang2003time}. It comprises of two stages -  firstly, the linear part of the series is predicted using a linear model (e.g., ARIMA), and the residuals generated from this linear model are assumed to contain nonlinear patterns and are re-modeled in the second stage using a nonlinear model  (e.g., ARNN)  \cite{chakraborty2019forecasting}. The forecasts from these two stages are finally aggregated to generate the desired output. This hybridized approach has shown significant improvement over its component forecasters in several applications \cite{zhang2003time,chakraborty2020unemployment,chakraborty2020nowcasting,chakraborty2019forecasting}. We have considered three hybridized methods in our study, namely: 1. {Hybrid ARIMA-WARIMA} (We call it Hybrid-1) \cite{chakraborty2020real}; 2. {Hybrid ARIMA-ANN} (We call it Hybrid-2) \cite{zhang2003time}; 3. {Hybrid ARIMA-ARNN} (We call it Hybrid-3) \cite{chakraborty2019forecasting}. Forecasts for these hybrid models are generated using the implementation available at \cite{chakraborty2020nowcasting}.

\subsection{\label{Comparison}Experimental Results and Benchmark Comparison}

In this section, we discuss the implementation of the proposed EWNet model for epicasting. Several benchmark models are also considered for comparing the performance of our proposed epicaster. To assess the effectiveness of EWNet and comparative models, we use the standard cross-validation technique for time series forecasting, say rolling window method \cite{de2021hybrid}. To demonstrate the generalizability of the EWNet model, we analyze its epicasting performance for three different forecast horizons - long, medium, and short-term spanning over $(52, 26, 13)$ weeks for weekly datasets and $(12,6,3)$ months for monthly datasets, respectively. Furthermore, we compare the accuracy measures of our proposed EWNet model with state-of-the-art statistical models, machine learning methods, advanced deep learning architectures, and hybridized approaches. We initially partitioned the datasets into three segments for the experimentation: train, validation, and test set. The validation set was chosen to represent the temporal behavior of both the train and test sets \cite{hyndman2018forecasting}. We considered the validation size twice that of the test, following \cite{godahewa2021monash}. The validation set was used for tuning the hyper-parameters of the proposed EWNet$(p,k)$ model based on MASE metric, a popularly used forecasting metric \cite{hyndman2018forecasting}. Implementation of the EWNet algorithm (see Section \ref{Proposed_Model} for details) is done using R statistical software. 

During the implementation of the EWNet model, a multiresolution-based MODWT approach was first employed using the \emph{modwt} function of the ``wavelets" package in R to decompose the training data into its corresponding wavelet and scaling coefficients using the pyramid algorithm with `haar' filter and the number of levels set to the floor function of $\log_e(length(train))$ (see details in Algorithm \ref{algo}). In the next step, each series of wavelet and scaling coefficients (also named as details and smooth, respectively) are modeled with an autoregressive neural network having $p$ lagged inputs and $k$ hidden nodes arranged in a single hidden layer. For selecting the value of $p$, we follow a grid search approach over the range $(1 - 20)$ for epidemic datasets considered in this study. The choice of another model parameter $(k)$ defining the number of hidden nodes in the hidden layer of EWNet was made using the previously defined formula $k = \left[\frac{(p+1)}{2}\right]$ (as described in Section \ref{Stable_Learning}). Implementation of neural network generates a one-step-ahead forecast of the series using the \emph{nnetar} function of R statistical software \cite{hyndman2018forecasting}. Once the forecast for the validation of the desired horizon is generated for a grid of $p$ values, the parameter ($p$) was chosen by minimizing the MASE score on the validation dataset. Once the $p$ is finalized, we re-train the model using the chosen value of $p$ to generate one-step ahead out-of-sample forecasts. Furthermore, autoregressive feedforward neural network is also utilized iteratively to generate the forecast of the desired horizon. Finally, the output generated from all the networks is aggregated for forecasting the epidemic datasets. 

Below we discuss the values of the EWNet model parameters $(p,k)$ used for different epidemic datasets. In the case of the Singapore dengue incidence dataset, the chosen parameters were $(1,1)$ for all three forecast horizons, however, for the Venezuela dengue dataset, the values of $(p,k)$ are selected as $(1,1)$, $(7,4)$, and $(11,6)$ for 13, 26, and 52-weeks ahead forecasts, respectively. For forecasting short, medium, and long-term dengue incidence in Colombia, we use $(11,6)$, $(30,15)$, and $(7,4)$ as the values of the hyperparameters whereas for malaria incidence, the corresponding values are $(19,10)$, $(13,7)$, and $(20,10)$. For generating 3, 6, and 12-months ahead forecasts of hepatitis B incidence in China, the selected values of $(p,k)$ are $(2,1)$, $(1,1)$, and $(15,8)$. In the case of the Bangkok dataset, the trained EWNet model utilizes $(5,3)$, $(6,3)$, and $(1,1)$ as the model tuning parameters for forecasting dengue cases with 3, 6, and 12-month lead time, respectively. The values of the hyperparameters $(p,k)$ of the EWNet model for generating short, medium, and long-term forecasts of the Philippines are $(19,10)$, $(15,8)$, and $(14,7)$ and for Hong Kong datasets were $(7,4)$, $(6,3)$, and $(10,5)$, respectively. The malaria case loads of Venezuela are forecasted for short-term using EWNet $(20,10)$ model, and for the medium and long-term forecast the fitted EWNet model architecture has $(12,6)$ and $(1,1)$ as the chosen set of parameters values. In the case of Iquitos dengue incidence, the tuning hyper-parameter values are $(19,10)$, $(1,1)$, and $(5,3)$ for 13, 26, and 52-weeks ahead forecasts. For generating a long-term forecast of dengue incidence in San Juan the model, hyper-parameters are selected as $(9,5)$ whereas, in the case of short and medium-term forecasting, the chosen values are $(20,10)$ and $(1,1)$. The forecasts for Ahmedabad dengue cases are generated with the chosen architecture of the EWNet model as $(9,5)$, $(19,10)$, and $(15,8)$ for 13, 26, and 52 weeks, respectively. For generating short, medium, and long-term forecasts of influenza incidence cases the proposed model is trained with $(1,1)$, $(10,5)$, and $(1,1)$ for Japan, $(2,1)$, $(5,3)$, and $(1,1)$ for Mexico, and $(13,7)$, $(14,7)$, and $(4,2)$ for Australia, respectively. 


Once we implemented our proposed model on these epidemic datasets, we generated out-of-sample forecasts for different forecast horizons. Beneath, we summarize the epicasting performance of the proposed EWNet model with other state-of-the-art forecasters in terms of four performance measures. Three different forecast horizons are considered: short, medium, and long-term. Experimental results presented in tables \ref{S-DL}, \ref{M-DL}, and \ref{L-Trad} depict that the models' efficiency depends mainly on the type of disease considered and the forecast horizon. The accuracy measures for the Australian influenza cases show that our proposed EWNet architecture outperforms all the benchmark epicasters for different forecast horizons. Notably, the short-term forecast of the EWNet framework is more reliable than the second-best epicaster, ARNN. This improvement in the forecast accuracy is predominately attributed to the MODWT decomposition of the epidemic series. In the case of influenza incidence in Japan, the data-driven SVR model epicasts the 13-weeks ahead disease dynamics most accurately as measured by the RMSE metric, whereas the forecasts generated by the conventional Theta model lie closer to the actual incidence cases in terms of the absolute, scaled, and relative error metrics. However, the deep learning-based LSTM network and the hybrid ARIMA ANN methods are more precise for their medium-term forecasting analog. Moreover, the long-term influenza forecasts generated by the proposed EWNet framework for the Japan region are highly competitive with the deep neural architecture-based NBeats and ARNN frameworks. For Mexico's long-term influenza forecasting task, the proposed EWNet framework outperforms the baseline epicasters in terms of all the key performance indicators except the sMAPE score, where the LSTM network exhibits the least score. On the contrary, for the medium-term and short-term counterparts of the Mexico influenza epicasting, the persistence model and the hybridized ARIMA-WARIMA (Hybrid-1) frameworks, respectively, produce the best results. Moreover, based on the accuracy measures of the dengue forecasting, we can conclude that the proposed EWNet model generates a more reliable long-term and medium-term forecast for Ahmedabad and Hong Kong regions. In particular, for the Ahmedabad dengue incidence cases, the 52-weeks ahead forecast is improved by 37\% due to the use of the stable nonlinear neural network framework with the MODWT decomposition (as done in EWNet) instead of the linear ARIMA model with the MODWT decomposition (as done in WARIMA). However, for the short-term dengue forecasting of these regions, the proposed EWNet model and the Deep AR framework display competitive performance. The former model has the least RMSE and MASE scores and the latter performs best in terms of MAE and sMAPE metrics. Furthermore, for the dengue incidence cases of the Iquitos, Philippines, and Venezuela regions, the proposed EWNet approach demonstrates superior long-term forecasting ability compared to all the statistical, machine learning, and deep-learning forecasters. However, for the 26-weeks and 13-weeks ahead epicasting of the Venezuela dengue cases, the kernel-based SVR model, the persistence model, and the EWNet framework generate competitive out-of-sample predictions. Although the proposed EWNet framework and the SVM model exhibit the best short-term forecasting performance for the Philippines and Iquitos regions, the deep learning-based LSTM and NBeats methods significantly surpass other forecasters with the lowest medium-term forecasting error for these regions, respectively. The long-term forecasts generated by the proposed EWNet model for Singapore's dengue cases are competitive with the conventional RWD's epicasts. However, the stochastic ETS model and the machine learning-based ANN framework demonstrate better forecasting ability for this region's medium-term and short-term dengue incidence cases. Additionally, for the crude dengue incidence dataset of the San Juan region, the ARNN, EWNet, and LSTM models generate better out-of-sample predictions with 13-weeks, 26-weeks, and 52-week lead times, respectively. For the Bangkok region's long-term and medium-term dengue forecasts, we observe that the ETS model and RW model of statistical paradigm outperform all the forecasters, respectively. However, the performance of these forecasters lags behind the proposed EWNet model in generating a 13-weeks ahead forecast. Furthermore, the hyperplane-based SVR model generates the best medium-term forecasts for Colombia's dengue and malaria incidence cases. However, for the short-term forecast, although the SVR model can maintain its performance superiority in dengue incidence cases, the proposed EWNet framework significantly improves the forecast accuracy for malaria cases. On the other hand, for the 52-weeks ahead forecast of the Colombia region, the proposed EWNet model generates the best dengue forecast, and the traditional RWD model provides the same for the malaria counterpart. For the Venezuela region, statistical BSTS and WARIMA models, data-driven ARNN and Deep AR methods, and the proposed EWNet framework generate competitive forecasts for malaria incidence. Furthermore, in the case of hepatitis B cases in China, the proposed EWNet model and the SVR model generate competitive long-term forecasts. However, for medium-term and short-term forecasting, the MODWT-based WARIMA and EWNet framework transcends all other epicasters, respectively.

From the above experimental evaluations, it is identifiable that the epicasting performance of the advanced models, like SVR, ANN, ARNN, LSTM, Transformer, Deep AR, NBeats, and TCN, drastically drops for long-term forecasting compared with the proposed EWNet model for the majority of the infectious disease datasets. This phenomenon occurs primarily due to the lack of a humongous amount of historical data in most datasets, which can also be seen in several recent studies \cite{godahewa2021monash, chakraborty2020nowcasting, petropoulos2022forecasting}. Moreover, we can observe that the proposed EWNet framework outperforms the benchmark forecasters in the epicasting tasks, on average. This is primarily due to the non-stationary and nonlinear characteristics of the real-world epidemic datasets, as evident in Table \ref{Global_cht}. The wavelets coupled with ARNN in an ensemble framework (as done in the EWNet architecture) capture the non-stationary and seasonality of the time series using the wavelet decomposition, whereas the ARNN is responsible for handling nonlinear behavior. Additionally, since the epidemic datasets exhibit long-range dependency (as in Table \ref{Global_cht}), the ARNN framework present in the forecasting stage of the EWNet model can generate more reliable long-term forecasts \cite{leoni2009long}. It is also important to note that, despite the rapid surge of different attention-based models in epidemic forecasting \cite{wu2020deep, sasal2022w}, the performance of the multi-head Transformers model is significantly worse than the majority of the forecasters. This is because although Transformers can accurately extract semantic relations among the elements in a long sequence, in a time series modeling for extracting temporal correlations in an ordered sequence, the model employs positional encoding and tokenizes the dataset into several sub-series. This nature of the permutation-invariant self-attention mechanism eventually leads to the loss of temporal information resulting in imprecise forecast \cite{zeng2022transformers}. Moreover, unlike the proposed EWNet framework, the wavelet-based deep learners W-Transformers and W-NBeats lack the desired theoretical basis that restricts the model from showing ‘explosive’ behavior or growing variance over time, hence they fail to generate reliable forecasting results as compared to the proposed framework. Another potential cause for their failure is the small-data problems of epidemic datasets. Most deep learning methods are highly suitable for high-frequency (e.g., daily or hourly) datasets. However, high-frequency epidemiological datasets with many observations are seldom available, hence the applicability of these models is limited, especially in the epicasting domain.

Along with point estimates of the future epidemic cases, we also showcase the probabilistic band of the forecasts (for the test data). It is crucial in many applications, as they enable optimal decision-making under various forms of uncertainty in contrast to point forecasts. There are two widely used approaches for quantifying the uncertainty in machine learning-based forecasts: confidence intervals (CI) and conformal predictions (CP). The former is useful for quantifying the certainty of an estimate, whereas the latter is used to create prediction intervals – the confidence around a given prediction to capture the uncertainty of the model prediction \cite{vovk2005conformal}. We employ both approaches within our framework and obtain probabilistic bands over the point estimates for the test period of the epidemic datasets. For deriving the confidence intervals, we follow a simple pre-control limits approach \cite{montgomery2020introduction} and obtain more than 85\% confidence intervals. In formulating the EWNet framework (as in Eqn. (\ref{eq2})), we assume $\epsilon_t$ as a sequence of i.i.d. random shocks. Therefore, under the assumptions of normality, we use the formula for obtaining the probabilistic bands as upper pre-control limits (UPCL) $ = \text{ mean } + 1.5 \times \text{ sigma}$ and lower pre-control limits (LPCL) $ = \text{ mean } - 1.5 \times \text{ sigma}$. Under this assumption, we expect 86\% of the test data to lie within the probabilistic bands. However, the results may violate when the Gaussian assumption is not met (as seen in a few data examples in Figs. \ref{WARNN_Fitting_1}  and \ref{WARNN_Fitting_2}). There are other ways to obtain the confidence intervals explored in \cite{panja2022interpretable} (using simulations via Monte Carlo or bootstrapping) and in \cite{salinas2020deepar} (using expectations of loss function under the forecast distribution). The primary drawback of these computationally expensive algorithms is that their prediction intervals increase exponentially for long-range forecasting. However, these approaches are discarded since epidemic forecasts have real-time usage and cannot be computationally expensive. In Figs. \ref{WARNN_Fitting_1}  and \ref{WARNN_Fitting_2}, we present the probabilistic band obtained using mean $\pm 1.5$ sigma for short, medium, and long-term forecasts on all the epidemic datasets. Furthermore, we find the conformal predictions to associate reliable estimates of uncertainty quantification. Conformal prediction converts point estimates to a prediction region in a distribution-free and model-agnostic way that guarantees convergence \cite{vovk2017nonparametric}. We use the ``caretForecast" package in R to obtain conformal prediction intervals which are built by studying the distribution of the residuals. Since data and modeling uncertainties are considered for the validation data, conformal prediction generates trustable prediction intervals, as depicted in Figs. \ref{WARNN_Fitting_Conformal_1}  and \ref{WARNN_Fitting_Conformal_2}.
\begin{remark}
    Below we provide an in-depth analysis of the probabilistic bands in Figs. \ref{WARNN_Fitting_1} - \ref{WARNN_Fitting_Conformal_2} using pre-control limits and conformal prediction approaches:
    \begin{itemize}
        \item The medium and long-term prediction intervals of the EWNet framework (as in Figs. \ref{WARNN_Fitting_1} and \ref{WARNN_Fitting_2}) for Ahmedabad and Iquitos dengue datasets demonstrate that our proposal underestimates the crude incidence cases for these regions for a few weeks. One plausible reason for this could be the changing climatic patterns, including natural calamities, weather changes, and global warming, which eventually lead to a rise in precipitation, resulting in a sudden dengue epidemic outbreak. 
        \item For constructing the probabilistic band of the EWNet framework using the confidence interval approach, we assumed that the random shocks $\epsilon_t$ (refer to Eqn. (\ref{eq2})) follow a Gaussian distribution. However, this assumption is not met for some epidemic datasets, e.g., dengue cases of San Juan, Singapore, and Venezuela regions, and thus our results (including CIs) are violated. Hence, to overcome this drawback, we have also generated the conformal predictions following the model-agnostic  approach (Figs. \ref{WARNN_Fitting_Conformal_1}, \ref{WARNN_Fitting_Conformal_2}), which generates trustable prediction intervals using the distribution of the residuals.
        \item Moreover, it is frequently observed that the exposure of a population to any epidemic outbreaks develops herd immunity resulting in a decrease in the crude incidence cases as seen in the Colombia dengue and Japan influenza datasets. Traditional compartmental models (e.g., SIR) in epidemiology literature consider the population susceptibility cycles in their model formulation using certain pre-specified constraints; however, our proposed EWNet framework is unable to generalize this phenomenon owing to its pure data-driven approach. Although, regarding real-time forecasts and decision-making, accurate and reliable forecasts generated by EWNet for most datasets significantly enrich the epicasting benchmarks.
        \item For the malaria forecasting task of Colombia and Venezuela regions, we notice that the corresponding incidence datasets demonstrate certain anomalies (outliers and high peaks). These sudden changes in the level of infection are due to several factors, including but not limited to the impact of policy changes, environmental hazards, population behavior, and human settlements. These anomalous observations in the time series significantly deteriorate the forecasters' performance, including our proposed EWNet framework. 
        \item Thus, we recommend that practitioners and health officials consider the factors listed above while utilizing our EWNet framework for planning and decision-making in public health. Moreover, EWNet can easily adapt and improve during its usage when new test samples are available. This makes the proposed forecasting framework useful and reliable from the practitioner’s perspective.
        \end{itemize}
        \end{remark}

\begin{landscape}
\begin{table*}
\tiny
    \centering
    \caption{Long-term forecasting performance of the proposed EWNet model in comparison to the statistical, machine learning, and deep learning forecasting techniques (best results are \underline{\textbf{highlighted}})}
    \begin{tabular}{p{0.06\textwidth}p{0.042\textwidth}p{0.028\textwidth}p{0.028\textwidth}p{0.04\textwidth}p{0.028\textwidth}p{0.03\textwidth}p{0.052\textwidth}p{0.038\textwidth}p{0.038\textwidth}p{0.028\textwidth}p{0.035\textwidth}p{0.028\textwidth}p{0.032\textwidth}p{0.028\textwidth}p{0.038\textwidth}p{0.038\textwidth}p{0.03\textwidth}p{0.035\textwidth}p{0.03\textwidth}p{0.025\textwidth}p{0.042\textwidth}p{0.04\textwidth}p{0.06\textwidth}p{0.05\textwidth}} \hline

    Data & Metrics & {RW} & {RWD} & ARIMA & ETS & Theta & WARIMA & SETAR & TBATS & BSTS & Hybrid 1 & ANN & ARNN & {SVR} & Hybrid 2 & Hybrid 3 & {LSTM} & NBeats & Deep AR & TCN & Transfo- rmers & {W NBeats} & {W-Trans former} & \textcolor{blue}{Proposed} \\
    & & \cite{pearson1905problem} & \cite{entorf1997random} & \cite{box2015time} & \cite{hyndman2008forecasting} & \cite{assimakopoulos2000theta} & \cite{aminghafari2007forecasting} & \cite{tong2009threshold} & \cite{de2011forecasting} & \cite{scott2014predicting} & \cite{chakraborty2020real} & \cite{rumelhart1986learning} & \cite{faraway1998time} & \cite{smola2004tutorial} & \cite{zhang2003time} & \cite{chakraborty2019forecasting} & \cite{hochreiter1997long} & \cite{oreshkin2019n} & \cite{salinas2020deepar} & \cite{chen2020probabilistic} & \cite{wu2020deep} & \cite{singhal2022fusion} & \cite{sasal2022w} & \textcolor{blue}{EWNet} \\ \hline
    Australia & \textit{RMSE} & {82.08} & {81.74} & 81.58 & 84.84 & 83.13 & 58.55 & 61.28 & 64.97 & 175.8 & 81.71 & 76.00 & 69.65 & {90.59} & 81.60 & 78.74 & {85.17} & 58.76 & 81.66 & 840.4 & 94.95 & {105.18} & {103.4} & \underline{\textbf{49.41}} \\ 
    Influenza & \textit{MASE} & {3.350} & {3.331} & 3.314 & 3.573 & 3.451 & 2.667 & 2.683 & 2.811 & 8.908 & 3.354 & 2.995 & 2.689 & {4.057} & 3.314 & 3.164 & {3.568} & 2.836 & 3.440 & 33.72 & 4.435 & {5.556} & {5.613} & \underline{\textbf{2.330}} \\
    & \textit{MAE} & {53.21} & {52.90} & 52.63 & 56.75 & 54.80 & 42.35 & 42.62 & 44.65 & 141.4 & 53.27 & 145.7 & 171.6 & {64.44} & 148.9 & 191.2 & {56.67} & 159.2 & 	81.90 &	172.8 & 174.1 &	{88.25} & {89.15} & \underline{\textbf{37.00}} \\
    & \textit{sMAPE} & {90.82} &	{89.71} & 88.40 & 108.3 & 98.95 & 66.92 & 61.56 & 69.64 & 187.2 &	93.00 &	88.40 &	108.3 & {164.9} & 93.03 & 187.2 & {115.4} &	69.64 &	61.56 &	98.95 &	66.92 &	{120.1} & {122.7} & \underline{\textbf{58.65}} \\ \hline

    Japan &	\textit{RMSE} & {254.4} & {263.5} & 164.7 & 186.1 & 187.3 & 196.6 & 297.3 & 174.8 & 205.7 & 167.5 & 191.9 & 202.7 & {189.3} & 161.6 & 163.6 & {171.0} & \underline{\textbf{148.7}} & 179.6 & 211.3 & 157.4 & {276.9} & {289.9} & 156.5 \\
    Influenza & \textit{MASE} & {5.521} & {5.709} & 3.354 & 3.950 & 3.978 & 4.008 & 6.488 & 3.665 & 4.402 & 3.428 & 4.145 & 3.174 & {2.366} & 3.224 & 3.309 & {2.638} & \underline{\textbf{2.270}} & 3.737 & 3.137 & 2.945 & {5.398} & {5.876} & 2.798 \\
    & \textit{MAE} & {239.9} & {248.1} & 145.7 & 171.6 & 172.8 & 174.1  & 281.9 & 159.2 & 191.2 & 148.9 & 180.1 & 137.9 & {102.8} & 140.1 & 143.8 & {114.6} &	\underline{\textbf{98.64}} &	162.4 &	136.3 &	127.9 &	{234.6} & {255.3} & 121.5 \\
    & \textit{sMAPE} & {138.7} & {139.1} &	130.8 &	134.9 &	135.1 &	136.7 &	142.3 &	132.8 &	136.3 &	131.5 & 133.2 & \underline{\textbf{107.0}} &	{112.7} & 130.0 & 130.5 & {130.1} & 121.9 & 	134.2 &	162.6 &	129.3 &	{159.7} & {164.1} & 128.6 \\ \hline

    Mexico & \textit{RMSE} & {135.6}	&	{138.0}	&	67.82	&	153.7	&	89.59	&	113.6	&	723.5	&	93.40	&	206.8	&	66.60	&	275.8	&	1355	&	{65.03}	&	68.21	&	197.5	&	{52.63}	&	216.8	&	197.0	&	1212	&	568.9	&	{91.93}	&	{97.61}	&	\underline{\textbf{38.37}} \\ 
    Influenza	&	\textit{MASE} 	&	{8.345}	&	{8.488}	&	4.279	&	9.445	&	5.595	&	6.988	&	46.31	&	5.814	&	12.52	&	4.206	&	16.01	&	72.58	&	{2.447}	&	4.305	&	8.990	&	{2.684}	&	11.67	&	12.45	&	67.94	&	37.36	&	{5.475}	&	{5.666}	&	\underline{\textbf{1.945}} \\
    &	\textit{MAE} & 	{126.5}	&	{128.6}	&	64.86	&	143.1	&	84.81	&	105.9	&	701.8	&	88.11	&	189.8	&	63.74	&	242.7	&	1100	&	{37.09} &	65.25	&	136.2	&	{40.68}	&	176.8	&	188.8	&	1029	&	566.3	&	{82.98}	&	{85.88}	&	\underline{\textbf{29.47}} \\
    & \textit{sMAPE} & {138.9}	&	{139.2}	&	122.3	&	141.5	&	129.4	&	135.6	&	175.7	&	130.1 &	146.2	&	121.6	&	158.9	&	178.0	&	{120.4}	&	122.5	&	129.6	&	{\underline{\textbf{104.8}}}	&	140.0	&	150.2	&	196.9	&	177.1	&	{129.9}	&	{130.3}	&	108.4 \\ \hline

    Ahmedabad  &	\textit{RMSE} 	&	{15.24}	&	{14.98}	&	14.51	&	14.21	&	14.01	&	11.39	&	12.14	&	14.21	&	30.95	&	14.16	&	13.14	&	15.02	&	{15.59} 	&	14.43	&	14.07	&	{15.19}	&	16.29	&	14.54	&	584.2	&	13.02	&	{22.06}	&	{22.76}	&	\underline{\textbf{8.427}}	\\
    Dengue	&	\textit{MASE} 	&	{1.902}	&	{1.887}	&	1.953	&	1.987	&	1.975	&	1.489	&	1.979	&	1.987	&	4.497	&	1.769	&	1.949	&	1.814	&	{1.859}	&	1.959	&	1.996	&	{1.754}	&	2.024	&	1.747	&	82.37	&	2.205	&	{3.812}	&	{3.927}	&	\underline{\textbf{1.084}}	\\
    &	\textit{MAE} 	&	{9.885}	&	{9.802}	&	10.14	&	10.32	&	10.26	&	7.740	&	10.28	&	10.32	&	23.36	&	9.194	&	10.13 &	9.430	&	{9.659}	&	10.18	&	10.37	&	{9.117}	&	10.51	&	9.077	&	428.0	&	11.45	&	{19.81}	&	{20.41}	&	\underline{\textbf{5.633}} \\
    &	\textit{sMAPE} 	&	{99.07}	&	{98.52}	&	102.8	&	103.7	&	103.0	&	96.99	&	99.34	&	103.7	&	179.9	&	93.99	&	101.1	&	89.18	&	{93.27}	&	103.0	&	103.6	&	{85.55}	&	104.9	&	88.97	&	191.1	&	104.3	&	{142.3}	&	{151.3}	&	\underline{\textbf{69.69}}	 \\ \hline 
    
    Bangkok &	\textit{RMSE} 	&	{596.8}	&	{584.1}	&	493.0	&	424.6	&	473.9	&	779.4	&	\underline{\textbf{415.4}}	&	436.6 &	637.8	&	529.4	&	462.4	&	518.0	&	{480.7}	&	478.8	&	479.8	&	{884.9}	&	889.7	&	892.6	&	866.3	&	893.1	&	{1203}	&	{700.6}	&	426.7 \\	
    Dengue	&	\textit{MASE} 	&	{2.091}	&	{2.055}	&	1.938	&	\underline{\textbf{1.590}}	&	1.872	&	3.262	&	1.648	&	1.734	&	2.433	&	2.123	&	1.820	&	2.004	&	{1.906}	&	1.875	&	1.861	&	{3.303}	&	3.333	&	3.331	&	3.197	&	3.344	&	{4.314}	&	{2.646}	&	1.703 \\ 	
    &	\textit{MAE} 	&	{476.8}	&	{468.6}	&	441.8 &	\underline{\textbf{362.4}}	&	426.7	&	743.6	&	375.6	&	395.2	&	554.6	&	484.0	&	415.1	&	457.0	&	{434.6} &	427.6	&	424.3	&	{753.0}	&	760.1	&	759.5	&	729.0	&	762.5	&	{983.6}	&	{603.3}	&	388.2 \\	
    &	\textit{sMAPE} 	&	{70.44}	&	{68.95}	&	65.43	&	\underline{\textbf{53.95}}	&	63.09	&	76.92	&	56.23	&	58.73	&	86.00	&	71.33	&	61.56	&	67.98	&	{64.17}	&	63.20	&	62.67	&	{187.5}	&	194.3	&	190.5	&	171.2	&	195.1	&	{84.15}	&	{90.82}	&	57.24 \\ \hline
    
    Colombia &	\textit{RMSE} 	&	{1160}	&	{1278}	&	998.3	&	997.2	&	1036	&	863.0 &	\underline{\textbf{707.2}}	&	997.4	&	1576	&	1001	&	799.3	&	811.3	&	{1334}	&	994.7	&	997.6	&	{2125}	&	1107	&	2016	&	3317	&	732.1	&	{1629}	&	{1951}	&	730.9 \\	
    Dengue	&	\textit{MASE} 	&	{5.565}	&	{6.150}	&	4.864	&	4.861	&	5.023	&	4.463	&	3.880	&	4.859	&	7.598	&	4.849	&	4.118	&	3.930	&	{6.699}	&	4.852	&	4.862	&	{11.88}	&	5.621	&	11.16	&	16.99	&	5.495	&	{9.002}	&	{9.534}	&	\underline{\textbf{3.611}}	\\
    &	\textit{MAE} 	&	{918.2}	&	{1015}	&	802.6	&	802.1	&	828.8	&	736.3	&	640.1	&	801.7 &	1253	&	800.1	&	679.5	&	648.4	&	{1105}	&	800.6	&	802.2	&	{1961}	&	927.4 &	1842	&	2803	&	615.4	&	{1485}	&	{1573}	&	\underline{\textbf{595.8}} \\	
    &	\textit{sMAPE} 	&	{44.52}	&	{47.31}	&	40.99	&	40.97	&	41.82	&	38.91	&	35.14	&	40.96	&	53.44	&	40.87	&	36.69	&	35.48	&	{66.63}	&	40.93	&	40.97	&	{191.3}	&	45.54	&	165.6	&	145.8	&	49.50	&	{70.93}	&	{114.5}	&	\underline{\textbf{33.23}} \\ \hline
    
    Hong &	\textit{RMSE} 	&	{3.786}	&	{3.814}	&	4.219	&	4.136	&	4.278	&	3.261	&	4.599	&	3.992	&	4.187	&	4.362	&	4.014	&	4.877	&	{4.075}	&	4.553	&	4.076	&	{3.729}	&	4.599	&	6.617	&	57.82	&	4.747	&	{5.734}	&	{4.812}	&	\underline{\textbf{3.050}} \\
    Kong	&	\textit{MASE} 	&	{0.786}	&	{0.803}	&	0.914	&	0.904	&	0.948	&	0.739	&	0.911	&	0.853	&	0.932	&	0.989	&	0.823	&	1.082 &	{0.789}	&	1.010	&	0.884	&	{0.784}	&	0.918	&	1.437	&	11.38	&	0.955	&	{1.229}	&	{1.057}	&	\underline{\textbf{0.635}} \\	
    Dengue	&	\textit{MAE} 	&	{3.000}	&	{3.065}	&	3.491	&	3.452	&	3.619	&	2.823	&	3.478	&	3.256	&	3.560	&	3.778	&	3.144	&	4.131	&	{3.012}	&	3.859	&	3.377	&	{2.993}	&	3.508	&	5.487	&	43.48	&	3.648	&	{4.694}	&	{4.034}	&	\underline{\textbf{2.423}}	\\
    &	\textit{sMAPE} 	&	{36.11}	&	{36.76}	&	40.51	&	40.29	&	41.74	&	33.93	&	42.99	&	38.52	&	41.27	&	43.12	&	38.56	&	43.61	&	{36.16}	&	43.64	&	39.63	&	{36.19}	&	39.60	&	63.59	&	158.9	&	45.41	&	{49.18}	&	{47.68}	&	\underline{\textbf{30.87}} \\ \hline	

    Iquitos &	\textit{RMSE} 	&	{14.19}	&	{14.35}	&	10.12	&	14.19	&	14.24	&	12.97	&	10.35	&	10.08	&	17.79	&	10.50	&	12.05	&	11.36	&	{12.42}	&	10.19	&	10.24	&	{12.71}	&	13.25	&	12.49	&	120.9	&	13.67	&	{20.52}	&	{20.21}	&	\underline{\textbf{8.181}} \\
    Dengue	&	\textit{MASE} 	&	{2.213}	&	{2.258}	&	1.799	&	2.213	&	2.227	&	1.983	&	1.662	&	1.813	&	3.111	&	1.733	&	1.863	&	1.785	&	{1.917}	&	1.779	&	1.774	&	{1.944}	&	2.077	&	1.926	&	18.02	&	2.116	&	{3.253}	&	{3.212}	&	\underline{\textbf{1.559}} \\
    &	\textit{MAE} 	&	{9.635}	&	{9.829}	&	7.832	&	9.635	&	9.696	&	8.633	&	7.234	&	7.892	&	13.54	&	7.543	&	8.109	&	7.772	&	{8.345}	&	7.744	&	7.724	&	{8.461}	&	9.042	&	8.384	&	78.45	&	9.213	&	{14.16}	&	{13.98}	&	\underline{\textbf{6.784}}	\\
    &	\textit{sMAPE} 	&	{198.0}	&	{200.0}	&	105.2	&	199.9	&	200.0	&	137.7	&	102.6	&	105.3	&	198.8	&	105.4	&	118.9	&	111.1	&	{124.7}	&	105.1	&	105.2	&	{129.7}	&	113.1	&	125.6	&	169.2	&	159.7	&	{173.1}	&	{160.8}	&	\underline{\textbf{96.44}} \\ \hline	

    Philippines	&	\textit{RMSE} 	&	{43.26}	&	{45.29}	&	39.86	&	40.21	&	43.88	&	34.25	&	95.08	&	38.64	&	45.34	&	30.89	&	88.07	&	61.26	&	{37.55}	&	42.71	&	33.42	&	{56.88}	&	18.63	&	43.01	&	171.2	&	56.05	&	{63.03}	&	{47.39}	&	\underline{\textbf{15.01}} \\
    Dengue	&	\textit{MASE} 	&	{1.088}	&	{1.151}	&	1.011	&	0.979	&	1.109	&	0.713	&	2.364	&	0.777	&	0.973	&	0.647	&	2.331	&	1.251	&	{0.822}	&	1.098	&	0.841	&	{1.226}	&	0.431	&	0.831	&	4.348	&	1.205	&	{1.598}	&	{0.983} 	&	\underline{\textbf{0.306}} \\
    &	\textit{MAE} 	&	{37.84}	&	{40.06}	&	35.18	&	34.07	&	38.58	&	24.82	&	82.26	&	27.06	&	33.87	&	22.51	&	81.11	&	43.53	&	{28.60}	&	38.23	&	29.27	&	{42.65}	&	14.99	&	28.92	&	151.2	&	41.94	&	{55.59}	&	{34.18}	&	\underline{\textbf{10.67}} \\ 	
    &	\textit{sMAPE} 	&	{66.62}	&	{68.96}	&	64.13	&	62.57	&	67.60	&	45.88	&	95.34	&	53.25	&	80.50	&	45.02	&	96.70	&	70.90	&	{55.01}	&	67.12	&	79.63	&	{110.1}	&	31.18	&	57.88	&	200.0	&	106.8	&	{82.88}	&	{73.56}	&	\underline{\textbf{25.91}} \\ \hline
    
    San  &	\textit{RMSE} 	&	{115.9}	&	{115.9}	&	100.1	&	105.7	&	108.2	&	91.12	&	103.5	&	108.2	&	108.8	&	93.08	&	236.8	&	100.4	&	{112.7}	&	142.9	&	91.66	&	{\underline{\textbf{74.38}}}	&	104.4	&	112.5	&	426.8	&	121.4	&	{153.2}	&	{149.4}	&	99.69 \\	
    Juan	&	\textit{MASE} 	&	{5.843}	&	{5.837}	&	4.758	&	5.082	&	5.257	&	4.230	&	4.949	&	5.234	&	5.317	&	4.258	&	14.34	&	4.776	&	{5.589}	&	7.765	&	4.170	&	{\underline{\textbf{3.359}}}	&	4.900	&	5.571	&	19.19	&	6.207	&	{7.320}	&	{7.217}	&	4.722 \\
    Dengue &	\textit{MAE} 	&	{93.59}	&	{93.50}	&	78.21	&	81.41	&	84.22	&	67.75	&	79.27	&	83.84	&	85.18	&	68.21	&	229.8	&	76.52	&	{89.53}	&	124.4	&	66.80	&	{\underline{\textbf{53.80}}}	&	78.51	&	89.24	&	307.5	&	99.43	&	{117.3}	&	{115.6}	&	75.64 \\	
    &	\textit{sMAPE} 	&	{152.3}	&	{151.9}	&	97.07	&	107.6	&	116.1	&	78.12	&	102.5	&	114.4	&	120.2	&	77.05	&	114.8	&	96.31	&	{134.7}	&	95.11	&	74.21	&	{\underline{\textbf{59.47}}}	&	94.36	&	133.4	&	155.2	&	177.9	&	{132.9}	&	{130.5}	&	94.48 \\ \hline
    
    Singapore &	\textit{RMSE} 	&	{133.7}	&	{129.4}	&	130.1	&	122.4	&	128.6	&	193.0	&	136.9	&	129.8	&	159.7	&	130.4	&	162.9	&	142.1	&	{195.9}	&	130.3	&	130.3	&	{224.9}	&	174.9	&	129.5	&	392.1	&	134.1	&	{271.2}	&	{248.1}	&	\underline{\textbf{121.9}} \\
    Dengue	&	\textit{MASE} 	&	{2.348}	&	{\underline{\textbf{2.241}}}	&	2.490	&	2.579	&	2.473	&	3.873	&	2.883	&	2.453	&	3.026	&	2.505	&	4.078	&	3.454	&	{3.878}	&	2.509	&	2.504	&	{4.993}	&	4.016	&	2.411	&	8.256	&	3.359	&	{5.782}	&	{4.741}	&	2.435 \\
    &	\textit{MAE} 	&	{86.54}	&	{\underline{\textbf{82.61}}}	&	91.78	&	95.06	&	91.16	&	142.7	&	106.2	&	90.43	&	111.5	&	92.34	&	150.3	&	127.3	&	{142.9}	&	92.49	&	92.31	&	{184.0} &	148.0	&	88.89	&	304.3	&	123.8	&	{213.1}	&	{174.8}	&	90.46	\\
    &	\textit{sMAPE} 	&	{34.53}	&	{\underline{\textbf{32.84}}}	&	36.31	&	37.69	&	36.11	&	59.64	&	41.18	&	35.82	&	45.54	&	36.51	&	53.82	&	47.75	&	{68.32}	&	36.56	&	36.50	&	{107.8}	&	53.00	&	35.25	&	132.8	&	46.50	&	{72.23}	&	{64.82}	&	35.12	\\ \hline
    
    Venezuela &	\textit{RMSE} 	&	{597.6}	&	{598.3}	&	622.2	&	614.8	&	618.5	&	582.1	&	645.1	&	622.9	&	600.1	&	618.8	&	705.6	&	682.0	&	{835.0}	&	621.5	&	621.9	&	{1470}	&	682.5	&	1355	&	2003	&	874.0	&	{1076}	&	{1235}	&	\underline{\textbf{563.6}} \\
    Dengue	&	\textit{MASE} 	&	{3.182}	&	{3.183}	&	3.270	&	3.386	&	3.266	&	3.061	&	3.329	&	3.272	&	3.263	&	3.258	&	3.569	&	3.471	&	{4.300}	&	3.267	&	3.268	&	{8.580}	&	3.582	&	7.802	&	10.66	&	4.524	&	{5.843}	&	{6.441}	&	\underline{\textbf{2.927}} \\
    &	\textit{MAE} 	&	{507.3}	&	{507.4}	&	521.2	&	539.7	&	520.6	&	488.0	&	530.6	&	521.6	&	520.2	&	519.2	&	569.0	&	553.4	&	{685.4}	&	520.8	&	521.0	&	{1368}	&	571.0	&	1243	&	1700	&	721.2	&	{931.4}	&	{1027}	&	\underline{\textbf{466.6}} \\	
    &	\textit{sMAPE} 	&	{38.68}	&	{38.69}	&	39.85	&	40.60	&	39.80	&	37.08	&	40.64	&	39.88	&	39.53	&	39.69	&	44.24	&	42.76	&	{56.10}	&	39.81	&	39.83	&	{188.4}	&	45.04	&	152.9	&	147.5	&	60.21	&	{75.38}	&	{103.4}	&	\underline{\textbf{35.37}} \\ \hline

    China  &	\textit{RMSE} 	&	{8766}	&	{104$E^2$}	&	9358	&	9249	&	103$E^2$	&	8592	&	100$E^2$	&	9379	&	103$E^2$	&	9701	&	9868	&	101$E^2$	&	{\underline{\textbf{7289}}} &	9348	&	109$E^2$	&	{975$E^2$}	&	113$E^2$ &	975$E^2$	&	599$E^2$ &	975$E^2$	&	{104$E^2$}	&	{975$E^2$} &	8017 \\ 	
    Hepatitis &	\textit{MASE} 	&	{1.133}	&	{1.374}	&	1.221	&	1.205	&	1.350	&	1.137	&	1.340	&	1.224	&	1.368	&	1.264	&	1.289	&	1.293	&	{0.995}	&	1.203	&	1.461	&	{15.57}	&	1.543	&	15.57	&	9.171	&	15.57	&	{1.304}	&	{15.57}	&	\underline{\textbf{0.889}} \\
    B &	\textit{MAE} 	&	{7083}	&	{8588}	&	7635	&	7532	&	8441	&	7109	&	8376	&	7653	&	8551	&	7904	&	8060	&	8088	&	{6218}	&	7524	&	9137	&	{973$E^2$}	&	9648	&	973$E^2$	&	573$E^2$	&	973$E^2$	&	{8154}	&	{973$E^2$}	&	\underline{\textbf{5556}}	\\
    &	\textit{sMAPE} 	&	{7.344}	&	{9.010}	&	7.947	&	7.835	&	8.841	&	7.370	&	8.797	&	7.968	&	8.963	&	8.243	&	8.416	&	8.467	&	{6.415}	&	7.826	&	9.620	&	{199.9}	&	10.20	&	199.9	&	85.32	&	199.9	&	{8.507}	&	{199.9}	&	\underline{\textbf{5.790}}\\ \hline
    
    Colombia &	\textit{RMSE} 	&	{628.9}	&	{\underline{\textbf{626.5}}}	&	810.2	&	798.5	&	804.2	&	804.3	&	692.9	&	802.9	&	779.0	&	812.9	&	838.2	&	919.9	&	{753.6}	&	813.3	&	814.1	&	{1766}	&	861.3	&	1696	&	2510	&	1134	&	{1466}	&	{1566}	&	714.3 \\	
    Malaria	&	\textit{MASE} 	&	{2.969}	&	{\underline{\textbf{2.939}}}	&	3.547	&	3.500	&	3.518	&	3.618	&	3.192	&	3.518	&	3.448	&	3.561	&	3.748	&	4.365	&	{3.325}	&	3.559	&	3.570	&	{9.163}	&	4.049	&	8.750	&	11.40	&	5.410	&	{7.081}	&	{7.188}	&	3.305  \\
    &	\textit{MAE} 	&	{538.7}	&	{\underline{\textbf{533.3}}}	&	643.6	&	635.2	&	638.3	&	656.6	&	579.2	&	638.5	&	625.6	&	646.1	&	680.2	&	792.2	&	{603.4}	&	645.9	&	648.0	&	{1662}	&	734.9	&	1587	&	2070	&	981.8	&	{1285}	&	{1304}	&	599.7 \\ 	
    &	\textit{sMAPE} 	&	{34.04}	&	{\underline{\textbf{33.71}}}	&	41.71	&	41.04	&	41.30	&	42.63	&	36.89	&	41.30	&	40.27	&	41.91	&	45.00	&	56.16	&	{38.59}	&	41.89	&	42.08	&	{193.9}	&	50.82	&	174.5	&	136.4	&	74.94	&	{90.65}	&	{115.4}	&	38.44	\\ \hline
    
    Venezuela &	\textit{RMSE} 	&	{418.4} &	{395.3}	&	250.1	&	271.0	&	261.7	&	\underline{\textbf{193.8}}	&	301.2	&	249.9	&	213.9	&	238.8	&	206.6	&	652.3	&	{745.2}	&	247.1	&	236.6	&	{1448}	&	421.0 &	1325	&	1893	&	827.1	&	{327.7} &	{1018}	&	232.3 \\	
    Malaria	&	\textit{MASE} 	&	{2.188}	&	{2.064}	&	1.279	&	1.387	& 1.336 &	\underline{\textbf{0.912}}	&	1.553	&	1.278	&	1.079	&	1.216	&	1.032	&	3.530	&	{4.242}	&	1.262	&	1.191	&	{8.591}	&	2.123	&	7.839	&	9.884	&	4.774	&	{1.712}	&	{5.799}	&	1.183 \\
    &	\textit{MAE} 	&	{364.3}  &	{343.6}	&	212.8	&	230.9	&	222.4	&	\underline{\textbf{151.8}}	&	258.4	&	212.7	&	179.5	&	202.5	&	171.9	&	587.7	&	{706.2}	&	210.2	&	198.3	&	{1430}	&	353.5	&	1305	&	1645	&	794.9	&	{285.1}	&	{965.5}	&	196.8 \\	
    &	\textit{sMAPE} 	&	{27.19}	&	{25.41}	&	14.95	&	16.29	&	15.65	&	\underline{\textbf{10.77}}	&	18.40	&	14.94	&	12.57	&	14.19	&	12.01	&	49.10	&	{62.37}	&	14.76	&	13.90	&	{192.5}	&	26.45	&	161.3	&	141.7	&	73.45	&	{20.03}	&	{98.22}	&	13.80 \\ \hline	
\label{L-Trad}     
\end{tabular}
\end{table*}
\end{landscape}

\begin{landscape}
\begin{table*}
\tiny
    \centering \caption{Medium-term forecasting performance of the proposed EWNet model in comparison to the statistical, machine learning, and deep learning forecasting techniques (best results are \underline{\textbf{highlighted}})}
    \begin{tabular}{p{0.06\textwidth}p{0.042\textwidth}p{0.028\textwidth}p{0.028\textwidth}p{0.04\textwidth}p{0.028\textwidth}p{0.03\textwidth}p{0.052\textwidth}p{0.038\textwidth}p{0.038\textwidth}p{0.028\textwidth}p{0.035\textwidth}p{0.028\textwidth}p{0.032\textwidth}p{0.028\textwidth}p{0.038\textwidth}p{0.038\textwidth}p{0.03\textwidth}p{0.035\textwidth}p{0.03\textwidth}p{0.025\textwidth}p{0.042\textwidth}p{0.04\textwidth}p{0.06\textwidth}p{0.05\textwidth}} \hline

    Data & Metrics & {RW} & {RWD} & ARIMA & ETS & Theta & WARIMA & SETAR & TBATS & BSTS & Hybrid 1 & ANN & ARNN & {SVR} & Hybrid 2 & Hybrid 3 & {LSTM} & NBeats & Deep AR & TCN & Transfo- rmers & {W NBeats} & {W-Trans former} & \textcolor{blue}{Proposed} \\
    & & \cite{pearson1905problem} & \cite{entorf1997random} & \cite{box2015time} & \cite{hyndman2008forecasting} & \cite{assimakopoulos2000theta} & \cite{aminghafari2007forecasting} & \cite{tong2009threshold} & \cite{de2011forecasting} & \cite{scott2014predicting} & \cite{chakraborty2020real} & \cite{rumelhart1986learning} & \cite{faraway1998time} & \cite{smola2004tutorial} & \cite{zhang2003time} & \cite{chakraborty2019forecasting} & \cite{hochreiter1997long} & \cite{oreshkin2019n} & \cite{salinas2020deepar} & \cite{chen2020probabilistic} & \cite{wu2020deep} & \cite{singhal2022fusion} & \cite{sasal2022w} & \textcolor{blue}{EWNet} \\ \hline
    
    Australia &	\textit{RMSE}	&	{76.80}	&	{76.59}	&	75.03	&	70.53	&	81.39	&	84.58	&	64.60	&	85.66	&	134.2	&	74.51	&	72.65	&	97.31	&	{119.9}	&	74.99	&	74.43	&	{93.98}	&	59.31	&	94.81	&	152.9	&	126.0	&	{138.9}	&	{135.7}	&	\underline{\textbf{39.42}} \\
    Influenza	&	\textit{MASE}	&	{2.612}	&	{2.614}	&	2.575	&	2.485	&	2.729	&	2.812	&	2.268	&	2.735	&	3.859	&	2.559	&	2.537	&	3.087	&	{4.032}	&	2.571	&	2.554	&	{3.009}	&	1.860	&	3.189	&	5.630	&	4.259	&	{4.872}	&	{4.867}	&	\underline{\textbf{1.228}} \\
    &	\textit{MAE}	&	{64.58}	&	{64.62}	&	63.64	&	61.41	&	67.46	&	69.52	&	56.07	&	67.60	&	95.39	&	63.25	&	62.72	&	76.31	&	{99.67}	&	63.55	&	63.15	&	{74.39}	&	45.98	&	78.84	&	139.1	&	105.3	&	{120.4}	&	{120.3}	&	\underline{\textbf{30.36}} \\
    &	\textit{sMAPE}	&	{72.62}	&	{72.55}	&	71.28	&	67.21	&	76.74	&	79.23	&	61.65	&	76.83	&	66.81	&	70.83	&	69.09	&	110.7	&	{166.1}	&	71.16	&	70.68	&	{86.76}	&	48.85	&	93.51	&	178.0	&	180.3	&	{139.8}	&	{133.8}	&	\underline{\textbf{43.91}} \\ \hline
    
    Japan &	\textit{RMSE}	&	{7.301}	&	{7.359}	&	115.9	&	6.839	&	7.076	&	114.8	&	239.3	&	138.1	&	46.95	&	109.6	&	99.47	&	236.2	&	{6.571}	&	18.21	&	112.8	&	{\underline{\textbf{6.321}}}	&	142.8	&	170.7	&	355.0	&	54.64	&	{180.5}	&	{187.1}	&	43.05 \\
    Influenza	&	\textit{MASE}	&	{1.111}	&	{1.125}	&	19.31	&	1.002	&	1.072	&	18.18	&	36.99	&	23.17	&	6.497	&	18.18	&	16.12	&	32.83	&	{0.967}	&	\underline{\textbf{0.894}}	&	18.74	&	{0.898}	&	21.88	&	30.03	&	52.29	&	7.053	&	{31.73}	&	{32.52}	&	6.381 \\
    &	\textit{MAE}	&	{6.308}	&	{6.388}	&	109.7	&	5.693	&	6.089	&	103.2	&	210.1	&	131.6	&	36.90	&	103.3	&	91.57	&	186.5	&	{5.495}	&	15.90	&	106.4	&	{\underline{\textbf{5.099}}}	&	124.3	&	170.5	&	297.0	&	40.06	&	{180.2}	&	{184.7}	&	36.24 \\
    &	\textit{sMAPE}	&	{77.57}	&	{78.08}	&	172.2	&	73.37	&	76.10	&	167.4	&	176.9	&	176.2	&	172.3	&	174.8	&	166.6	&	162.5	&	{71.18}	&	\underline{\textbf{19.21}}	&	170.7	&	{68.29}	&	169.8	&	184.7	&	192.0	&	235.4	&	{185.4}	&	{185.5}	&	135.5 \\ \hline
    
    Mexico &	\textit{RMSE}	&	{\underline{\textbf{3.990}}}	&	{4.537}	&	24.08	&	4.797	&	3.997	&	24.98	&	19.23	&	68.28	&	43.99	&	21.04	&	232.2	&	358.5	&	{6.645}	&	24.34	&	27.28	&	{7.177}	&	509.6	&	233.0	&	239.2	&	644.0	&	{77.71}	&	{80.92}	&	22.72 \\
    Influenza	&	\textit{MASE}	&	{\underline{\textbf{0.607}}}	&	{0.684}	&	4.357	&	0.700	&	0.633	&	4.071	&	2.666	&	12.03	&	7.433	&	3.749	&	38.91	&	41.79	&	{1.032}	&	4.403	&	4.562	&	{1.026}	&	89.58	&	43.80	&	33.73	&	121.0	&	{14.55}	&	{15.16}	&	4.039 \\
    &	\textit{MAE}	&	{\underline{\textbf{3.231}}}	&	{3.642}	&	23.18	&	3.720	&	3.368	&	21.65	&	14.18	&	63.99	&	39.54	&	19.94	&	207.0	&	222.3	&	{5.489}	&	23.42	&	24.26	&	{5.459}	&	476.5	&	233.0	&	179.4	&	644.0	&	{77.39}	&	{80.64}	&	21.48 \\
    &	\textit{sMAPE}	&	{\underline{\textbf{38.68}}}	&	{44.18}	&	112.8	&	41.62	&	40.82	&	103.5	&	83.00	&	150.2	&	187.5	&	106.9	&	173.1	&	173.0	&	{75.12}	&	113.2	&	109.6	&	{78.95}	&	191.2	&	185.6	&	160.3	&	194.5	&	{162.4}	&	{163.7}	&	108.1 \\ \hline

    Ahmedabad &	\textit{RMSE}	&	{19.12}	&	{18.94}	&	20.46	&	23.13	&	23.05	&	25.60	&	19.18	&	23.13	&	22.31	&	21.09	&	21.36	&	22.46	&	{21.89}	&	20.48	&	19.06	&	{20.94}	&	16.38	&	21.39	&	38.77	&	20.24	&	{25.33}	&	{26.20}	&	\underline{\textbf{11.74}} \\
    Dengue	&	\textit{MASE}	&	{1.642}	&	{1.626}	&	1.809	&	2.065	&	2.056	&	2.446	&	1.723	&	2.065	&	1.996	&	1.854	&	1.892	&	1.992	&	{1.938}	&	1.811	&	1.673	&	{1.892}	&	1.443	&	2.007	&	3.502	&	1.855	&	{2.329}	&	{2.489}	&	\underline{\textbf{1.023}} \\
    &	\textit{MAE}	&	{14.85}	&	{14.69}	&	16.35	&	18.67	&	18.58	&	22.11	&	15.57	&	18.67	&	18.04	&	16.76	&	17.11	&	18.01	&	{17.52}	&	16.37	&	15.13	&	{17.10}	&	13.05	&	18.14	&	31.66	&	16.77	&	{21.06}	&	{22.50}	&	\underline{\textbf{9.256}} \\
    &	\textit{sMAPE}	&	{86.81}	&	{85.45}	&	103.6	&	130.7	&	129.2	&	164.0	&	95.69	&	130.7	&	123.6	&	106.4	&	111.5	&	121.5	&	{115.8}	&	103.8	&	91.30	&	{109.9}	&	82.47	&	112.9	&	146.4	&	106.0	&	{118.3}	&	{139.4}	&	\underline{\textbf{54.16}} \\ \hline 

    Bangkok &	\textit{RMSE}	&	{\underline{\textbf{200.2}}}	&	{214.8}	&	553.0	&	478.3	&	589.9	&	383.0	&	388.8	&	253.0	&	678.5	&	518.6	&	392.3	&	657.7	&	{591.6}	&	429.0	&	533.9	&	{1196}	&	1197	&	1200	&	1214	&	1207	&	{479.1}	&	{1166}	&	316.0 \\
    Dengue	&	\textit{MASE}	&	{\underline{\textbf{0.616}}}	&	{0.648}	&	2.008	&	1.629	&	2.157	&	1.345	&	1.312	&	0.784	&	2.543	&	1.892	&	1.327	&	2.447	&	{2.167}	&	1.519	&	1.886	&	{4.576}	&	4.578	&	4.588	&	4.631	&	4.615	&	{1.660}	&	{4.457}	&	1.006 \\
    &	\textit{MAE}	&	{\underline{\textbf{159.2}}}	&	{167.3}	&	518.5	&	420.6	&	556.8	&	347.3	&	338.8	&	202.3	&	656.5	&	488.4	&	342.8	&	632.0	&	{559.5}	&	392.4	&	487.1	&	{1181}	&	1182	&	1184	&	1195	&	1191	&	{428.6}	&	{1150}	&	259.9 \\
    &	\textit{sMAPE}	&	{\underline{\textbf{13.59}}}	&	{14.21}	&	53.63	&	41.11	&	58.93	&	32.44	&	31.51	&	18.13	&	74.30	&	49.81	&	32.08	&	70.22	&	{59.33}	&	37.97	&	49.59	&	{194.5}	&	195.0	&	195.6	&	195.4	&	198.0	&	{35.34}	&	{184.8}	&	22.54 \\ \hline
    
    Colombia &	\textit{RMSE}	&	{975.4}	&	{1017}	&	949.8	&	917.2	&	941.8	&	840.5	&	809.2	&	906.0	&	801.4	&	934.4	&	1092	&	1010	&	{\underline{\textbf{411.4}}}	&	952.1	&	947.5	&	{1285}	&	917.9	&	1153	&	3465	&	630.1	&	{1847}	&	{701.9}	&	417.8 \\
    Dengue &	\textit{MASE}	&	{9.321}	&	{9.708}	&	9.068	&	8.717	&	8.943	&	8.201	&	7.662	&	8.612	&	7.665	&	8.886	&	10.41	&	9.638	&	{\underline{\textbf{3.141}}}	&	9.067	&	8.991	&	{12.72}	&	8.916	&	11.28	&	33.73	&	5.495	&	{17.79}	&	{5.308}	&	4.010 \\
    &	\textit{MAE}	&	{901.6}	&	{938.9}	&	877.0	&	843.1	&	865.0	&	793.2	&	741.0	&	832.9	&	741.4	&	859.4	&	1007	&	932.2	&	{\underline{\textbf{303.7}}}	&	877.0	&	869.6	&	{1230}	&	862.4	&	1091	&	3262	&	531.4	&	{1720}	&	{513.4}	&	387.8 \\ 
    &	\textit{sMAPE}	&	{55.81}	&	{57.29}	&	54.79	&	53.33	&	54.24	&	51.25	&	48.80	&	52.90	&	48.90	&	54.02	&	59.89	&	57.01	&	{\underline{\textbf{24.77}}}	&	54.76	&	54.41	&	{191.8}	&	54.20	&	150.5	&	174.5	&	49.50	&	{81.95}	&	{47.67}	&	28.50  \\ \hline
    
    Hong &	\textit{RMSE}	&	{2.582}	&	{2.589}	&	3.143	&	2.663	&	2.791	&	3.201	&	2.565	&	2.880	&	2.976	&	3.423	&	3.729	&	3.289	&	{2.635}	&	2.874	&	3.018	&	{2.405}	&	8.370	&	4.719	&	29.17	&	3.879	&	{3.874}	&	{4.124}	&	\underline{\textbf{2.133}} \\
    Kong &	\textit{MASE}	&	{0.667}	&	{0.672}	&	0.885	&	0.733	&	0.766	&	0.841	&	0.741	&	0.821	&	0.838	&	0.958	&	1.038	&	0.872	&	{0.664}	&	0.815	&	0.850	&	{0.652}	&	2.252	&	1.350	&	8.055	&	1.045	&	{1.022}	&	{1.259}	&	\underline{\textbf{0.578}} \\
    Dengue	&	\textit{MAE}	&	{2.000}	&	{2.016}	&	2.655	&	2.200	&	2.298	&	2.522	&	2.224	&	2.463	&	2.515	&	2.875	&	3.115	&	2.618	&	{1.991}	&	2.446	&	2.550	&	{1.957}	&	6.756	&	4.052	&	24.16	&	3.135	&	{3.067}	&	{3.776}	&	\underline{\textbf{1.736}} \\
    &	\textit{sMAPE}	&	{23.52}	&	{23.69}	&	29.65	&	25.52	&	26.46	&	28.21	&	25.91	&	27.91	&	28.52	&	31.57	&	34.42	&	26.37	&	{23.27}	&	27.71	&	28.70	&	{23.03}	&	49.40	&	39.88	&	163.6	&	39.64	&	{33.74}	&	{41.86}	&	\underline{\textbf{20.31}} \\ \hline
    
    Iquitos &	\textit{RMSE}	&	{14.06}	&	{14.04}	&	13.12	&	14.93	&	14.95	&	14.98	&	14.20	&	13.17	&	14.68	&	13.50	&	15.43	&	14.79	&	{13.73}	&	13.25	&	13.40	&	{\underline{\textbf{11.71}}}	&	13.87	&	15.09	&	62.83	&	12.28	&	{25.01}	&	{27.55}	&	11.95 \\
    Dengue	&	\textit{MASE}	&	{1.811}	&	{1.809}	&	1.703	&	1.940	&	1.943	&	1.974	&	1.848	&	1.713	&	1.898	&	1.759	&	2.008	&	1.923	&	{1.778}	&	1.730	&	1.739	&	{\underline{\textbf{1.518}}}	&	1.811	&	1.594	&	7.124	&	1.573	&	{3.266}	&	{3.549}	&	1.557 \\ 
    &	\textit{MAE}	&	{11.15}	&	{11.14}	&	10.49	&	11.95	&	11.96	&	12.16	&	11.38	&	10.55	&	11.69	&	10.83	&	12.37	&	11.84	&	{10.95}	&	10.65	&	10.71	&	{\underline{\textbf{9.352}}}	&	11.16	&	12.22	&	43.88	&	9.692	&	{20.12}	&	{21.87}	&	9.594 \\ 
    &	\textit{sMAPE}	&	{102.4}	&	{102.1}	&	91.60	&	115.7	&	116.1	&	122.6	&	103.5	&	92.36	&	112.7	&	98.37	&	124.4	&	115.1	&	{110.7}	&	93.73	&	94.77	&	{101.9}	&	117.6	&	85.52	&	170.3	&	85.55	&	{159.9}	&	{158.5} &	\underline{\textbf{81.96}} \\ \hline
    
    Philippines &	\textit{RMSE}	&	{28.75}	&	{28.51}	&	24.33	&	34.76	&	22.86	&	19.20	&	144.1	&	34.72	&	44.83	&	26.71	&	30.58	&	14.88	&	{23.81}	&	23.49	&	12.42	&	{46.55}	&	\underline{\textbf{5.502}}	&	28.10	&	137.5	&	45.36	&	{51.68}	&	{27.75}	&	6.810 \\
    Dengue &	\textit{MASE}	&	{0.652}	&	{0.641}	&	0.556	&	0.987	&	0.586	&	0.578	&	4.452	&	0.978	&	1.346	&	0.729	&	0.831	&	0.527	&	{0.747}	&	0.635	&	0.338	&	{1.535}	&	\underline{\textbf{0.177}}	&	0.693	&	4.801	&	1.495	&	{1.736}	&	{0.566}	&	0.247 \\
    &	\textit{MAE}	&	{17.18}	&	{16.89}	&	14.65	&	26.01	&	15.46	&	15.24	&	117.2	&	25.78	&	35.48	&	19.22	&	21.89	&	13.88	&	{19.67}	&	16.74	&	8.917	&	{40.44}	&	\underline{\textbf{4.686}}	&	18.27	&	126.4	&	39.39	&	{45.75}	&	{14.92}	&	6.507 \\
    &	\textit{sMAPE}	&	{31.53}	&	{30.75}	&	24.92	&	59.26	&	27.25	&	25.83	&	181.3	&	58.12	&	101.3	&	38.87	&	47.13	&	29.90	&	{35.23}	&	28.92	&	28.20	&	{123.6}	&	\underline{\textbf{9.441}}	&	36.45	&	200.0	&	118.5	&	{68.43}	&	{25.37}	&	14.33 \\ \hline
    
    San  &	\textit{RMSE}	&	{72.19}	&	{73.64}	&	97.26	&	73.07	&	72.33	&	76.13	&	95.39	&	76.34	&	93.55	&	94.87	&	199.9	&	119.7	&	{116.5}	&	138.7	&	96.52	&	{102.6}	&	106.9	&	113.1	&	467.0	&	111.3	&	{193.1}	&	{191.5}	&	\underline{\textbf{64.44}} \\ 
    Juan	&	\textit{MASE}	&	{3.264}	&	{3.305}	&	4.594	&	3.616	&	3.619	&	3.381	&	4.290	&	3.733	&	4.441	&	4.448	&	9.784	&	5.928	&	{5.627}	&	6.754	&	4.552	&	{4.677}	&	5.172	&	5.524	&	19.78	&	5.287	&	{8.903}	&	{8.883}	&	\underline{\textbf{3.242}} \\
    Dengue &	\textit{MAE}	&	{59.27}	&	{60.02}	&	83.42	&	65.66	&	65.71	&	61.39	&	77.90	&	67.79	&	80.65	&	80.78	&	177.6	&	107.6	&	{102.2}	&	122.6	&	82.66	&	{84.93}	&	93.92	&	100.3	&	359.2	&	96.02	&	{161.7}	&	{161.3}	&	\underline{\textbf{58.87}} \\
    &	\textit{sMAPE}	&	{\underline{\textbf{47.03}}}	&	{47.31}	&	79.37	&	53.62	&	53.56	&	49.22	&	65.27	&	55.59	&	58.39	&	75.64	&	82.51	&	134.1	&	{109.8}	&	98.24	&	78.12	&	{73.67}	&	101.2	&	94.91	&	160.6	&	96.86	&	{153.2}	&	{152.8}	&	48.33 \\ \hline
    
    Singapore &	\textit{RMSE}	&	{166.2}	&	{163.0}	&	151.9	&	137.2	&	149.4	&	140.0	&	136.2	&	141.7	&	151.3	&	151.9	&	\underline{\textbf{119.3}}	&	119.4	&	{226.7}	&	151.0	&	151.3	&	{290.1}	&	219.6	&	148.9	&	437.5	&	120.8	&	{321.5}	&	{293.2}	&	141.7 \\
    Dengue	&	\textit{MASE}	&	{2.460}	&	{2.396}	&	2.244	&	\underline{\textbf{2.188}}	&	2.220	&	2.389	&	2.373	&	2.307	&	2.270	&	2.277	&	2.330	&	2.260	&	{4.145}	&	2.237	&	2.244	&	{6.375}	&	3.961	&	2.244	&	10.58	&	2.576	&	{5.044}	&	{4.455}	&	2.251 \\
    &	\textit{MAE}	&	{98.81}	&	{96.20}	&	90.13	&	\underline{\textbf{87.85}}	&	89.13	&	95.94	&	95.28	&	92.63	&	91.17	&	91.44	&	93.60	&	90.79	&	{166.4}	&	89.86	&	90.13	&	{256.0}	&	159.0	&	90.12	&	425.1	&	103.4	&	{202.6}	&	{178.9}	&	90.40 \\
    &	\textit{sMAPE}	&	{28.34}	&	{27.32}	&	25.11	&	\underline{\textbf{24.65}}	&	24.79	&	27.48	&	27.33	&	26.30	&	25.53	&	25.61	&	27.07	&	26.22	&	{61.04}	&	25.03	&	25.12	&	{126.4}	&	58.99	&	25.18	&	184.9	&	29.48	&	{67.44}	&	{64.12}	&	25.48 \\ \hline
    
    Venezuela &	\textit{RMSE}	&	{476.3}	&	{481.6}	&	486.6	&	588.0	&	514.4	&	609.7	&	465.7	&	508.5	&	542.0	&	472.4	&	490.8	&	486.0	&	{524.7}	&	487.3	&	486.8	&	{1109}	&	707.9	&	985.2	&	2341	&	526.8	&	{875.4}	&	{723.8}	&	\underline{\textbf{392.0}} \\
    Dengue &	\textit{MASE}	&	{\underline{\textbf{2.535}}}	&	{2.557}	&	3.178	&	3.880	&	3.385	&	4.058	&	3.003	&	3.346	&	3.438	&	3.057	&	3.210	&	3.175	&	{2.774}	&	3.182	&	3.179	&	{7.500}	&	4.014	&	6.484	&	15.27	&	2.787	&	{5.759}	&	{4.115}	&	2.540 \\
    &	\textit{MAE}	&	{\underline{\textbf{343.2}}}	&	{346.1}	&	430.2	&	525.1	&	458.2	&	549.2	&	406.4	&	452.9	&	465.3	&	413.7	&	434.5	&	429.7	&	{375.4}	&	430.8	&	430.4	&	{1015}	&	543.3	&	877.7	&	2068	&	377.2	&	{779.6}	&	{557.1}	&	343.8 \\
    &	\textit{sMAPE}	&	{\underline{\textbf{32.49}}}	&	{32.78}	&	39.95	&	46.24	&	41.93	&	48.26	&	38.14	&	41.57	&	43.23	&	38.66	&	40.27	&	39.93	&	{36.04}	&	39.99	&	39.96	&	{187.1}	&	56.23	&	137.7	&	157.5	&	36.28	&	{63.34}	&	{61.09}	&	33.27 \\ \hline
    
    China &	\textit{RMSE}	&	{5820}	&	{5304}	&	5154	&	5472	&	5494	&	\underline{\textbf{4542}}	&	6556	&	5546	&	5732	&	5103	&	8104	&	9845	&	{6407}	&	5154	&	5979	&	{985$E^2$}	&	9735	&	985$E^2$	&	316$E^2$	&	985$E^2$	&	{7910}	&	{985$E^2$}	&	5145 \\
    Hepatitis	&	\textit{MASE}	&	{0.871}	&	{0.803}	&	0.943	&	1.043	&	1.058	&	\underline{\textbf{0.745}}	&	1.192	&	1.063	&	0.900	&	0.881	&	1.385	&	1.247	&	{1.265}	&	0.913	&	1.170	&	{20.82}	&	1.886	&	20.81	&	6.488	&	20.81	&	{1.549}	&	{20.82}	&	0.912 \\
    B	&	\textit{MAE}	&	{4118}	&	{3792}	&	4458	&	4928	&	4998	&	\underline{\textbf{3519}}	&	5632	&	5024	&	4254	&	4164	&	6547	&	5893	&	{5976}	&	4310	&	5532	&	{983$E^2$}	&	8913	&	983$E^2$	&	306$E^2$	&	983$E^2$	&	{7320}	&	{983$E^2$}	&	4313 \\
    &	\textit{sMAPE}	&	{4.211}	&	{3.891}	&	4.583	&	5.072	&	5.145	&	\underline{\textbf{3.618}}	&	5.831	&	5.172	&	4.351	&	4.267	&	6.554	&	5.826	&	{6.177}	&	4.431	&	5.706	&	{199.9}	&	9.384	&	199.9	&	37.05	&	199.9	&	{7.622}	&	{199.9}	&	4.459 \\ \hline 
    
    Colombia &	\textit{RMSE}	&	{604.1}	&	{598.7}	&	617.0	&	655.3	&	649.4	&	561.9	&	361.8	&	646.8	&	695.0	&	616.4	&	514.3	&	393.3	&	{\underline{\textbf{329.1}}}	&	614.6	&	614.8	&	{1220}	&	934.4	&	1137	&	3063	&	528.6	&	{1078}	&	{643.5}	&	555.5 \\
    Malaria  &	\textit{MASE}	&	{4.537}	&	{4.496}	&	4.660	&	4.996	&	4.950	&	4.201	&	2.499	&	4.916	&	5.317	&	4.652	&	3.905	&	2.683	&	{\underline{\textbf{2.208}}}	&	4.645	&	4.675	&	{10.31}	&	7.482	&	9.548	&	25.19	&	3.977	&	{8.347}	&	{4.522}	&	3.546 \\
    &	\textit{MAE}	&	{516.0}	&	{511.3}	&	529.9	&	568.1	&	562.9	&	477.7	&	284.2	&	559.0	&	604.6	&	528.9	&	444.1	&	305.1	&	{\underline{\textbf{250.7}}}	&	528.3	&	531.6	&	{1172}	&	850.9	&	1085	&	2865	&	452.3	&	{949.2}	&	{514.2}	&	403.3 \\
    &	\textit{sMAPE}	&	{39.22}	&	{38.97}	&	39.98	&	42.03	&	41.76	&	36.91	&	25.05	&	41.54	&	43.87	&	39.91	&	35.27	&	26.26	&	{\underline{\textbf{22.36}}}	&	39.90	&	40.09	&	{193.8}	&	55.50	&	164.5	&	173.9	&	44.08	&	{58.94}	&	{51.09}	&	31.55 \\ \hline
    
    Venezuela &	\textit{RMSE}	&	{317.6}	&	{303.3}	&	244.5	&	204.6	&	199.7	&	260.5	&	274.1	&	278.4	&	170.3	&	235.3	&	159.9	&	192.7	&	{825.8}	&	238.7	&	226.5	&	{1609}	&	1145	&	\underline{\textbf{70.19}}	&	1361	&	567.9	&	{328.5}	&	{1257}	&	197.8 \\
    Malaria	&	\textit{MASE}	&	{1.653}	&	{1.563}	&	1.217	&	0.992	&	0.970	&	1.328	&	1.398	&	1.425	&	0.839	&	1.159	&	\underline{\textbf{0.795}}	&	0.963	&	{4.769}	&	1.180	&	1.104	&	{9.421}	&	62.79	&	3.294	&	72.59	&	31.04	&	{1.721}	&	{7.328}	&	0.973 \\ 
    &	\textit{MAE}	&	{281.3}	&	{265.9}	&	207.0	&	168.7	&	165.0	&	225.9	&	237.8	&	242.5	&	142.7	&	197.2	&	135.3	&	164.0	&	{811.6}	&	200.8	&	187.9	&	{1603}	&	1140	&	\underline{\textbf{59.82}}	&	1318	&	563.8	&	{292.8}	&	{1246}	&	165.6 \\
    &	\textit{sMAPE}	&	{18.69}	&	{17.55}	&	13.32	&	10.67	&	10.42	&	14.71	&	15.52	&	15.85	&	8.950	&	12.63	&	\underline{\textbf{8.466}}	&	10.15	&	{66.95}	&	12.89	&	11.99	&	{198.5}	&	161.5	&	47.75	&	195.6	&	135.9	&	{19.67}	&	{126.5}	&	10.48 \\ \hline
         
    \end{tabular}     
    \label{M-DL}
\end{table*}
\end{landscape}
\begin{landscape}
\begin{table*}
\tiny
        \centering \caption{Short-term forecasting performance of the proposed EWNet model in comparison to the statistical, machine learning, and deep learning forecasting techniques (best results are \underline{\textbf{highlighted}})}
    \begin{tabular}{p{0.06\textwidth}p{0.042\textwidth}p{0.028\textwidth}p{0.028\textwidth}p{0.04\textwidth}p{0.028\textwidth}p{0.03\textwidth}p{0.052\textwidth}p{0.038\textwidth}p{0.038\textwidth}p{0.028\textwidth}p{0.035\textwidth}p{0.028\textwidth}p{0.032\textwidth}p{0.028\textwidth}p{0.038\textwidth}p{0.038\textwidth}p{0.03\textwidth}p{0.035\textwidth}p{0.03\textwidth}p{0.025\textwidth}p{0.042\textwidth}p{0.04\textwidth}p{0.06\textwidth}p{0.05\textwidth}} \hline

    Data & Metrics & {RW} & {RWD} & ARIMA & ETS & Theta & WARIMA & SETAR & TBATS & BSTS & Hybrid 1 & ANN & ARNN & {SVR} & Hybrid 2 & Hybrid 3 & {LSTM} & NBeats & Deep AR & TCN & Transfo- rmers & {W NBeats} & {W-Trans former} & \textcolor{blue}{Proposed} \\
    & & \cite{pearson1905problem} & \cite{entorf1997random} & \cite{box2015time} & \cite{hyndman2008forecasting} & \cite{assimakopoulos2000theta} & \cite{aminghafari2007forecasting} & \cite{tong2009threshold} & \cite{de2011forecasting} & \cite{scott2014predicting} & \cite{chakraborty2020real} & \cite{rumelhart1986learning} & \cite{faraway1998time} & \cite{smola2004tutorial} & \cite{zhang2003time} & \cite{chakraborty2019forecasting} & \cite{hochreiter1997long} & \cite{oreshkin2019n} & \cite{salinas2020deepar} & \cite{chen2020probabilistic} & \cite{wu2020deep} & \cite{singhal2022fusion} & \cite{sasal2022w} & \textcolor{blue}{EWNet} \\ \hline
    Australia &	\textit{RMSE}	&	{107.0}	&	{108.3}	&	93.19	&	85.92	&	112.3	&	61.60	&	58.30	&	40.53	&	39.79	&	94.67	&	85.69	&	35.02	&	{69.57}	&	93.19	&	92.65	&	{55.24}	&	47.89	&	54.18	&	145.9	&	62.85	&	{92.12}	&	{101.4}	&	\underline{\textbf{29.37}} \\
    Influenza    	&	\textit{MASE}	&	{4.585}	&	{4.636}	&	4.065	&	3.737	&	4.812	&	2.770	&	2.538	&	1.802	&	1.795	&	4.203	&	3.774	&	1.280	&	{2.375}	&	4.075	&	4.099	&	{2.064}	&	1.724	&	2.021	&	6.097	&	2.844	&	{4.306}	&	{4.725}	&	\underline{\textbf{1.070}} \\
    &	\textit{MAE}	&	{92.46}	&	{93.49}	&	81.97	&	75.35	&	97.03	&	55.87	&	51.13	&	36.33	&	36.19	&	84.76	&	76.12	&	25.81	&	{47.91}	&	82.17	&	82.66	&	{41.63}	&	34.77	&	40.75	&	122.9	&	57.35	&	{86.84}	&	{95.28}	&	\underline{\textbf{21.59}} \\ 
    &	\textit{sMAPE}	&	{97.09}	&	{97.39}	&	93.82	&	91.24	&	98.55	&	81.88	&	77.59	&	66.52	&	66.03	&	95.71	&	91.62	&	43.29	&	{71.75}	&	93.96	&	94.47	&	{67.24}	&	63.70	&	83.14	&	149.1	&	84.30	&	{107.2}	&	{110.9}	&	\underline{\textbf{34.58}} \\ \hline
    
    Japan &	\textit{RMSE}	&	{10.56}	&	{10.65}	&	84.52	&	8.950	&	8.803	&	30.09	&	96.69	&	102.6	&	33.09	&	74.54	&	44.88	&	210.1	&	{\underline{\textbf{7.920}}}	&	76.86	&	82.14	&	{8.944}	&	30.49	&	9.505	&	129.0	&	10.55	&	{175.5}	&	{191.0} &	22.28 \\
    Influenza	&	\textit{MASE}	&	{1.034}	&	{1.045}	&	8.564	&	0.658	&	\underline{\textbf{0.640}}	&	2.887	&	7.699	&	10.49	&	3.241	&	7.453	&	4.542	&	17.42	&	{0.755}	&	7.735	&	8.281	&	{0.667}	&	2.656	&	0.695	&	12.36	&	0.843	&	{19.48}	&	{21.16}	&	2.135 \\
    &	\textit{MAE}	&	{9.308}	&	{9.405}	&	77.07	&	5.923	&	\underline{\textbf{5.763}}	&	25.98	&	69.29	&	94.42	&	29.17	&	67.07	&	40.88	&	156.8	&	{6.799}	&	69.62	&	74.53	&	{5.999}	&	23.90	&	6.263	&	111.3	&	7.594	&	{175.3}	&	{190.4}	&	19.22 \\
    &	\textit{sMAPE}	&	{80.89}	&	{81.29}	&	152.8	&	69.59	&	\underline{\textbf{66.65}}	&	117.8	&	135.6	&	159.2	&	194.6	&	153.8	&	134.6	&	156.4	&	{70.17}	&	152.5	&	151.2	&	{73.56}	&	108.8	&	75.99	&	180.8	&	106.0	&	{180.6}	&	{181.9}	&	107.2 \\ \hline
    
    Mexico &	\textit{RMSE}	&	{6.719}	&	{7.375}	&	5.977	&	4.514	&	4.667	&	9.282	&	9.974	&	48.57	&	15.84	&	\underline{\textbf{4.201}}	&	197.7	&	11.64	&	{7.491}	&	5.844	&	10.22	&	{9.164}	&	33.15	&	8.851	&	135.4	&	8.494	&	{71.91}	&	{70.68}	&	10.19 \\ 
    Influenza	&	\textit{MASE}	&	{0.851}	&	{0.935}	&	0.771	&	0.595	&	0.623	&	1.164	&	1.189	&	6.854	&	2.139	&	\underline{\textbf{0.544}}	&	25.00	&	1.503	&	{0.960}	&	0.748	&	1.344	&	{1.209}	&	4.905	&	1.179	&	18.30 &	1.126	&	{11.18}	&	{10.99}	&	1.328 \\
    &	\textit{MAE}	&	{5.462}	&	{6.001}	&	4.946	&	3.821	&	3.998	&	7.468	&	7.630	&	43.97	&	13.72	&	\underline{\textbf{3.490}}	&	160.4	&	9.645	&	{6.161}	&	4.802	&	8.624	&	{7.759}	&	31.47	&	7.567	&	117.4	&	7.226	&	{71.71}	&	{70.57}	&	8.525 \\
    &	\textit{sMAPE}	&	{64.87}	&	{74.09}	&	49.82	&	43.36	&	45.43	&	61.27	&	58.57	&	133.2	&	169.5	&	\underline{\textbf{39.99}}	&	159.7	&	99.38	&	{77.30}	&	48.71	&	76.75	&	{113.5}	&	123.4	&	108.8	&	179.8	&	100.9	&	{157.9}	&	{157.5}	&	117.0 \\ \hline
    
    Ahmedabad &	\textit{RMSE}	&	{23.74}	&	{24.49}	&	29.93	&	32.87	&	26.09	&	15.93	&	13.14	&	21.18	&	58.17	&	30.58	&	29.17	&	6.548	&	{15.73}	&	30.07	&	28.26	&	{13.12}	&	9.711	&	5.961	&	56.49	&	13.82	&	{24.96}	&	{27.34}	&	\underline{\textbf{5.751}} \\ 
    Dengue	&	\textit{MASE}	&	{2.882}	&	{2.972}	&	3.859	&	4.189	&	3.246	&	1.993	&	1.602	&	2.632	&	7.399	&	4.083	&	3.630	&	0.807	&	{1.853}	&	3.868	&	3.531	&	{1.637}	&	1.079	&	0.752	&	6.774	&	1.738	&	{3.204}	&	{3.536}	&	\underline{\textbf{0.748}} \\
    &	\textit{MAE}	&	{19.69}	&	{20.31}	&	26.37	&	28.62	&	22.18	&	13.61	&	10.94	&	17.98	&	50.56	&	27.90	&	24.81	&	5.520	&	{12.66}	&	26.43	&	24.13	&	{11.19}	&	7.378	&	\underline{\textbf{4.865}}	&	46.29	&	11.87	&	{21.89}	&	{24.16}	&	5.111 \\
    &	\textit{sMAPE}	&	{78.97}	&	{79.84}	&	89.59	&	91.72	&	83.17	&	69.35	&	62.26	&	77.25	&	109.1	&	92.79	&	86.43	&	35.82	&	{77.79}	&	89.60	&	85.87	&	{63.96}	&	52.18	&	\underline{\textbf{34.81}}	&	157.6	&	68.72	&	{127.1}	&	{146.8}	&	38.18 \\ \hline
    
    Bangkok &	\textit{RMSE}	&	{206.2}	&	{210.8}	&	205.1	&	279.7	&	209.2	&	215.8	&	232.0	&	351.9	&	201.8	&	271.9	&	327.6	&	330.7	&	{514.6}	&	227.5	&	244.2	&	{1122}	&	1126	&	1129	&	1138	&	1132	&	{407.1}	&	{1092}	&	\underline{\textbf{192.3}} \\
    Dengue	&	\textit{MASE}	&	{0.465}	&	{0.476}	&	0.488	&	0.573	&	0.478	&	0.439	&	0.455	&	0.821	&	0.479	&	0.735	&	0.771	&	0.776	&	{1.315}	&	0.596	&	0.588	&	{3.069}	&	3.080	&	3.091	&	3.112	&	3.096	&	{0.896}	&	{2.984}	&	\underline{\textbf{0.438}} \\
    &	\textit{MAE}	&	{167.0}	&	{171.1}	&	175.3	&	206.1	&	171.6	&	\underline{\textbf{155.6}}	&	163.6	&	295.2	&	172.2	&	264.1	&	277.5	&	279.1	&	{472.8}	&	214.4	&	211.5	&	{1103}	&	1107	&	1111	&	1118	&	1113	&	{322.0}	&	{1072}	&	158.0 \\ 
    &	\textit{sMAPE}	&	{15.25}	&	{15.60}	&	15.99	&	18.26	&	15.65	&	13.77	&	\underline{\textbf{12.78}}	&	29.44	&	15.72	&	23.22	&	28.38	&	27.24	&	{51.51}	&	19.58	&	19.39	&	{194.1}	&	195.5	&	197.1	&	198.3	&	197.7	&	{26.46}	&	{183.4}	&	14.41 \\ \hline
    
    Colombia &	\textit{RMSE}	&	{162.5}	&	{164.7}	&	227.3	&	250.4	&	261.0	&	\underline{\textbf{132.4}}	&	205.5	&	196.3	&	162.5	&	192.2	&	251.2	&	160.7	&	{155.6}	&	227.1	&	228.5	&	{960.9}	&	255.0	&	819.5	&	2981	&	905.4	&	{2215}	&	{156.8}	&	166.0 \\
    Dengue	&	\textit{MASE}	&	{1.006}	&	{1.006}	&	1.754	&	2.095	&	2.217	&	1.005	&	1.519	&	1.387	&	1.576	&	1.178	&	2.113	&	1.541	&	{\underline{\textbf{0.982}}}	&	1.751	&	1.699	&	{10.39}	&	2.540	&	8.819	&	31.74	&	9.780	&	{24.26}	&	{1.205}	&	1.074 \\
    &	\textit{MAE}	&	{91.69}	&	{91.63}	&	159.7	&	190.8	&	201.9	&	91.57	&	138.3	&	126.3	&	143.5	&	107.3	&	192.5	&	140.3	&	{\underline{\textbf{89.42}}}	&	159.5	&	154.7	&	{947.2}	&	231.4	&	803.3	&	2891	&	890.8	&	{2210}	&	{109.7}	&	97.84 \\ 
    &	\textit{sMAPE}	&	{10.81}	&	{10.79}	&	16.84	&	19.50	&	20.41	&	10.67	&	14.98	&	13.90	&	16.67	&	12.15	&	19.64	&	16.18	&	{\underline{\textbf{10.53}}}	&	16.82	&	16.36	&	{190.5}	&	27.28	&	139.4	&	200.0	&	168.7	&	{107.1}	&	{12.91}	&	11.31 \\ \hline
    
    Hong &	\textit{RMSE}	&	{3.366}	&	{3.455}	&	3.503	&	3.142	&	3.342	&	3.743	&	3.158	&	2.978	&	3.357	&	2.813	&	4.850	&	2.654	&	{1.786}	&	3.985	&	3.196	&	{2.480}	&	6.246	&	1.394	&	12.34	&	2.367	&	{5.892}	&	{4.029}	&	\underline{\textbf{1.317}} \\
    Kong &	\textit{MASE}	&	{1.067}	&	{1.101}	&	1.157	&	1.028	&	1.066	&	1.421	&	1.054	&	1.011	&	1.069	&	1.042	&	1.681	&	0.853	&	{0.576}	&	1.349	&	1.100	&	{0.881}	&	2.095	&	0.370	&	4.303	&	0.728	&	{2.019}	&	{1.596}	&	\underline{\textbf{0.315}} \\
    Dengue &	\textit{MAE}	&	{2.667}	&	{2.751}	&	2.891	&	2.570	&	2.666	&	3.553	&	2.636	&	2.528	&	2.673	&	2.605	&	4.204	&	2.133	&	{1.441}	&	3.373	&	2.751	&	{2.202}	&	5.238	&	\underline{\textbf{0.925}}	&	10.75	&	1.820	&	{5.047}	&	{3.991}	&	1.089 \\
    &	\textit{sMAPE}	&	{33.98}	&	{34.77}	&	36.16	&	33.14	&	33.99	&	41.98	&	33.82	&	32.80	&	34.05	&	33.21	&	47.05	&	28.25	&	{19.52}	&	40.04	&	35.02	&	{29.49}	&	54.19	&	\underline{\textbf{9.355}}	&	157.6	&	24.89	&	{53.33}	&	{49.33}	&	18.76 \\ \hline
    
    Iquitos &	\textit{RMSE}	&	{18.03}	&	{18.29}	&	7.039	&	13.90	&	13.89	&	8.382	&	16.47	&	7.046	&	6.635	&	7.360	&	10.41	&	10.68	&	{\underline{\textbf{4.908}}}	&	6.722	&	6.292	&	{10.41}	&	5.915	&	7.379	&	149.2	&	12.19	&	{12.37}	&	{13.11}	&	5.996 \\
    Dengue	&	\textit{MASE}	&	{3.246}	&	{3.292}	&	1.179	&	2.500	&	2.498	&	1.306	&	3.019	&	1.168	&	1.072	&	1.218	&	1.807	&	1.865	&	{\underline{\textbf{0.821}}}	&	1.122	&	1.039	&	{1.820}	&	0.989	&	1.171	&	23.22	&	2.186	&	{1.751}	&	{1.812}	&	0.927 \\
    &	\textit{MAE}	&	{16.77}	&	{17.01}	&	6.088	&	12.91	&	12.90	&	6.747	&	15.60	&	6.033	&	5.540	&	6.293	&	9.340	&	9.639	&	{\underline{\textbf{4.240}}}	&	5.801	&	5.372	&	{9.403}	&	4.794	&	6.053	&	120.0	&	11.29	&	{9.047}	&	{9.359}	&	4.791 \\
    &	\textit{sMAPE}	&	{120.2}	&	{120.6}	&	89.09	&	111.5	&	111.5	&	94.06	&	118.8	&	88.12	&	86.03	&	88.73	&	103.1	&	103.1	&	{101.9}	&	87.69	&	\underline{\textbf{85.61}}	&	{104.8}	&	101.3	&	91.54	&	189.9	&	107.8	&	{118.7}	&	{125.0}	&	114.3 \\ \hline 
    
    Philippines &	\textit{RMSE}	&	{30.22}	&	{30.06}	&	35.66	&	36.39	&	33.33	&	30.02	&	33.03	&	40.83	&	48.40	&	36.22	&	31.63	&	11.53	&	{29.68}	&	33.23	&	18.64	&	{58.29}	&	6.985	&	37.13	&	158.0 &	57.95	&	{45.83}	&	{48.37}	&	\underline{\textbf{3.413}} \\
    Dengue	&	\textit{MASE}	&	{0.374}	&	{0.377}	&	0.416	&	0.440	&	0.415	&	0.374	&	0.489	&	0.522	&	0.691	&	0.424	&	0.426	&	0.161	&	{0.368}	&	0.451	&	0.302	&	{0.893}	&	0.088	&	0.447	&	2.625	&	0.894	&	{0.742}	&	{0.553}	&	\underline{\textbf{0.047}} \\
    &	\textit{MAE}	&	{21.46}	&	{21.63}	&	23.94	&	25.29	&	23.84	&	21.52	&	28.09	&	30.01	&	39.73	&	24.34	&	24.49	&	9.287	&	{21.14}	&	25.91	&	17.39	&	{51.31}	&	5.054	&	25.67	&	150.7	&	51.34	&	{42.60}	&	{31.75}	&	\underline{\textbf{2.715}} \\
    &	\textit{sMAPE}	&	{32.67}	&	{33.02}	&	37.21	&	39.96	&	37.18	&	32.78	&	43.60	&	52.24	&	82.10	&	37.82	&	38.53	&	17.95	&	{32.11}	&	40.85	&	49.33	&	{128.9}	&	\underline{\textbf{6.318}}	&	40.74	&	200.0	&	130.0	&	{67.42}	&	{55.48}	&	6.663 \\ \hline
    
    San  &	\textit{RMSE}	&	{98.52}	&	{99.63}	&	29.28	&	99.80	&	107.6	&	72.30	&	53.57	&	64.20	&	115.2	&	30.86	&	222.1	&	\underline{\textbf{21.57}}	&	{42.91}	&	93.32	&	30.22 &	{41.57}	&	29.11	&	30.03	&	477.7	&	52.69	&	{101.9}	&	{90.73}	&	37.12 \\
    Juan 	&	\textit{MASE}	&	{5.966}	&	{6.032}	&	1.787	&	6.135	&	6.637	&	4.698	&	3.271	&	4.083	&	7.097	&	1.941	&	14.24	&	\underline{\textbf{1.169}}	&	{2.011} &	5.753	&	1.844	&	{2.166}	&	1.754	&	1.897	&	26.18	&	3.129	&	{4.837}	&	{4.500}	&	2.190 \\
    Dengue &	\textit{MAE}	&	{89.00}	&	{89.98}	&	26.65	&	91.50	&	99.00	&	70.07	&	48.79	&	60.91	&	105.8	&	28.95	&	212.5	&	\underline{\textbf{17.44}} &	{29.99}	&	85.83	&	27.51	&	{32.31}	&	26.16	&	28.29	&	390.5	&	46.68	&	{72.16}	&	{67.13}	&	32.67 \\
    &	\textit{sMAPE}	&	{78.08}	&	{78.49}	&	35.34	&	79.44	&	82.61	&	68.96	&	56.06	&	63.93	&	85.21	&	37.94	&	115.5	&	\underline{\textbf{24.13}} &	{36.83}	&	58.45	&	35.70	&	{41.06}	&	36.44	&	37.92	&	175.5	&	53.96	&	{100.0}	&	{99.40}	&	42.87 \\ \hline 
    
    Singapore &	\textit{RMSE}	&	{205.1}	&	{202.9}	&	217.7	&	213.8	&	218.1	&	237.2	&	187.4	&	222.7	&	221.5	&	217.2	&	\underline{\textbf{171.4}}	&	184.3	&	{275.6}	&	199.6	&	216.7	&	{370.4}	&	293.0	&	206.0	&	643.5	&	337.2	&	{477.3}	&	{463.1}	&	218.8 \\
    Dengue	&	\textit{MASE}	&	{2.937}	&	{2.906}	&	3.197	&	3.146	&	3.207	&	3.521	&	2.677	&	3.310	&	3.272	&	3.190	&	\underline{\textbf{2.448}}	&	2.639	&	{4.352}	&	2.925	&	3.173	&	{6.633}	&	4.823	&	2.958	&	12.19	&	5.906	&	{7.212}	&	{7.623}	&	3.248 \\
    &	\textit{MAE}	&	{149.3}	&	{147.7}	&	162.4	&	159.9	&	163.0	&	179.0	&	136.0	&	168.2	&	166.3	&	162.1	&	\underline{\textbf{124.4}}	&	134.1	&	{221.2}	&	154.8	&	161.3	&	{337.2}	&	245.2	&	150.3	&	620.0	&	300.2	&	{366.6}	&	{387.5}	&	165.1 \\ 
    &	\textit{sMAPE}	&	{38.07}	&	{37.54}	&	42.78	&	41.91	&	42.98	&	48.84	&	33.80	&	44.94	&	44.18	&	42.67	&	\underline{\textbf{30.31}}	&	33.24	&	{66.79}	&	42.78	&	42.34	&	{138.2}	&	79.50	&	38.45	&	201.0	&	111.4	&	{117.3}	&	{142.9}	&	43.83 \\ \hline
    
    Venezuela &	\textit{RMSE}	&	{791.2}	&	{797.6}	&	794.7	&	742.2	&	801.9	&	767.8	&	776.2	&	804.6	&	1001	&	829.6	&	781.1	&	783.6	&	{\underline{\textbf{727.2}}}	&	792.6	&	794.8	&	{1394}	&	1115	&	1260	&	2685	&	1340	&	{752.1}	&	{986.8}	&	814.1 \\
    Dengue	&	\textit{MASE}	&	{4.101}	&	{4.136}	&	4.128	&	3.876	&	4.186	&	4.132	&	4.062	&	4.197	&	5.275	&	4.345	&	4.073	&	4.094	&	{\underline{\textbf{3.716}}}	&	4.112	&	4.128	&	{8.118}	&	6.232	&	7.257	&	16.13	&	7.776	&	{3.953}	&	{5.257}	&	4.317 \\
    &	\textit{MAE}	&	{671.9}	&	{677.6}	&	676.3	&	635.0	&	685.7	&	677.0	&	665.5	&	687.6	&	864.1	&	711.7	&	667.2	&	670.8	&	{\underline{\textbf{608.9}}}	&	673.8	&	676.4	&	{1329}	&	1021	&	1189	&	2643	&	1274	&	{647.8}	&	{861.3}	&	707.4 \\
    &	\textit{sMAPE}	&	{59.90}	&	{60.61}	&	60.47	&	55.65	&	61.73	&	61.52	&	59.31	&	61.95	&	86.40	&	65.07	&	59.45	&	59.93	&	{\underline{\textbf{52.24}}}	&	60.15	&	60.49	&	{190.5}	&	115.1	&	151.7	&	200.0	&	174.0	&	{49.36}	&	{86.35}	&	64.79 \\ \hline
    
    China &	\textit{RMSE}	&	{9519}	&	{9217}	&	6300	&	6164	&	6098	&	6700	&	5416	&	6172	&	7813	&	7271	&	7038	&	4781	&	{6421}	&	6894	&	6081	&	{968$E^2$}	&	6580	&	968$E^2$	&	171$E^3$	&	968$E^2$	&	{6346}	&	{968$E^2$}	&	\underline{\textbf{3492}} \\
    Hepatitis	&	MASE	&	{0.985}	&	{0.959}	&	0.732	&	0.705	&	0.697	&	0.791	&	0.685	&	0.707	&	0.899	&	0.827	&	0.883	&	0.616	&	{0.769}	&	0.812	&	0.669	&	{12.99}	&	0.752	&	12.99	&	22.94	&	12.99	&	{0.779} &	{12.99}	&	\underline{\textbf{0.397}} \\
    B	&	\textit{MAE}	&	{7331}	&	{7135}	&	5445	&	5239	&	5183	&	5885	&	5094	&	5254	&	6686	&	6153	&	6573	&	4584	&	{5722}	&	6042	&	4975	&	{966$E^2$}	&	5592	&	966$E^2$	&	170$E^3$	&	966$E^2$	&	{5797}	&	{966$E^2$}	&	\underline{\textbf{2954}} \\
    &	\textit{sMAPE}	&	{7.499}	&	{7.311}	&	5.642	&	5.432	&	5.375	&	6.086	&	5.289	&	5.448	&	6.880	&	6.352	&	6.775	&	4.690	&	{5.941}	&	6.242	&	5.161	&	{199.9}	&	5.791	&	199.9	&	200.0	&	199.9	&	{6.022}	&	{199.9}	&	\underline{\textbf{3.063}} \\ \hline
    
    Colombia &	\textit{RMSE}	&	{256.7}	&	{252.8}	&	266.2	&	276.3	&	274.4	&	437.2	&	417.6	&	306.3	&	218.5	&	269.9	&	243.8	&	260.7	&	{373.7}	&	265.7	&	262.1	&	{989.4}	&	284.8	&	904.2	&	2685	&	928.5	&	{1347}	&	{390.4}	&	\underline{\textbf{183.4}} \\
    Malaria 	&	\textit{MASE}	&	{1.604}	&	{1.604}	&	1.472	&	1.481	&	1.473	&	2.384	&	2.344	&	1.617	&	1.396	&	1.676	&	1.384	&	1.459	&	{2.297}	&	1.487	&	1.453	&	{7.887}	&	1.584	&	7.155	&	21.57	&	7.364	&	{10.77}	&	{2.879}	&	\underline{\textbf{1.217}} \\
    &	\textit{MAE}	&	{194.5}	&	{194.3}	&	178.3	&	179.4	&	178.5	&	288.9	&	284.0	&	195.8	&	169.1	&	203.1	&	167.7	&	176.8	&	{278.4}	&	180.1	&	176.0	&	{955.6}	&	191.9	&	867.0	&	2614	&	892.3	&	{1305}	&	{348.9}	&	\underline{\textbf{147.5}} \\
    &	\textit{sMAPE}	&	{22.92}	&	{22.96}	&	21.11	&	21.11	&	21.05	&	28.89	&	28.75	&	22.36	&	20.45	&	23.70	&	20.24	&	21.01	&	{28.91}	&	21.29	&	20.92	&	{192.3}	&	22.33	&	156.3	&	200.0	&	165.8	&	{81.87}	&	{39.77}	&	\underline{\textbf{18.39}} \\ \hline
   
    Venezuela &	\textit{RMSE}	&	{121.7}	&	{123.9}	&	121.1	&	120.6	&	120.7	&	138.3	&	158.9	&	120.6	&	127.3	&	139.4	&	129.6	&	151.9	&	{814.2}	&	120.8	&	120.9	&	{1604}	&	187.9	&	1469	&	2766	&	1535	&	{395.7}	&	{1275}	&	\underline{\textbf{114.2}} \\
    Malaria	&	\textit{MASE}	&	{0.747}	&	{0.719}	&	0.825	&	0.798	&	0.786	&	0.960	&	1.112	&	0.801	&	\underline{\textbf{0.711}}	&	0.874	&	0.930	&	1.072	&	{6.147}	&	0.802	&	0.816	&	{12.21}	&	1.124	&	11.18	&	20.86	&	11.68	&	{2.736}	&	{9.687}	&	0.782 \\
    &	\textit{MAE}	&	{97.92}	&	{94.26}	&	108.0	&	104.4	&	102.9	&	125.7	&	145.6	&	104.9	&	\underline{\textbf{93.13}}	&	114.5	&	121.8	&	140.5	&	{805.3}	&	105.0	&	106.9	&	{1599}	&	147.3	&	1464	&	2733	&	1530	&	{358.4}	&	{1269}	&	102.5 \\ 
    &	\textit{sMAPE}	&	{6.107}	&	{5.883}	&	6.731	&	6.511	&	6.414	&	7.845	&	9.136	&	6.536	&	\underline{\textbf{5.815}}	&	7.141	&	7.598	&	8.797	&	{65.59}	&	6.546	&	6.664	&	{194.1}	&	9.264	&	164.0	&	200.0	&	178.1	&	{24.81}	&	{128.3}	&	6.407 \\ \hline
    
    \end{tabular}
    \label{S-DL}
\end{table*}
\end{landscape}

\begin{figure}[H]
    \centering
    \includegraphics[width=\textwidth]{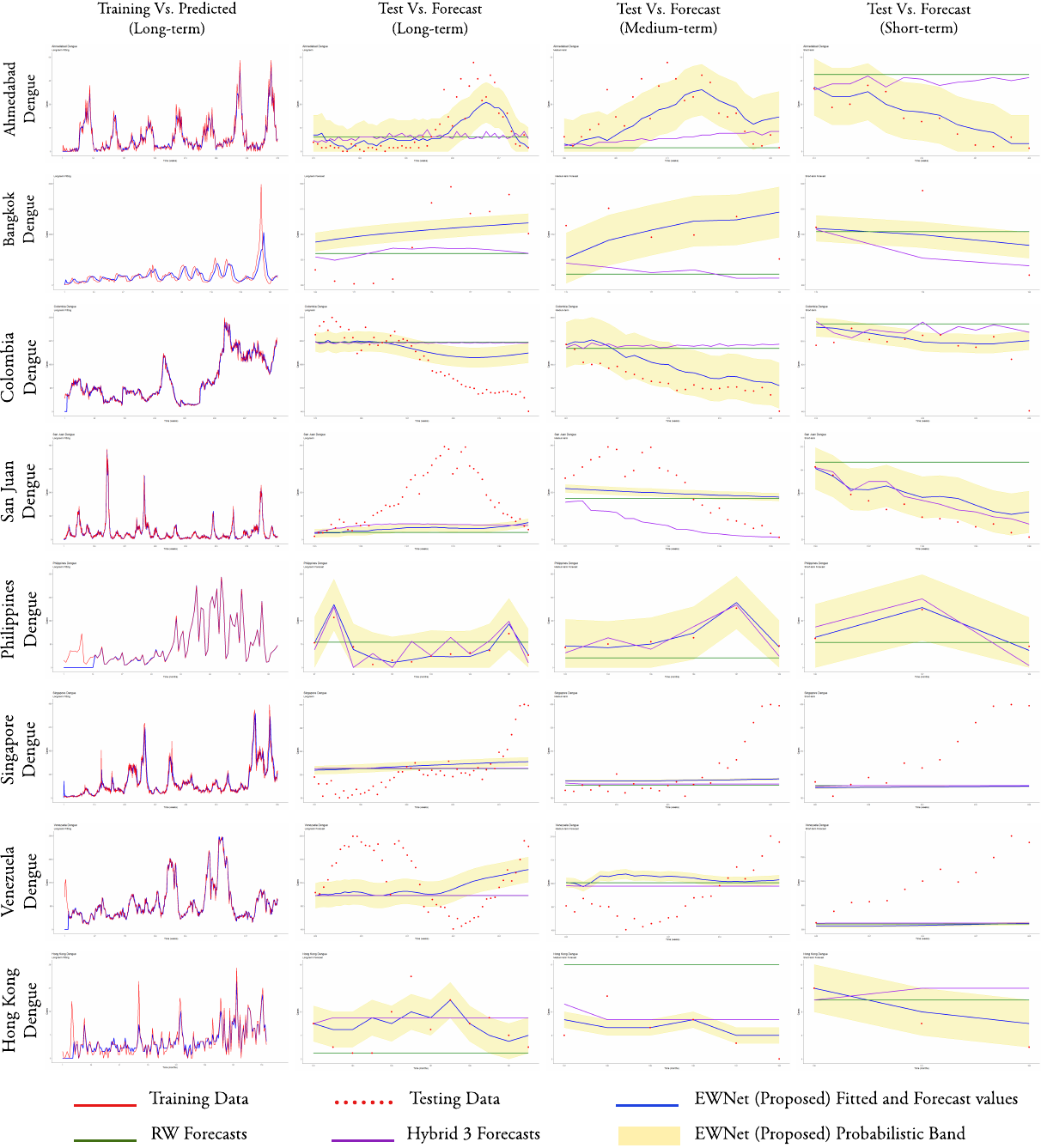}
    \caption{{The plot shows the ground truth (red), fitted values and forecasts of the EWNet model (blue), forecasts of the RW model (green), forecasts of the Hybrid-3 model (purple), and the probabilistic band (based on the confidence interval approach) of the proposed EWNet framework (yellow shaded) for different datasets. On each row, the plots from left to right represent the training and fitted values of the EWNet framework; long-term forecasts (point and interval) and ground truth data; medium-term forecasts (point and interval) and ground truth data; and short-term forecasts (point and interval) and ground truth data, respectively. For each plot, the vertical axis represents dengue cases, and the horizontal axis represents the time horizon.}}
    \label{WARNN_Fitting_1}
\end{figure}

\begin{figure}[H]
    \centering
    \includegraphics[width=\textwidth]{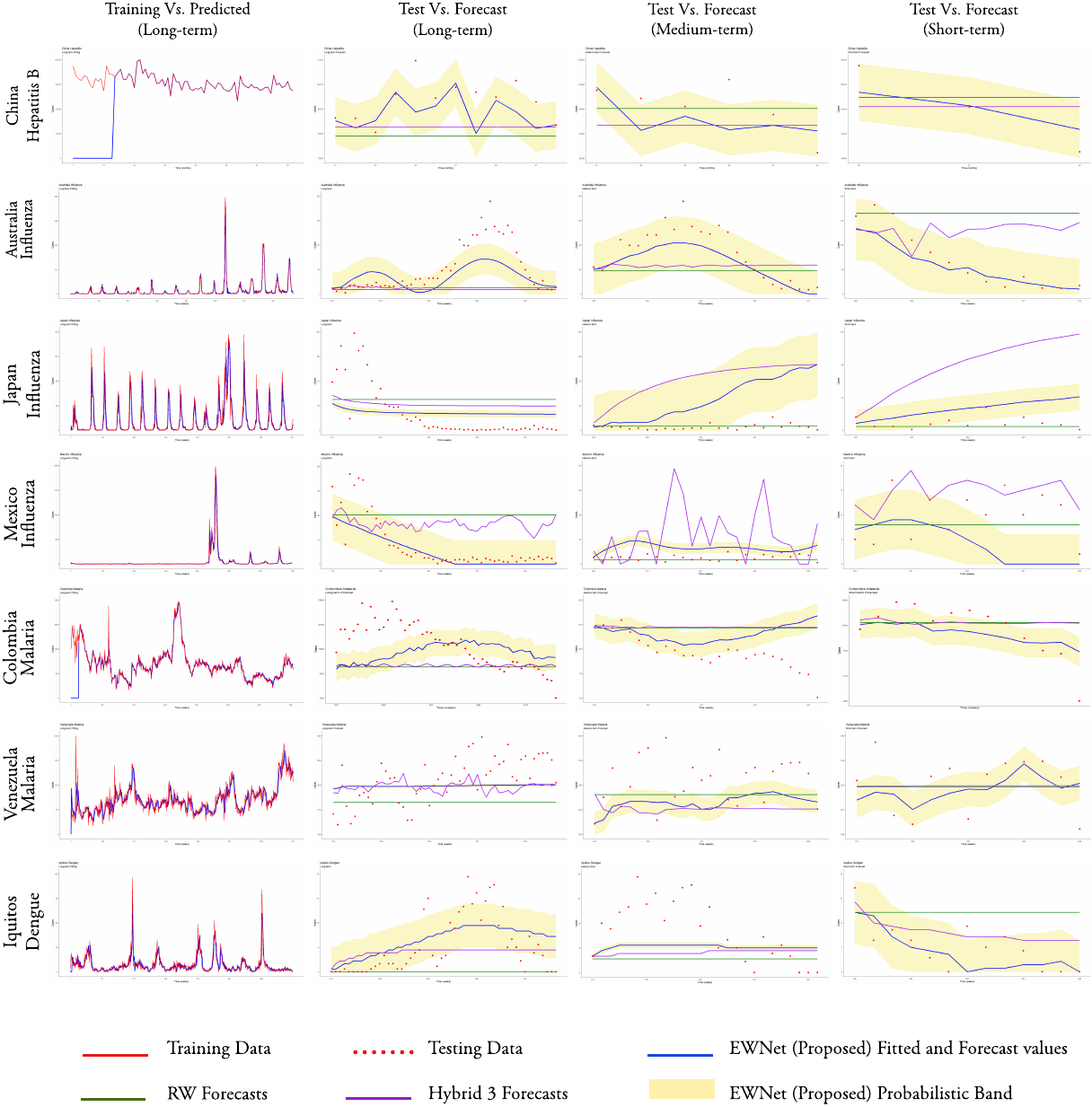}
    \caption{{The plot shows the ground truth (red), fitted values of the EWNet model (blue), and forecasts of the overall top two performing models based on four statistical measures: EWNet (blue), RW (green), Hybrid-3 (purple), and the probabilistic band (based on the confidence interval approach) of the proposed EWNet framework (yellow shaded) for different datasets. On each row, the plots from left to right represent the training and fitted values of the EWNet framework; long-term forecasts (point and interval) and ground truth data; medium-term forecasts (point and interval) and ground truth data; and short-term forecasts (point and interval) and ground truth data, respectively. For each plot, the vertical axis represents epidemic cases, and the horizontal axis represents the time horizon.}}
    \label{WARNN_Fitting_2}
\end{figure}

\begin{figure}[H]
    \centering
    \includegraphics[width=\textwidth]{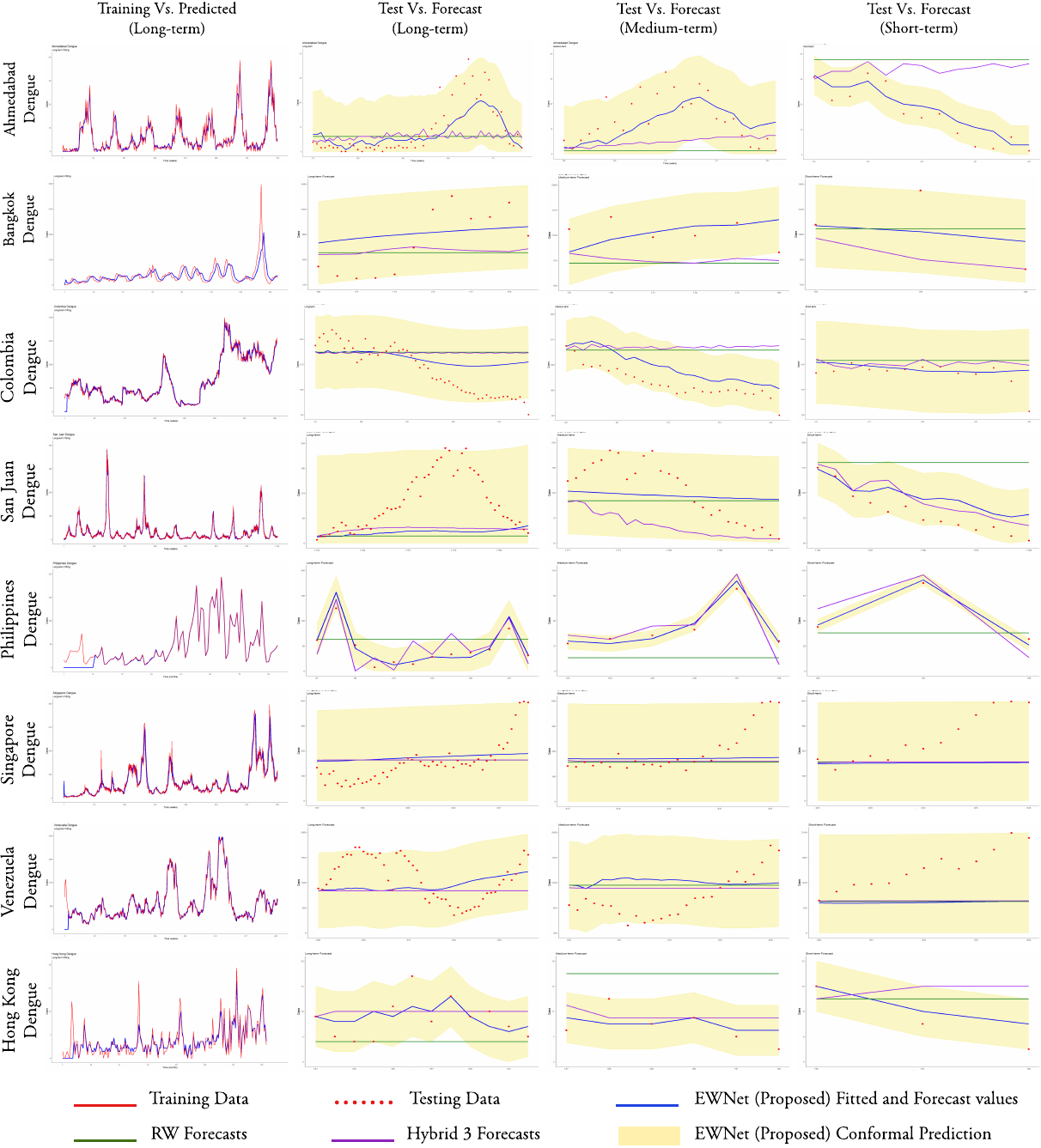}
    \caption{{The plot shows the ground truth (red), fitted values and forecasts of the EWNet model (blue), forecasts of the RW model (green), forecasts of the Hybrid-3 model (purple), and the probabilistic band (based on the conformal prediction approach) of the proposed EWNet framework (yellow shaded) for different datasets. On each row, the plots from left to right represent the training and fitted values of the EWNet framework; long-term forecasts (point and interval) and ground truth data; medium-term forecasts (point and interval) and ground truth data; and short-term forecasts (point and interval) and ground truth data, respectively. For each plot, the vertical axis represents dengue cases, and the horizontal axis represents the time horizon.}}
    \label{WARNN_Fitting_Conformal_1}
\end{figure}

\begin{figure}[H]
    \centering
    \includegraphics[width=\textwidth]{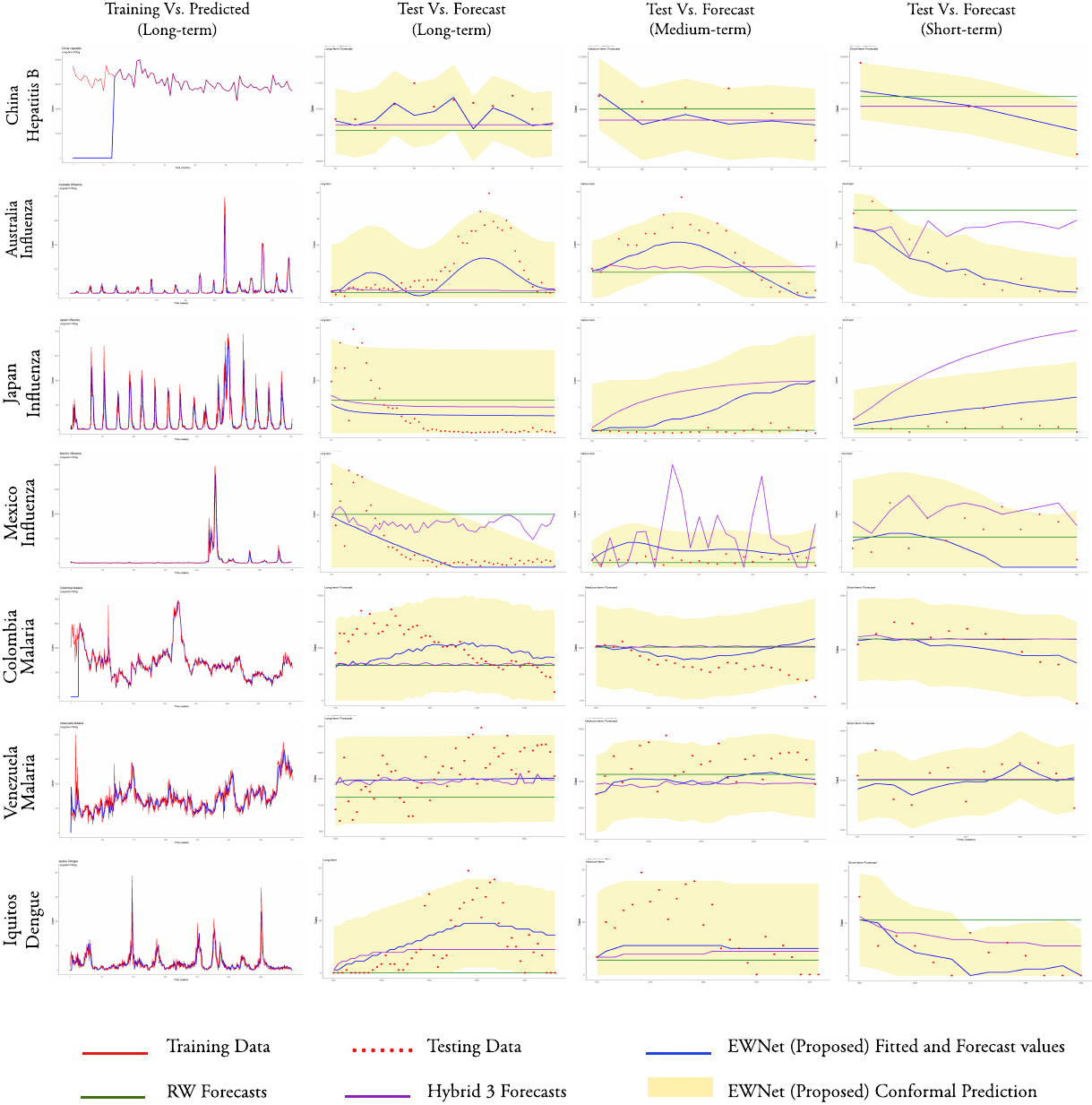}
    \caption{{The plot shows the ground truth (red), fitted values of the EWNet model (blue), and forecasts of the overall top two performing models based on four statistical measures: EWNet (blue), RW (green), Hybrid-3 (purple), and the probabilistic band (based on the conformal prediction approach) of the proposed EWNet framework (yellow shaded) for different datasets. On each row, the plots from left to right represent the training and fitted values of the EWNet framework; long-term forecasts (point and interval) and ground truth data; medium-term forecasts (point and interval) and ground truth data; and short-term forecasts (point and interval) and ground truth data, respectively. For each plot, the vertical axis represents epidemic cases, and the horizontal axis represents the time horizon.}}
    \label{WARNN_Fitting_Conformal_2}
\end{figure}


\section{\label{stat_signif_test}Significance of the Improvement and Threats to Validity}

In this section, we comment on the significance of the improvements in accuracy measures and discuss the potential threats that can impact the results of our experimental analysis.

\subsection{Overall Assessment of the Benchmark Comparisons and Potential Improvement}
A couple of interesting phenomena are observable from the experiments. {Firstly, the proposed EWNet framework produces the best epicasting results in 60\% of the datasets (9 out of 15 datasets) for long-term forecast horizon, whereas in medium-term and short-term forecasting, it outperforms the competitive forecasters in 27\% and 47\% cases, respectively in comparison with 22 benchmark forecasters} 
{Secondly, among the baseline forecasters, the ARNN \cite{faraway1998time} and the support vector regression (SVR) \cite{smola2004tutorial} models generate a better short-term forecast, whereas for medium-term epicasting, the persistence models namely, random walk (RW) \cite{pearson1905problem} and the random walk with drift (RWD) \cite{entorf1997random} methods demonstrate higher accuracy. Moreover, for long-term horizon WARIMA \cite{aminghafari2007forecasting}, hybrid ARIMA-WARIMA (Hybrid 1) \cite{chakraborty2020real}, and TBATS \cite{de2011forecasting} models have better forecasting ability than the previously proposed baseline epicasters. Nevertheless, the overall performance of the random walk (RW) \cite{pearson1905problem} model and hybrid ARIMA-ARNN (Hybrid-3) \cite{chakraborty2019forecasting} framework are better than other baselines in terms of different accuracy measures. Another critical observation is that the performance of the advanced deep learning frameworks, specifically LSTM \cite{hochreiter1997long}, NBeats \cite{oreshkin2019n}, and Deep AR \cite{salinas2020deepar} is superior in comparison with other models for 17\% of the cases.} 
This observation is interesting since the epidemic datasets' lengths range from 92 to 1196, and deep learners mostly succeed with large datasets. It is a common problem in epidemic datasets since historical records are seldom available. {In our experimental evaluation, we also employed other wavelet-based ensemble techniques with traditional ARIMA model and data-driven Transformers and NBeats methods in the combination phase as  WARIMA \cite{aminghafari2007forecasting}, W-Transformers \cite{sasal2022w}, and W-NBeats \cite{singhal2022fusion} models, respectively. Although the WARIMA \cite{aminghafari2007forecasting} method generates better epicasts for the long-term horizon, its overall rank of 9.79 (w.r.t. MASE score, ref Fig. \ref{MCB_Test}(b)) lags behind the proposed EWNet framework with an overall rank of 3.69 (w.r.t. MASE score, ref Fig. \ref{MCB_Test}(b)). This failure of the WARIMA model is primarily attributed to the inability of the linear ARIMA method to generalize well on nonlinear epidemic datasets. In the case of the recently proposed W-Transformers model \cite{sasal2022w}, the authors have extended the idea of EWNet by incorporating the attention-based Transformers model with the MODWT decomposed time series. As aptly pointed out by the authors in their manuscript, this approach works better with high-frequency datasets having several observations; however, for the epidemic datasets with fewer historical observations, this framework fails to generate satisfactory forecasts \cite{sasal2022w}. Moreover, the W-NBeats architecture utilizes the deep learning-based NBeats model in the ensemble framework. Since the NBeats model is a fully-connected deep neural network architecture based on backward and forward residual links, it is a benchmark method for large time series datasets with complex seasonalities \cite{oreshkin2019n}. However, real-world epidemic datasets exhibit irregularities and typically comprise of limited data (low-frequency), leading to the failure of the W-NBeats framework to generate satisfactory results in the epicasting task as compared to the proposed EWNet model.}

From Tables \ref{S-DL}-\ref{L-Trad}, we observe a significant improvement in epicasting by applying the proposed EWNet framework as reported by the RMSE, MASE, MAE, and sMAPE scores. Furthermore, the evaluation of the EWNet model on the crude incidence data of various diseases for diverse regions portrays that the proposed methodology can capture the long-range dependence of the series. 
Thus, based on the experimental evaluations, we can conclude that the framework proposed in this paper can potentially be used as an early warning system by public health officials and disease control programs to plan and prevent the outbreak with a substantial lead time.

\subsection{\label{robustness}Statistical Significance of the Results}

Next, we focus on determining the statistical significance of the forecasts obtained from our proposed model compared to its counterparts generated by other benchmark forecasters. We initially utilize multiple comparisons with the best (MCB) \cite{koning2005m3} procedure to determine the relative performance of different methods. For this non-parametric test, we compute the models' average ranks based on the RMSE, MASE, MAE, and sMAPE scores for different epidemic datasets and their corresponding critical distances. {The results of the MCB test presented in Fig. \ref{MCB_Test} can be interpreted as follows: The proposed EWNet model has the least rank ($3.57$), ($3.69$), ($3.82$), and ($4.31$); in terms of RMSE, MASE, MAE, and sMAPE scores. Moreover, the upper boundary of the critical distance for the EWNet model (marked by the shaded region) is the reference value for the test. Since all the benchmark forecasters have critical intervals (w.r.t. RMSE, MASE, and MAE scores) entirely above the reference value without overlap, they perform significantly worse than the proposed EWNet method. In the case of the sMAPE metric, there is a slight overlap between the critical intervals of the EWNet framework and the  RW model; however, the non-overlapping critical intervals for the other baseline forecasters indicate that their performance is significantly worse than the proposed method.}

\begin{figure}[H]
\centering
\includegraphics[width=0.46\textwidth]{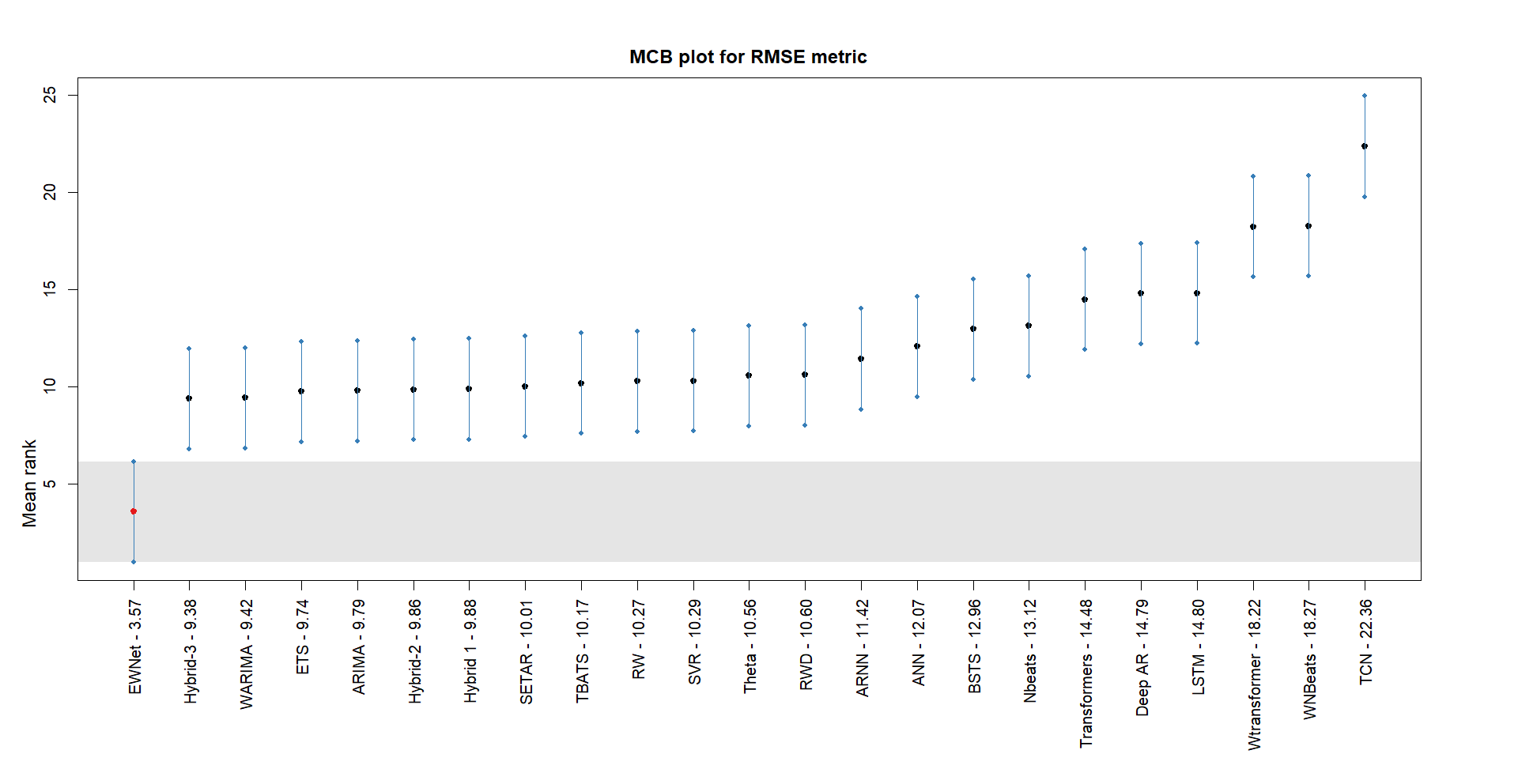}(a)
\includegraphics[width=0.46\textwidth]{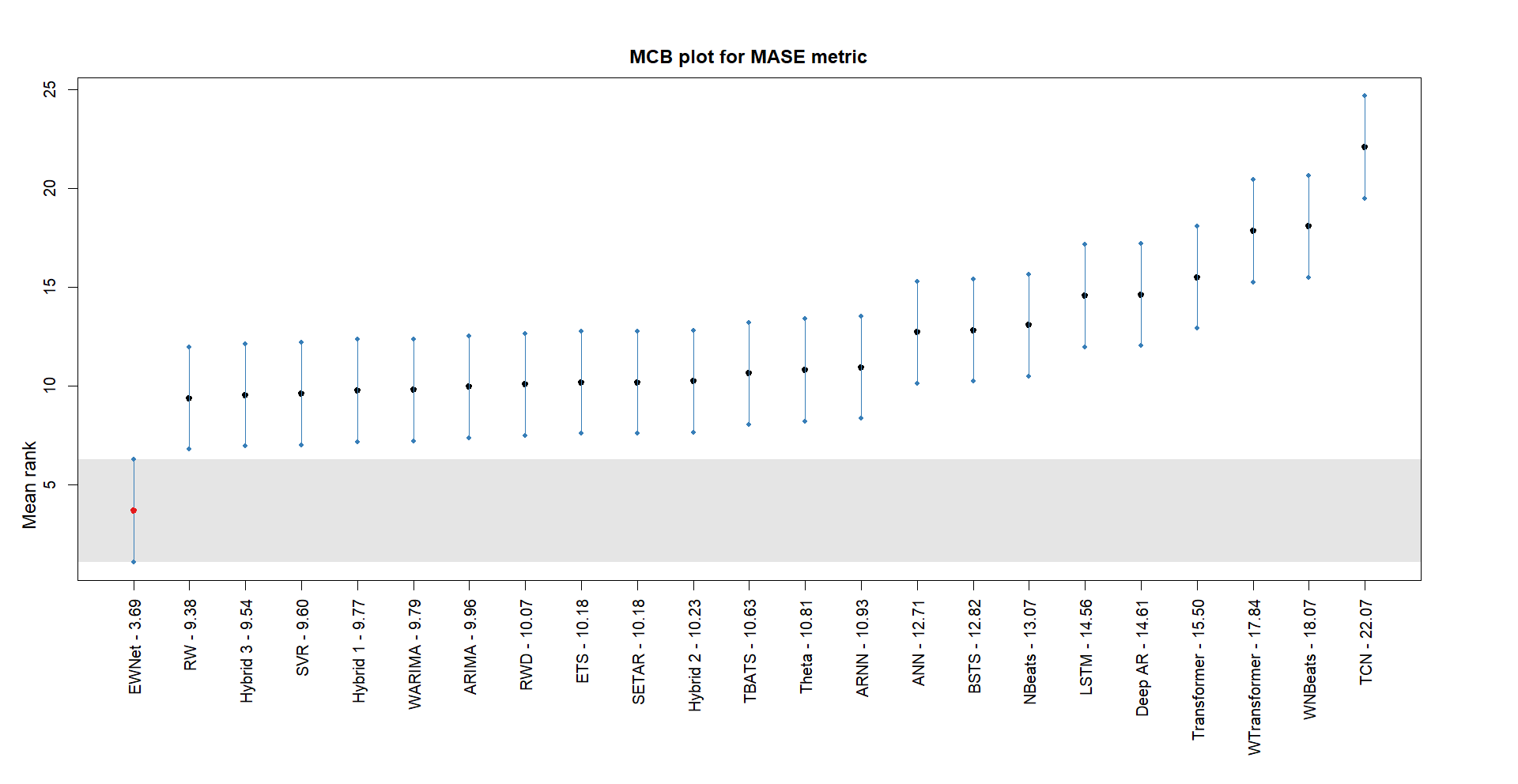}(b)\\
\includegraphics[width=0.46\textwidth]{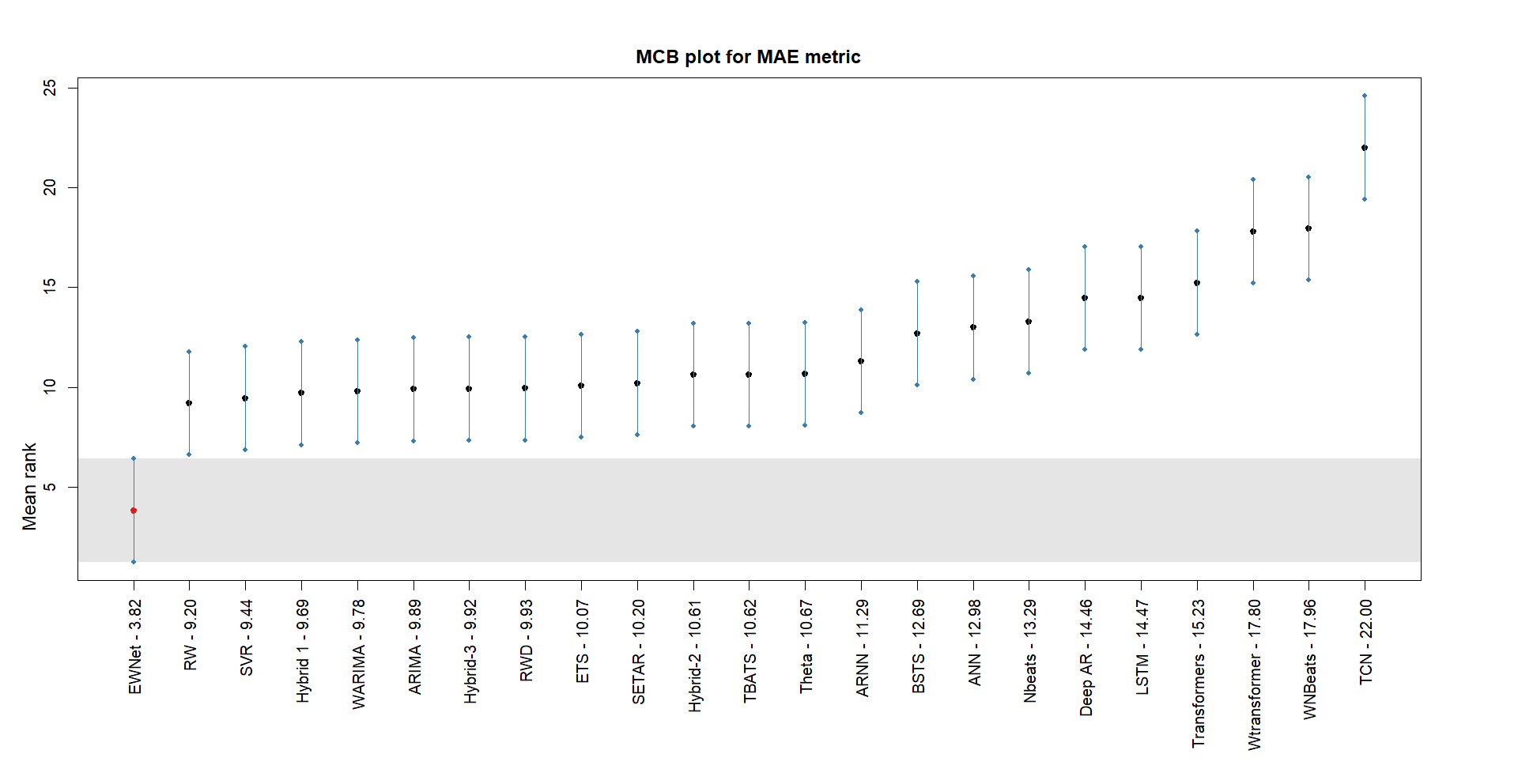}(c)
\includegraphics[width=0.46\textwidth]{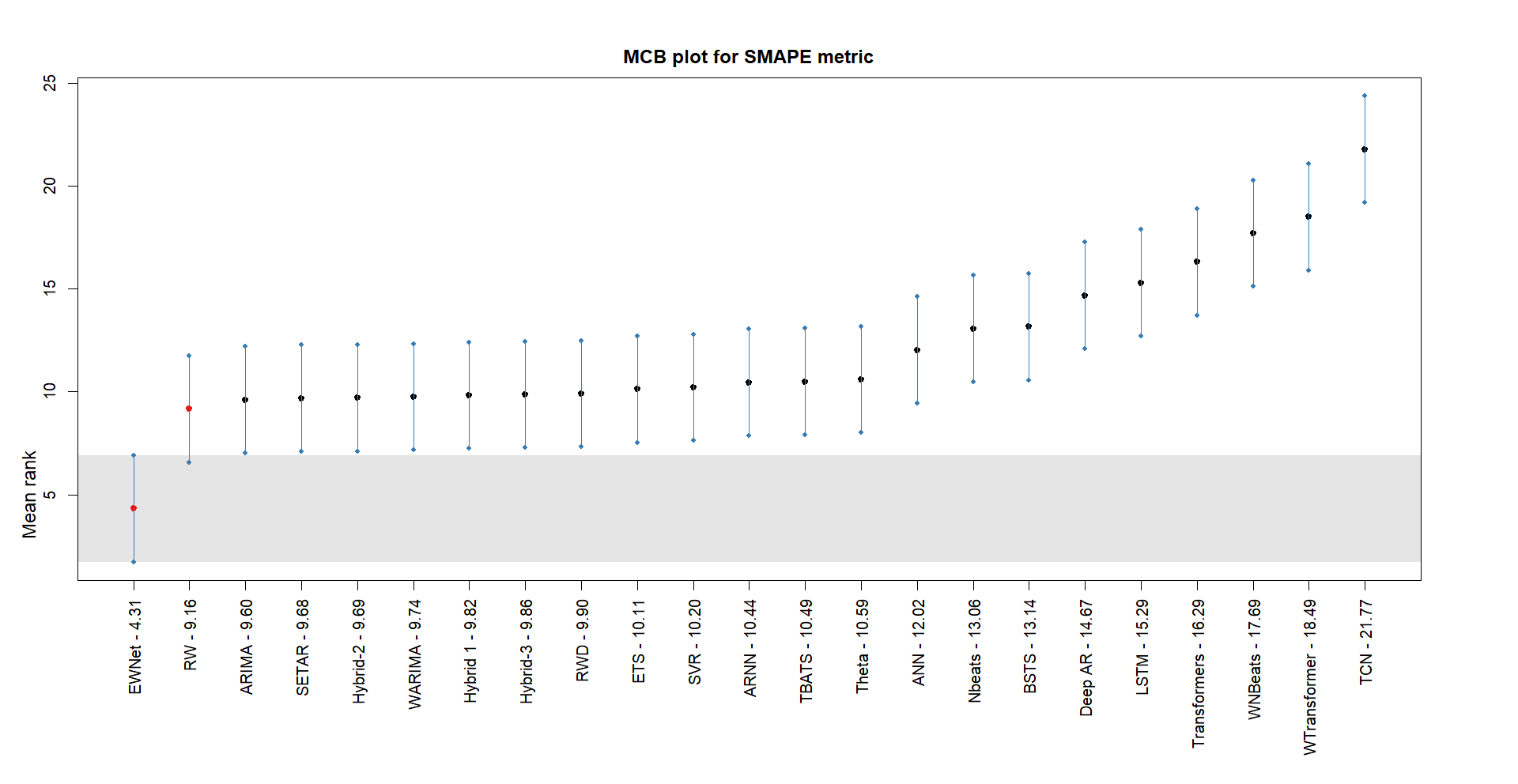}(d)
\caption{{Visualization of the multiple comparisons with the best (MCB) analysis. The figure demonstrates the MCB test results w.r.t. (a) RMSE metric, (b) MASE metric, (c) MAE metric, and (a) sMAPE metric. The vertical axis for each plot represents the average rank and the horizontal axis depicts the corresponding model such that {EWNet-3.57} in (a) indicates the average rank of the proposed EWNet model based on RMSE metric is $3.57$, and similar to others.}}\label{MCB_Test}
\end{figure}

Alongside the MCB test, we consider a non-parametric Friedman test \cite{friedman1937use, friedman1940comparison} for determining the robustness of our experimental evaluation. This statistical methodology tests the null hypothesis that all models are equivalent based on their rankings across various accuracy measures for different datasets. The ranking mechanism assigns rank 1 to the best-performing method, rank 2 to the second-best, and so on. The average of the ranks across all the datasets is then computed for different models. This distribution-free test rejects the null hypothesis of model equivalence if the value of the test statistic is greater than the critical value \cite{iman1980approximations}. Let $r_{m,d}$ denote the rank assigned to $m$-th model (out of $\tilde{M}$ models) for the $d$-th dataset (out of $\tilde{D}$ datasets). The Friedman test compares the average rank, computed using the following formula, among several algorithms: 
$ R_m = \frac{1}{\tilde{D}} \sum_d r_{m,d}. $
Under the null hypothesis, i.e., the ranks $R_m$ are equal for all $m = 1, 2, \ldots, \tilde{M}$, the Friedman statistic defined as:
\[\chi_F^2 = \frac{12 \tilde{D}}{\tilde{M}(\tilde{M}+1)} \left[\sum_m R_m^2 - \frac{\tilde{M}(\tilde{M}+1)^2}{4}\right],\]
follows $\chi^2$ distribution with $(\tilde{M}-1)$ degrees of freedom, when $\tilde{M}$ and $\tilde{D}$ are large. Owing to several difficulties with the Friedman statistic for a lesser number of datasets and algorithms, a modification of the test statistic was proposed in 1980 by Iman \cite{iman1980approximations} as 
\[
    F_F = \frac{(\tilde{D}-1)\chi_F^2}{\tilde{D} (\tilde{M}-1)-\chi_F^2},
\]
which is distributed as $F$ distribution with $(\tilde{M}-1)$ and $(\tilde{M}-1)(\tilde{D}-1)$ degrees of freedom. 

\begin{table}
\centering
\caption {Average rank of the algorithms corresponding to the performance measures (best-ranked model is \underline{\textbf{highlighted}})}
    \begin{tabular}{|c|c|c|c|c|}
    \hline
        Models & RMSE & MASE & MAE & sMAPE \\ \hline
        {RW} \cite{pearson1905problem} &  {10.27}	&	{9.378}	&	{9.200}	&	{9.156} \\ \hline
        {RWD} \cite{entorf1997random}  &  {10.60}	&	{10.07}	&	{9.933}	&	{9.889} \\ \hline
        ARIMA \cite{box2015time} &	9.786	&	9.961	&	9.887	&	9.604 \\ \hline
        ETS \cite{hyndman2008forecasting}   & 9.741	&	10.18	&	10.07	&	9.911 \\ \hline 
        Theta \cite{assimakopoulos2000theta} & 10.56	&	10.80	&	10.67	&	10.33  \\ \hline 
        WARIMA \cite{aminghafari2007forecasting} & 9.422	&	9.788	&	9.778	&	9.400 \\ \hline  
        SETAR \cite{tong2009threshold} & 10.01	&	10.18	&	10.20	&	9.244 \\ \hline 
        TBATS \cite{de2011forecasting} & 10.17	&	10.63	&	10.62	&	10.27 \\ \hline
        BSTS \cite{scott2014predicting} &	12.96	&	12.82	&	12.69	&	13.29 \\ \hline
        Hybrid-1 \cite{chakraborty2020real} &	9.876	&	9.766	&	9.689	&	9.711 \\ \hline
        ANN \cite{rumelhart1986learning} &	12.07	&	12.71	&	12.98	&	12.09 \\ \hline
        ARNN \cite{faraway1998time} &	11.42	&	10.93	&	11.27	&	10.56 \\ \hline
        {SVR} \cite{smola2004tutorial} &	{10.29}	&	{9.600}	&	{9.444}	&	{10.36}\\ \hline
        Hybrid-2 \cite{zhang2003time} &	9.862	&	10.22	&	10.61	&	9.822 \\ \hline
        Hybrid-3 \cite{chakraborty2019forecasting} &	9.380	&	9.511	&	9.889	&	9.978 \\ \hline
        {LSTM} \cite{hochreiter1997long} &	{14.80}	&	{14.56}	&	{14.47}	&	{15.44} \\ \hline
        NBeats \cite{oreshkin2019n} &	13.12	&	13.07	&	13.27	&	13.20 \\ \hline
        Deep AR \cite{salinas2020deepar} &	14.79	&	14.61	&	14.46	&	14.71 \\ \hline
        TCN \cite{chen2020probabilistic} &	22.36	&	22.07	&	22.00	&	21.73 \\ \hline
        Transformers \cite{wu2020deep} &	14.48	&	15.50	&	15.23	&	16.40 \\ \hline
        {W-NBeats} \cite{singhal2022fusion} &	{18.27}	&	{18.07}	&	{17.96}	&	{17.67} \\ \hline
        {W-Transformer} \cite{sasal2022w} &	{18.22}	&	{17.84}	&	{17.80}	&	{18.47} \\ \hline 
        \textcolor{blue}{Proposed EWNet} &	\underline{\textbf{3.573}}	&	\underline{\textbf{3.689}}	&	\underline{\textbf{3.822}}	&	\underline{\textbf{4.431}} \\ \hline
    \end{tabular}
    \label{rank_of_models}
\end{table}

\begin{table}
    \centering
    \caption{Values of Friedman Test statistic for various accuracy metrics}
    \begin{tabular}{|l|l|l|l|l|} \hline
    Test Statistic & RMSE & MASE & MAE & sMAPE \\ \hline
    $\chi_F^2$ & {316.60} & {306.88} & {302.00} & {311.42} \\ \hline
    $F_F$  &  {20.686} & {19.766} & {19.314} & {20.193} \\ \hline
    \end{tabular}
    \label{Chi and F statistic values}
\end{table}

Following the Friedman test procedure, we compute the ranks of various models for different epidemic datasets. Table \ref{rank_of_models} provides the average ranks of the models for different accuracy measures. From Table \ref{rank_of_models}, we can infer that the proposed EWNet model gets the upper hand in epicasting the disease dynamics over all other models. 
{Amongst several benchmarks, hybrid ARIMA-ARNN (Hybrid-3) (second best model w.r.t RMSE) and random walk (RW) (second best model w.r.t MASE, MAE, and sMAPE) perform better than other baselines. 
Moreover, we summarize the value of the Friedman test statistics $\chi_F^2$ and $F_F$ obtained for the 23 models across different test horizons of the 15 datasets in Table \ref{Chi and F statistic values}. Since the observed value of the statistic $F_F$ (as tabulated in Table \ref{Chi and F statistic values}) is greater than the critical value $F_{22,968} = 1.553$, so we reject the null hypothesis at $5\%$ level of significance and conclude that the performance of the algorithms considered in our study is significantly different across all the performance measures.}

Furthermore, we proceed to check whether the forecast performance of the proposed EWNet model is significantly different from other models by utilizing a post-hoc non-parametric Wilcoxon signed-rank test \cite{woolson2007wilcoxon}. This test checks the null hypothesis that no significant difference exists between the forecasts generated by the proposed EWNet model and state-of-the-art approaches at 95\% significance level. The distribution-free Wilcoxon signed-rank test procedure rejects the null hypothesis if the calculated p-value for the test is below 0.05 and concludes that there is a significant difference between the epicasting ability of the proposed model and other state-of-the-art methods. From the results obtained in this test, tabulated in Table \ref{p-values}, we can infer that the proposed EWNet model's performance is statistically significant compared to all other models considered in the analysis. Thus from the above performed statistical tests, we can infer at a 5\% significance level that the potential improvement in the epicasting performance of our proposed EWNet framework is robust and statistically significant. 

\begin{table}[t]
    \centering
    \caption{Statistical Significance values (p-values) for EWNet and other models for Wilcoxin Signed-rank test}
    \begin{tabular}{|c|c|c|c|c|}
    \hline
        ~ & RMSE & MASE & MAE & sMAPE \\ \hline
        {RW} \cite{pearson1905problem} &  {0.00012}	&	{0.00094} &	{0.00466}	&	{0.00200} \\ \hline
        {RWD} \cite{entorf1997random}  &  {0.00008}	&	{0.00084} &	{0.00452}	&	{0.00188} \\ \hline
        ARIMA \cite{box2015time} &	$<$ 0.00001	&	$<$ 0.00001	&	$<$ 0.00001	&	{0.00108} \\ \hline
        ETS \cite{hyndman2008forecasting}   & $<$ 0.00001	&	{0.00014}	&	{$<$ 0.00001}	&	{0.00138} \\ \hline 
        Theta \cite{assimakopoulos2000theta} & $<$ 0.00001	&	{$<$ 0.00001}	&	$<$ 0.00001	&	{0.00058}  \\ \hline 
        WARIMA \cite{aminghafari2007forecasting} & $<$ 0.00001	&	$<$ 0.00001	&	$<$ 0.00001	&	{0.00328} \\ \hline  
        SETAR \cite{tong2009threshold} & {$<$ 0.00001}	&	$<$ 0.00001	&	$<$ 0.00001	&	{0.00424} \\ \hline 
        TBATS \cite{de2011forecasting} & $<$ 0.00001	&	$<$ 0.00001	&	$<$ 0.00001	&	{0.00016} \\ \hline
        BSTS \cite{scott2014predicting} &	$<$ 0.00001	&	$<$ 0.00001	&	$<$ 0.00001	&	$<$ 0.00001 \\ \hline
        Hybrid-1 \cite{chakraborty2020real} &	$<$ 0.00001	&	$<$ 0.00001	&	$<$ 0.00001	&	{0.00228} \\ \hline
        ANN \cite{rumelhart1986learning} &	$<$ 0.00001	&	$<$ 0.00001	&	$<$ 0.00001	&	$<$ 0.00001 \\ \hline
        ARNN \cite{faraway1998time} &	$<$ 0.00001	&	$<$ 0.00001	&	$<$ 0.00001	&	{$<$ 0.00001} \\ \hline
        {SVR} \cite{smola2004tutorial} &	{0.00028}	&	{0.00194} &	{0.00052}	&	{0.00100}\\ \hline
        Hybrid-2 \cite{zhang2003time} &	$<$ 0.00001	&	$<$ 0.00001	&	$<$ 0.00001	&	$<$ 0.00001 \\ \hline
        Hybrid-3 \cite{chakraborty2019forecasting} &	$<$ 0.00001	&	$<$ 0.00001	&	$<$ 0.00001	&	$<$ 0.00001 \\ \hline
        {LSTM} \cite{hochreiter1997long} &	{$<$ 0.00001}	&	{$<$ 0.00001}	&	{$<$ 0.00001}	&	{$<$ 0.00001} \\ \hline
        NBeats \cite{oreshkin2019n}      &	$<$ 0.00001	&	$<$ 0.00001	&	$<$ 0.00001	&	$<$ 0.00001 \\ \hline
        Deep AR \cite{salinas2020deepar} &	$<$ 0.00001	&	$<$ 0.00001	&	$<$ 0.00001	&	$<$ 0.00001 \\ \hline
        TCN \cite{chen2020probabilistic} &	$<$ 0.00001	&	$<$ 0.00001	&	$<$ 0.00001	&	$<$ 0.00001 \\ \hline
        Transformers \cite{wu2020deep}   &	$<$ 0.00001	&	$<$ 0.00001	&	$<$ 0.00001	&	$<$ 0.00001 \\ \hline
        {W-NBeats} \cite{singhal2022fusion} &	{$<$ 0.00001}	&	{$<$ 0.00001} &	{$<$ 0.00001}	&	{$<$ 0.00001} \\ \hline
        {W-Transformer} \cite{sasal2022w} &	{$<$ 0.00001}	&	{$<$ 0.00001}	&	{$<$ 0.00001}	&	{$<$ 0.00001} \\ \hline 
    \end{tabular}
    \label{p-values}
\end{table}

\subsection{\label{val_perf_measures}Validation of Data, Results, and Performance Measures}

Our analysis is based on fifteen epidemic datasets (influenza, dengue, malaria, and hepatitis B) collected from publicly available sources. The dengue datasets have been used multiple times in various studies for formulating better epicasting techniques \cite{chakraborty2019forecasting, deb2022ensemble, johnson2018phenomenological, johansson2019open}. Our chosen datasets are diverse in nature, representing several diseases from distinct locations, with varied lengths, frequency, and statistical characteristics, which generalizes our findings. However, further investigations on some other infectious disease datasets are essential in future work. We did not consider Covid-19 datasets in our study due to their dubious nature, and thus forecasting Covid-19 majorly failed due to lack of transparency, errors, and lack of determinacy \cite{ioannidis2020forecasting}. In our study, RMSE, MASE, MAE, and sMAPE are considered as the key performance indicator \cite{box2015time, hyndman2008forecasting}. Different accuracy measures are available in the time series forecasting literature, and the metric's choice may influence the forecasters' performance. Although we considered both absolute, percentage, and scaled error measures for computing the epicasters' performance, several other measures can be considered for studying the effectiveness of different models. 
The proposed EWNet overall performed well compared with twenty-two statistical, machine learning, and deep learning models. However, epidemic outbreaks sometimes vary with respect to climatic, social, environmental, biological, and human factors. In this study, we have only studied the past observations of epidemic datasets and extrapolated the forecasts based on past dependency to provide valuable insights into the disease dynamics. 

\section{\label{sec:Discussion}Conclusions and Discussions}
Infectious disease outbreaks play an essential role in global morbidity and mortality. Real-time epidemic forecasting provides an opportunity to predict geographic disease spread and case counts to inform public health interventions better when outbreaks occur. Providing actionable insights, such as accurate forecasting of case counts with reliable uncertainty quantification, is critical for resource allocation and preparedness planning. Epidemic forecasting (called `epicasting') is beginning to be integrated into infectious disease outbreak response decision-making processes. We propose an EWNet model that could accelerate the adoption of forecasting among public health practitioners, improve epidemic management, save lives, and reduce the economic impact of outbreaks. We investigated our proposed model on the laboratory-confirmed cases of influenza, dengue, hepatitis - B, and malaria datasets for different regions. The majority of these datasets exhibit assertive nonlinear and non-stationary behavior. We proposed a new variant of the wavelet-based forecasting technique using the ARNN model and outperformed several statistical, machine learning, and deep learning models on average. Additionally, we have shown theoretical results and derived their appropriate probabilistic bands, which back the success of the proposed EWNet model. {Based on the experimental results with epidemic datasets, the proposed EWNet model is well-suited to extrapolate the future dynamics of non-stationary and nonlinear epidemic datasets due to the hybridization of wavelet decomposition and ARNN framework. The proposed EWNet model can be deployed as an early warning system that can be monitored and automatically retrained with crude incidence data of the infectious disease in an incremental training or batch training procedure. Additionally, the theoretical basis for selecting the model's hyperparameters significantly reduces its run-time complexity compared to state-of-the-art deep learners. It enables the proposed epicaster to generate real-time forecasts. These real-time forecasts backed with reliable prediction intervals will allow health officials to monitor infectious disease dynamics and aid in designing effective disease-combatting policies.} However, several factors can be identified as essential components in establishing a good prediction for an epidemic or disease risk. For example, the accuracy of EWNet can be improved if we include geographical scale, temperature, rainfall, or other attributes that impact individual epidemics. These limitations of outbreak prediction will ensure the adoption of predictive tools by public health officials, operations managers, and healthcare practitioners. Forecasting the epidemic outbreak based on certain auxiliary variables may be considered a future scope of work to further improve the EWNet model for multivariate set-up. Another interesting future direction would be to explore the EWNet model in various other applied forecasting research.  


\subsection*{Conflict of interest}
The authors declare that they have no conflict of interest.

\subsection*{Data availability statement}
The datasets and the codes for implementing the proposed EWNet model are made available at \url{https://github.com/mad-stat/Epicasting}.

\bibliographystyle{plain}
\biboptions{square}
\bibliography{Bibliography}

\end{document}